\newif\ifprint
\printfalse

\newif\iflaysummary
\laysummaryfalse

\ifprint
\documentclass[a4paper, twoside, openright]{report}
\else
\documentclass[a4paper, twoside]{report}
\fi

\usepackage[nottoc]{tocbibind}
\usepackage[round, authoryear, sort]{natbib}
\usepackage{nicefrac}
\usepackage{multirow}
\usepackage{graphicx}
\usepackage{color}
\usepackage{subfig}
\usepackage{amsmath, amsfonts, amsthm}
\usepackage{amsthm}
\usepackage{enumitem}
\usepackage{bm}
\usepackage[a4paper, bottom=4cm, hmarginratio=1:1]{geometry}
\usepackage{fancyhdr}
\usepackage{url}
\usepackage{hyperref}
\usepackage{booktabs}
\usepackage{framed}
\usepackage{tikz}
\usepackage{microtype}
\usepackage{pdfpages}
\usepackage{parskip}
\usepackage[noline, noend, algoruled, algochapter]{algorithm2e}
\usepackage{setspace}
\usepackage{bibentry}
\usepackage{emptypage}

\fancypagestyle{plain}{
  \fancyhf{} 
  
}

\numberwithin{equation}{chapter}
\numberwithin{figure}{chapter}
\numberwithin{table}{chapter}

\theoremstyle{definition}
\newtheorem{theorem}{Theorem}[chapter]

\newcommand{\R}{\mathbb{R}}

\newcommand{\vect}[1]{\mathbf{#1}}
\newcommand{\mat}[1]{\mathbf{#1}}
\newcommand{\br}[1]{\mathopen{}\left(#1\right)\mathclose{}}
\newcommand{\set}[1]{\left\{#1\right\}}
\newcommand{\pair}[2]{\br{#1,#2}}
\newcommand{\abs}[1]{\left|#1\right|}
\newcommand{\norm}[1]{\left\|#1\right\|}
\newcommand{\bigo}[1]{\mathcal{O}\br{#1}}
\newcommand{\Prob}[1]{\mathrm{Pr}\br{#1}}
\newcommand{\prob}[1]{p\br{#1}}
\newcommand{\g}{\,|\,}
\newcommand{\kl}[2]{D_{\mathrm{KL}}\br{#1\,\|\,#2}}
\newcommand{\avg}[1]{{\mathop{\mathbb{E}}}\br{#1}}
\newcommand{\avgx}[2]{{\mathop{\mathbb{E}}}_{#2}\br{#1}}
\newcommand{\var}[1]{{\mathop{\mathbb{V}}}\br{#1}}
\newcommand{\varx}[2]{{\mathop{\mathbb{V}}}_{#2}\br{#1}}
\newcommand{\integralx}[3]{\int_{#3}#1\operatorname{d}\!#2}
\newcommand{\integral}[2]{\integralx{#1}{#2}{\,\,}}

\newcommand{\dirac}[1]{\delta\br{#1}}

\newcommand{\deriv}[2]{\frac{\partial{#1}}{\partial{#2}}}
\newcommand{\tderiv}[2]{\frac{d{#1}}{d{#2}}}
\newcommand{\gaussian}[2]{\mathcal{N}\br{#1,#2}}
\newcommand{\gaussianx}[3]{\mathcal{N}\br{#1\g #2,#3}}
\newcommand{\uniform}[2]{\mathcal{U}\br{#1,#2}}

\newcommand{\algref}[1]{\hyperlink{#1_anchor}{\ref*{#1}}}
\newcommand{\alglabel}[1]{\hypertarget{#1_anchor}{}}

\DeclareMathOperator{\trace}{tr}
\DeclareMathOperator*{\argmax}{arg\,max}
\DeclareMathOperator*{\argmin}{arg\,min}

\newcommand{\mytitle}{Neural Density Estimation and Likelihood-free Inference}
\newcommand{\myname}{George Papamakarios}
\newcommand{\thedate}{April 2019}

\begin{document}

\newenvironment{laysummary}
   {\renewcommand{\abstractname}{Lay summary}\begin{abstract}}
   {\end{abstract}\renewcommand{\abstractname}{Abstract}}

\newenvironment{acknowledgements}
   {\renewcommand{\abstractname}{Acknowledgements}\begin{abstract}}
   {\end{abstract}\renewcommand{\abstractname}{Abstract}}

\newenvironment{declaration}
   {\renewcommand{\abstractname}{Declaration}\begin{abstract}}
   {\end{abstract}\renewcommand{\abstractname}{Abstract}}

\newcommand{\HRule}{\rule{\linewidth}{0.5mm}}

\begin{titlepage}

\centering
\textsc{\LARGE University of Edinburgh}\\[1.5cm] 
\textsc{\Large School of Informatics}\\[0.5cm] 
\textsc{\large Doctor of Philosophy}\\[3.5cm]

\HRule\\[0.6cm]
{\huge \bfseries \mytitle}\\[0.4cm]
\HRule\\[1.5cm]
 
{\Large by\\[0.4cm]\textsc{\myname}}\\[8.5cm] 

{\large \thedate}

\vfill 

\end{titlepage}

\ifprint
\cleardoublepage
\fi
\newcounter{savepage}
\setcounter{savepage}{\thepage}

\begin{abstract}
I consider two problems in machine learning and statistics: the problem of estimating the joint probability density of a collection of random variables, known as \emph{density estimation}, and the problem of inferring model parameters when their likelihood is intractable, known as \emph{likelihood-free inference}. The contribution of the thesis is a set of new methods for addressing these problems that are based on recent advances in neural networks and deep learning.

The first part of the thesis is about density estimation. The joint probability density of a collection of random variables is a useful mathematical description of their statistical properties, but can be hard to estimate from data, especially when the number of random variables is large. Traditional density-estimation methods such as histograms or kernel density estimators are effective for a small number of random variables, but scale badly as the number increases. In contrast, models for density estimation based on neural networks scale better with the number of random variables, and can incorporate domain knowledge in their design. My main contribution is \emph{Masked Autoregressive Flow}, a new model for density estimation based on a bijective neural network that transforms random noise to data. At the time of its introduction, Masked Autoregressive Flow achieved state-of-the-art results in general-purpose density estimation. Since its publication, Masked Autoregressive Flow has contributed to the broader understanding of neural density estimation, and has influenced subsequent developments in the field.

The second part of the thesis is about likelihood-free inference. Typically, a statistical model can be specified either as a likelihood function that describes the statistical relationship between model parameters and data, or as a simulator that can be run forward to generate data. Specifying a statistical model as a simulator can offer greater modelling flexibility and can produce more interpretable models, but can also make inference of model parameters harder, as the likelihood of the parameters may no longer be tractable. Traditional techniques for likelihood-free inference such as approximate Bayesian computation rely on simulating data from the model, but often require a large number of simulations to produce accurate results. In this thesis, I cast the problem of likelihood-free inference as a density-estimation problem, and address it with neural density models. My main contribution is the introduction of two new methods for likelihood-free inference: \emph{Sequential Neural Posterior Estimation (Type A)}, which estimates the posterior, and \emph{Sequential Neural Likelihood}, which estimates the likelihood. Both methods use a neural density model to estimate the posterior/likelihood, and a sequential training procedure to guide simulations. My experiments show that the proposed methods produce accurate results, and are often orders of magnitude faster than alternative methods based on approximate Bayesian computation.
\end{abstract}

\setcounter{page}{\thesavepage}
\stepcounter{page}
\ifprint
\cleardoublepage
\fi
\setcounter{savepage}{\thepage}

\begin{acknowledgements}
I'm grateful to Iain Murray for supervising me during my PhD\@. Iain has had a massive impact on my research, my thinking, and my professional development; his contributions extend far beyond this thesis.
I'm also grateful to John Winn and Chris Williams, who served on my PhD committee and gave me useful feedback in our annual meetings.

I'm grateful to my co-authors Theo Pavlakou and David Sterratt for their help and contribution. Conor Durkan, Maria Gorinova and Amos Storkey generously read parts of this thesis and gave me useful feedback. I'm grateful to Michael Gutmann and Kyle Cranmer for our discussions, which shaped part of the thesis, and to Johann Brehmer for fixing an important error in my code.

I'd like to extend a special thank you to the developers of Theano \citep{Al-Rfou:2016:theano}, which I used in all my experiments. Their contribution to machine-learning research as a whole has been monumental, and they deserve significant credit for that.

I'm grateful to the Centre for Doctoral Training in Data Science, the Engineering and Physical Sciences Research Council, the University of Edinburgh and Microsoft Research, whose generous funding and support enabled me to pursue a PhD\@.
\end{acknowledgements}

\setcounter{page}{\thesavepage}
\stepcounter{page}
\ifprint
\cleardoublepage
\fi
\setcounter{savepage}{\thepage}

\begin{declaration}
I declare that this thesis was composed by me,
that the work contained herein is my own 
except where explicitly stated otherwise in the text,
and that this work has not been submitted for any other degree or
professional qualification except as specified.\par
\vspace{1in}\raggedleft({\em \myname})
\end{declaration}
   
\setcounter{page}{\thesavepage}
\stepcounter{page}
\ifprint
\cleardoublepage
\fi

\tableofcontents

\chapter{Introduction}
\label{chapter:intro}

\emph{Density estimation} and \emph{likelihood-free inference} are two fundamental problems of interest in machine learning and statistics; they lie at the core of probabilistic modelling and reasoning under uncertainty, and as such they play a significant role in scientific discovery and artificial intelligence. The goal of this thesis is to develop a new set of methods for addressing these two problems, based on recent advances in neural networks and deep learning.

The first part of the thesis is about \emph{density estimation}, the problem of estimating the joint probability density of a collection of random variables from samples. 
In a sense, density estimation is the reverse of sampling: in density estimation, we are given samples and we want to retrieve the density function from which the samples were generated; in sampling, we are given a density function and we want to generate samples from it.

Density estimation addresses one of the most fundamental problems in machine learning, the problem of discovering structure from data in an unsupervised manner. A density function is a complete description of the joint statistical properties of the data, and in that sense a model of the density function can be viewed as a model of data structure. As such, a model of the density function can be used in a variety of downstream tasks that involve knowledge of data structure, such as inference, prediction, data completion and data generation.

The second part of the thesis is about \emph{statistical inference}, the problem of inferring parameters of interest from observations given a model of their statistical relationship. In this work, I adopt a Bayesian framework for statistical inference: beliefs over unknown quantities are represented by density functions, and Bayes' rule is used to update prior beliefs given new observations. Bayesian inference is one of the two main approaches to statistical inference (the other one being frequentist inference), and is widely used in science and engineering. From now on, if not made specific, the term `inference' will refer to `Bayesian inference'.

\emph{Likelihood-free inference} refers to the situation where the likelihood of the model is too expensive to evaluate, which is typically the case when the model is specified as a simulator that stochastically generates observations given parameters. Such simulators are used ubiquitously in science and engineering for modelling complex mechanistic processes of the real world; as a result, several scientific and engineering problems can be framed as likelihood-free inference of a simulator's parameters. The goal of likelihood-free inference is to compute posterior beliefs over parameters using simulations from the model rather than likelihood evaluations.

In this thesis, I approach both density estimation and likelihood-free inference as machine-learning problems; in either case, the task is to estimate a density model from data. The methods developed in this thesis are heavily based on neural networks and deep learning. The use of deep learning is motivated by two reasons. First, neural networks have demonstrated state-of-the-art performance in a wide variety of machine-learning tasks, such as computer vision \citep{Krizhevsky:2012:alexnet}, natural-language processing \citep{Devlin:2018:bert}, generative modelling \citep{Radford:2016:dcgan}, and reinforcement learning \citep{Silver:2016:alphago} --- as we will see in this thesis, neural networks advance the state of the art in density estimation and likelihood-free inference too. Second, deep learning is actively supported by software frameworks such as \emph{Theano} \citep{Al-Rfou:2016:theano}, \emph{TensorFlow} \citep{Tensorflow:2015:whitepaper} and \emph{PyTorch} \citep{Paszke:2017:pytorch}, which provide access to powerful hardware (such as graphics-processing units) and facilitate designing and training neural networks in practice. This thesis focuses on the application of neural networks to density estimation and likelihood-free inference, and not on deep learning itself; knowledge of deep learning is generally assumed, but not absolutely required. A comprehensive review of neural networks and deep learning is provided
by \citet{Goodfellow:2016:deeplearningbook}.

\section{List of contributions}

The main contribution of this thesis is a set of new methods for density estimation and likelihood-free inference, which are based on techniques from neural networks and deep learning. In particular, the new methods contributed by this thesis are the following:
\begin{enumerate}[label=(\roman*)]
\item \textbf{Masked Autoregressive Flow}, an expressive neural density model whose density is tractable to evaluate. MAF can be trained on data to estimate their underlying density function. At the time of its introduction, MAF achieved state-of-the-art performance in density estimation, and has been influential ever since.
\item \textbf{Sequential Neural Posterior Estimation (Type A)}, a method for likelihood-free inference of simulator models that is based on neural density estimation. SNPE-A trains a neural density model on simulated data to approximate the posterior density of the parameters given observations. The designation `Type A' is in order to distinguish \mbox{SNPE-A} from its `Type-B' variant, which was proposed later. SNPE-A was shown to both improve accuracy and dramatically reduce the required number of simulations compared to traditional methods for likelihood-free inference.
\item \textbf{Sequential Neural Likelihood}, an alternative method for likelihood-free inference that uses a neural density model to approximate the likelihood instead of the posterior. SNL can be as fast and accurate as SNPE-A, but is more generally applicable and significantly more robust. Unlike SNPE-A, SNL can be used with Masked Autoregressive Flow.
\end{enumerate}

\section{Structure of the thesis}

The thesis is divided into two parts: part \ref{part:nde} is about \emph{neural density estimation}, whereas part \ref{part:lfi} is about \emph{likelihood-free inference}. Each part has a separate introduction and a separate conclusion. The two parts are intended to be standalone, and can be read independently. However, the part on likelihood-free inference makes heavy use of density-estimation techniques, hence, although not absolutely necessary, I would recommend that the part on density estimation be read first.

The thesis consists of three chapters, each of which is centred on a different published paper. Each chapter includes the paper as published; in addition, it provides extra background in order to motivate the paper, and evaluates the contribution and impact of the paper since its publication. The three chapters are listed below:

\nobibliography*

Chapter \ref{chapter:maf} is about \textbf{Masked Autoregressive Flow}, and is based on the following paper:

\begin{quote}
\bibentry{Papamakarios:2017:maf}
\end{quote}

Chapter \ref{chapter:efree} is about \textbf{Sequential Neural Posterior Estimation (Type A)}, and is based on the following paper:

\begin{quote}
\bibentry{Papamakarios:2016:efree}
\end{quote}

Chapter \ref{chapter:snl} is about \textbf{Sequential Neural Likelihood}, and is based on the following paper:

\begin{quote}
\bibentry{Papamakarios:2019:snl}
\end{quote}

\part{Neural density estimation}
\label{part:nde}
\chapter{Masked Autoregressive Flow for Density Estimation}
\label{chapter:maf}

This chapter is devoted to \emph{density estimation}, and how we can use neural networks to estimate densities. We start by explaining what density estimation is, why it is useful, and what challenges it presents (section \ref{sec:maf:intro}). We then review some standard methods for density estimation, such as \emph{mixture models}, \emph{histograms} and \emph{kernel density estimators}, and introduce the idea of \emph{neural density estimation} (section \ref{sec:maf:methods}). The main contribution of this chapter is the paper \emph{Masked Autoregressive Flow of Density Estimation}, which presents \emph{Masked Autoregressive Flow}, a new model for density estimation (sections \ref{sec:maf:paper} and \ref{sec:maf:contrib}). We conclude the chapter by reviewing advances in neural density estimation since the publication of the paper (section \ref{sec:maf:flows}), and by discussing how neural density estimators fit more broadly in the space of generative models (section \ref{sec:maf:gen}).

\section{Density estimation: what and why?}
\label{sec:maf:intro}

Suppose we have a stationary process that generates data. The process could be of any kind, as long as it doesn't change over time: it could be a black-box simulator, a computer program, an agent acting, the physical world, or a thought experiment. Suppose that each time the process is run forward it independently generates a $D$-dimensional real vector $\vect{x}'$. We can think of the \emph{probability density} at $\vect{x}\in\R^D$ as a measure of how often the process generates data near $\vect{x}$ per unit volume. In particular, let $B_{\epsilon}\br{\vect{x}}$ be a ball centred at $\vect{x}$ with radius $\epsilon>0$, and let $\abs{B_{\epsilon}\br{\vect{x}}}$ be its volume. Informally speaking, the probability density at $\vect{x}$ is given by:
\begin{equation}
\prob{\vect{x}} = \frac{\Prob{\vect{x}'\in B_{\epsilon}\br{\vect{x}}}}{\abs{B_{\epsilon}\br{\vect{x}}}}
\quad\text{for }\epsilon\rightarrow 0,
\label{eq:maf:density_def}
\end{equation}
where $\Prob{\vect{x}'\in B_{\epsilon}\br{\vect{x}}}$ is the probability that the process generates data in the ball. For brevity, I will often just say \emph{density} and mean probability density.

A function $\prob{\cdot}$ that takes an arbitrary vector $\vect{x}$ and outputs the density at $\vect{x}$ is called a \emph{probability-density function}. To formally define the density function, suppose that $\Prob{\vect{x}'\in \mathcal{X}}$ is a probability measure defined on all Lebesgue-measurable subsets of $\R^D$. We require $\Prob{\vect{x}'\in \mathcal{X}}$ to be absolutely continuous with respect to the Lebesgue measure, which  means $\Prob{\vect{x}'\in \mathcal{X}}$ is zero if $\mathcal{X}$ has zero volume. A real-valued function $\prob{\cdot}$ is called a probability-density function if it has the following two properties:
\begin{align}
\prob{\vect{x}} &\ge 0\quad\text{for all }\vect{x}\in{\R^D},\\
\integralx{\prob{\vect{x}}}{\vect{x}}{\mathcal{X}} &= \Prob{\vect{x}'\in\mathcal{X}}\quad\text{for all Lebesgue-measurable }\mathcal{X}\subseteq\R^D.
\label{eq:maf:density_property_2}
\end{align}
The second property implies that a density function must integrate to $1$, that is:
\begin{equation}
\integralx{\prob{\vect{x}}}{\vect{x}}{\R^D} = 1.
\end{equation}
The \emph{Radon--Nikodym theorem} \citep[theorem 32.2]{Billingsley:1995:measure} guarantees that a density function exists and is unique almost everywhere, in the sense that any two density functions can only differ in a set of zero volume.

The density function is not defined if $\Prob{\vect{x}'\in \mathcal{X}}$ is not absolutely continuous. For example, this is the case if $\Prob{\vect{x}'\in \mathcal{X}}$ is discrete, i.e.~concentrated on a countable subset of $\R^D$. From now on we will assume that the density function always exists, but many of the techniques discussed in this thesis can be adapted for e.g.~discrete probability measures.

In practice, we often don't have access to the density function of the process we are interested in. Rather, we have a set of datapoints generated by the process (or the ability to generate such a dataset) and we would like to estimate the density function from the dataset. Hence, the problem of \emph{density estimation} can be stated as follows:

\begin{quote}
\emph{Given a set $\set{\vect{x}_1, \ldots, \vect{x}_N}$ of independently and identically generated datapoints, how can we estimate the probability density at an arbitrary location $\vect{x}$?}
\end{quote}

\subsection{Why estimate densities?}
\label{sec:maf:intro:why_densities}

Before I address the problem of \emph{how} to estimate densities, I will discuss the issue of \emph{why}. The question I will attempt to answer is the following:

\begin{quote}
\emph{Why is the density function useful, and why should we expend resources trying to estimate it?}
\end{quote}

To begin with, a model of the density function is a complete statistical model of the generative process. With the density function, we can (up to computational limitations) do the following:
\begin{enumerate}[label=(\roman*)]
\item Calculate the probability of any Lebesgue-measurable subset of $\R^D$ under the process, by integration using property \eqref{eq:maf:density_property_2}.
\item Sample new data using general-purpose sampling algorithms, such as Markov-chain Monte Carlo \citep{Neal:1993:mcmc, Murray:2007:mcmc}.
\item Calculate expectations under the process using integration:
\begin{equation}\avg{f\br{\vect{x}}} = \integralx{f\br{\vect{x}}p\br{\vect{x}}}{\vect{x}}{\R^D}.
\end{equation}
\item Test the density model against data generated from the actual process, using one-sample goodness-of-fit testing \citep{Chwialkowski:2016:gof, Liu:2016:gof, Jitkrittum:2017:gof}.
\end{enumerate}

Nevertheless, although the above are useful applications of the density function, I would argue that they don't necessitate modelling the density function. Indeed, one could estimate a \emph{sampling model} from the data, i.e.~a model that generates data in way that is statistically similar to the process. A sampling model can be used (again up to computational limitations) instead of a density model in all the above applications as follows:
\begin{enumerate}[label=(\roman*)]
\item The probability of a Lebesgue-measurable subset $\mathcal{X}\subset\R^D$ can be estimated by the fraction of samples falling in $\mathcal{X}$.
\item Sampling new data is trivial (and in practice usually much more efficient) with the sampling model.
\item Calculating expectations can be estimated by Monte-Carlo integration.
\item Testing the model against the actual process can be done with a two-sample test \citep{Gretton:2012:mmd}.
\end{enumerate}

So why invest in density models if there are other ways to model the desirable statistical properties of a generative process? In the following, I discuss some applications for which knowing the value of the density is important.

\subsubsection{Bayesian inference}

In the context of Bayesian inference, density functions encode degrees of belief. Bayes' rule describes how beliefs should change in light of new evidence as a operation over \emph{densities}. In particular, if we express our beliefs about a quantity $\bm{\theta}$ using a density model $\prob{\bm{\theta}}$ and the statistical relationship between $\vect{x}$ and $\bm{\theta}$ using a conditional density model $\prob{\vect{x}\g\bm{\theta}}$, we can calculate how our beliefs about $\bm{\theta}$ should change in light of observing $\vect{x}$ by:
\begin{equation}
\prob{\bm{\theta}\g\vect{x}} \propto \prob{\vect{x}\g\bm{\theta}}\, \prob{\bm{\theta}}.
\end{equation}

Being able to estimate densities such as the prior $\prob{\bm{\theta}}$, the likelihood $\prob{\vect{x}\g\bm{\theta}}$ and the posterior $\prob{\bm{\theta}\g\vect{x}}$ from data can be valuable for Bayesian inference. For example, density estimation on large numbers of unlabelled data can be used for constructing effective priors \citep{Zoran:2011:patches}. In cases where the likelihood is unavailable, density estimation on joint examples of $\vect{x}$ and $\bm{\theta}$ can be used to model the posterior or the likelihood \citep{Papamakarios:2016:efree, Lueckmann:2017:snpe, Papamakarios:2019:snl, Lueckmann:2018:maxvar}. A big part of this thesis is dedicated to using density estimation for likelihood-free inference, and this is a topic that I will examine thoroughly in the following chapters.

\subsubsection{Data compression}

There is a close relationship between density modelling and data compression. Suppose we want to encode a message $\vect{x}$ up to a level of precision defined by a small neighbourhood $B\br{\vect{x}}$ around $\vect{x}$. Assuming the message was generated from a distribution with density $\prob{\vect{x}}$, the information content associated with $\vect{x}$ at this level of precision is:
\begin{equation}
I\br{\vect{x}} = -\log\integralx{\prob{\vect{x'}}}{\vect{x'}}{B\br{\vect{x}}} \approx -\log\prob{\vect{x}} - \log \abs{B\br{\vect{x}}}.
\end{equation}
The above means that the density at $\vect{x}$ tells us how many bits an optimal compressor should use to encode $\vect{x}$ at a given level of precision (assuming base-$2$ logarithms). Conversely, a data compressor implicitly defines a density model. Given a perfect model of the density function, data compressors such as arithmetic coding \citep[section 6.2]{MacKay:2002:itila} can achieve almost perfect compression. On the other hand, if the data compressor uses (explicitly or implicitly) a density model $q\br{\vect{x}}$ that is not the same as the true model $\prob{\vect{x}}$, the average number of wasted bits is:
\begin{equation}
\avgx{-\log q\br{\vect{x}} - \log \abs{B\br{\vect{x}}}}{\prob{\vect{x}}} - \avgx{-\log\prob{\vect{x}}- \log \abs{B\br{\vect{x}}}}{\prob{\vect{x}}} = \kl{\prob{\vect{x}}}{q\br{\vect{x}}} > 0.
\end{equation}
Hence, a more accurate density model implies a more efficient data compressor, and vice versa.

\subsubsection{Model training and evaluation}

Even if we are not interested in estimating the density function per se, the density of the training data under the model is useful as an objective for training the model. For example, suppose we wish to fit a density model $q\br{\vect{x}}$ to training data $\set{\vect{x}_1, \ldots, \vect{x}_N}$ using maximum-likelihood estimation. The objective to be maximized at training time is:
\begin{equation}
L\br{q} = \frac{1}{N}\sum_n \log q\br{\vect{x}_n}.
\end{equation}
After training, being able to calculate $L\br{q}$ on a held-out test set is useful for ranking and comparing models. Hence, the usefulness of the density function as a training and evaluation objective motivates endowing our models with density-estimation capabilities even if we don't intend to use the model as a density estimator per se.

It is possible to formulate training and evaluation objectives that don't explicitly require densities, for example based on likelihood ratios \citep{Goodfellow:2014:gan}, kernel-space discrepancies \citep{Dziugaite:2015:mmdgan}, or Wasserstein distances \citep{Arjovsky:2017:wgan}. One argument in favour of the maximum-likelihood objective is its good asymptotic properties: a maximum-likelihood density estimator is \emph{consistent}, i.e.~it converges in probability to the density being estimated (assuming it's sufficiently flexible to represent it), and \emph{efficient}, i.e.~among all consistent estimators it attains the lowest mean squared error \citep[section 9.4]{Wasserman:2010:stats}. Moreover, in the limit of infinite data, maximizing $\frac{1}{N}\sum_n \log q\br{\vect{x}_n}$ is equivalent to minimizing $\kl{\prob{\vect{x}}}{q\br{\vect{x}}}$. In addition to its interpretation as a measure of efficiency of a data compressor which I already discussed, the KL divergence is the only divergence between density functions that possesses a certain set of properties, namely \emph{locality}, \emph{coordinate invariance} and \emph{subsystem independence}, as defined by \citet{Caticha:2004:entropy}. \citet{Rezende:2018:divergence} has argued that these properties of the KL divergence justify its use as a training and evaluation objective, and explain its popularity in machine learning.

\subsubsection{Density models as components of other algorithms}

Density models are often found as components of other algorithms in machine learning and statistics. For instance, sampling algorithms such as importance sampling and sequential Monte Carlo rely on an auxiliary model, known as the \emph{proposal}, which is required to provide the density of its samples \citep{Papamakarios:2015:distilling, Gu:2015:neural_smc, Paige:2016:inference_nets, Muller:2018:nis}. Variational autoencoders (which will be discussed in more detail in section \ref{sec:maf:gen:vaes}) require two density models as subcomponents, known as the \emph{encoder} and the \emph{prior} \citep{Rezende:2014:vae, Kingma:2014:vae}. Finally, in variational inference of continuous parameters, the approximate posterior is required to be a density model over the parameters of interest \citep{Ranganath:2014:bbvi, Kucukelbir:2015:autovistan}.

\subsection{Density estimation in high dimensions}
\label{sec:maf:intro:high_dim}

Estimating densities in high dimensions is a hard problem. Naive methods that work well in low dimensions often break down as the dimensionality increases (as I will discuss in more detail in section \ref{sec:maf:methods}). The observation that density estimation, as well as other machine-learning tasks, becomes dramatically harder as the dimensionality increases is often referred to as the \emph{curse of dimensionality} 
(\citeauthor{Hastie:2001:elements}, \citeyear{Hastie:2001:elements}, section 2.5; \citeauthor{Bishop:2006:prml}, \citeyear{Bishop:2006:prml}, section 1.4).

I will illustrate the curse of dimensionality for density estimation with a simple example. Consider a process that generates data uniformly in the $D$-dimensional unit cube $\left[0, 1\right]^D$. Clearly, the density $\prob{\vect{x}}$ is equal to $1$ for all $\vect{x}$ inside the cube. Suppose that we try to estimate $\prob{\vect{x}}$ for some $\vect{x}$ inside the cube using the fraction of training datapoints that fall into a small ball around $\vect{x}$. The expected fraction of training datapoints that fall into a ball $B_{\epsilon}\br{\vect{x}}$ centred at $\vect{x}$ with radius $\epsilon$ is no greater than the volume of the ball, which is given by:
\begin{equation}
\abs{B_{\epsilon}\br{\vect{x}}}
= \frac{\br{\pi^{\nicefrac{1}{2}}\epsilon}^D}{\Gamma\br{\frac{D}{2} + 1}},
\end{equation}
where $\Gamma\br{\cdot}$ is the Gamma function. In practice, we would need to make the ball large enough to contain at least one datapoint, otherwise the estimated density will be zero. However, no matter how large we make the radius $\epsilon$, the volume of the ball approaches zero as $D$ grows larger. In other words, in a dimension high enough, almost all balls will be empty, even if we make their radius larger than the side of the cube!

\begin{figure}[t]
\centering
\subfloat[Volume of a $D$-dimensional ball of radius $\epsilon$.\label{fig:maf:balls:a}]{\includegraphics[width=0.48\textwidth]{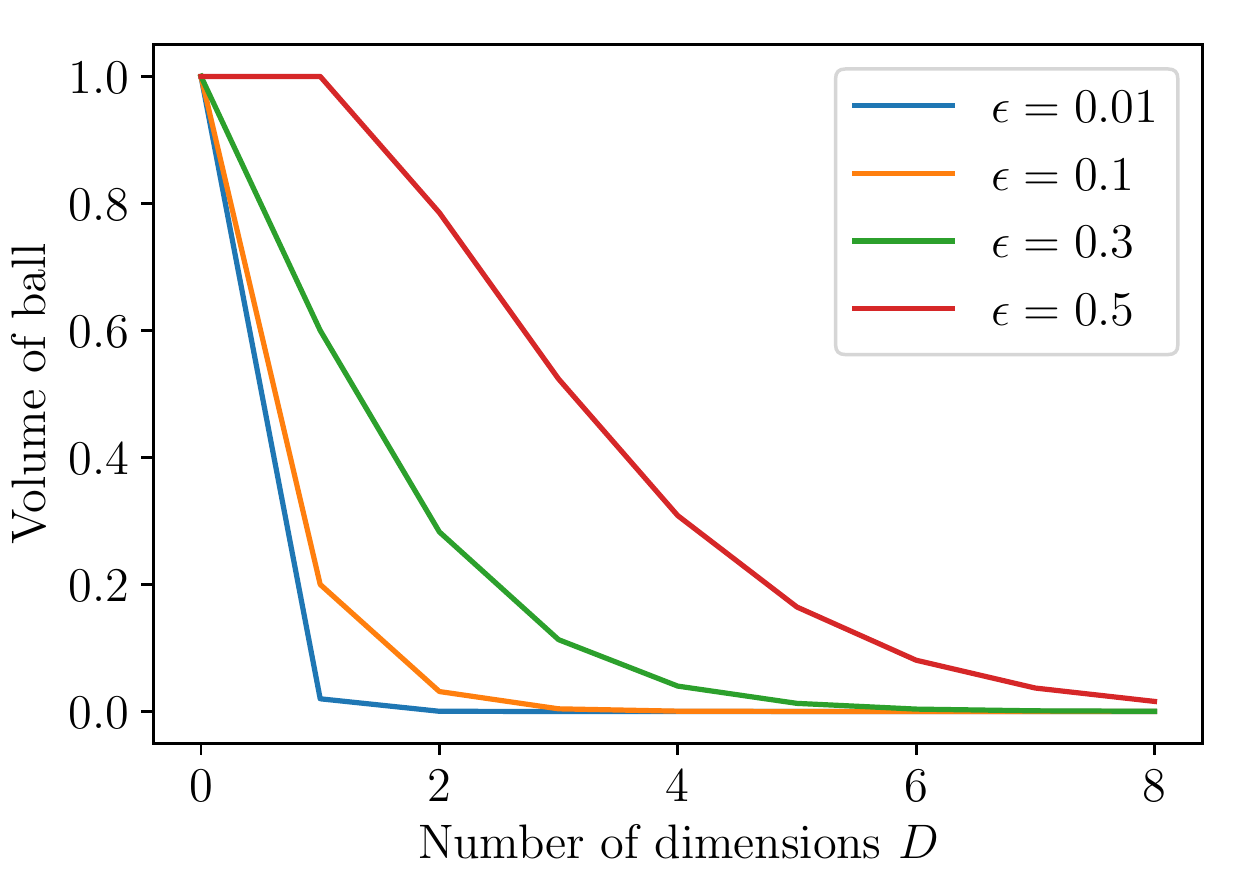}}
\hfill
\subfloat[Expected number of datapoints we need to generate until one falls in the ball.\label{fig:maf:balls:b}]
{\includegraphics[width=0.48\textwidth]{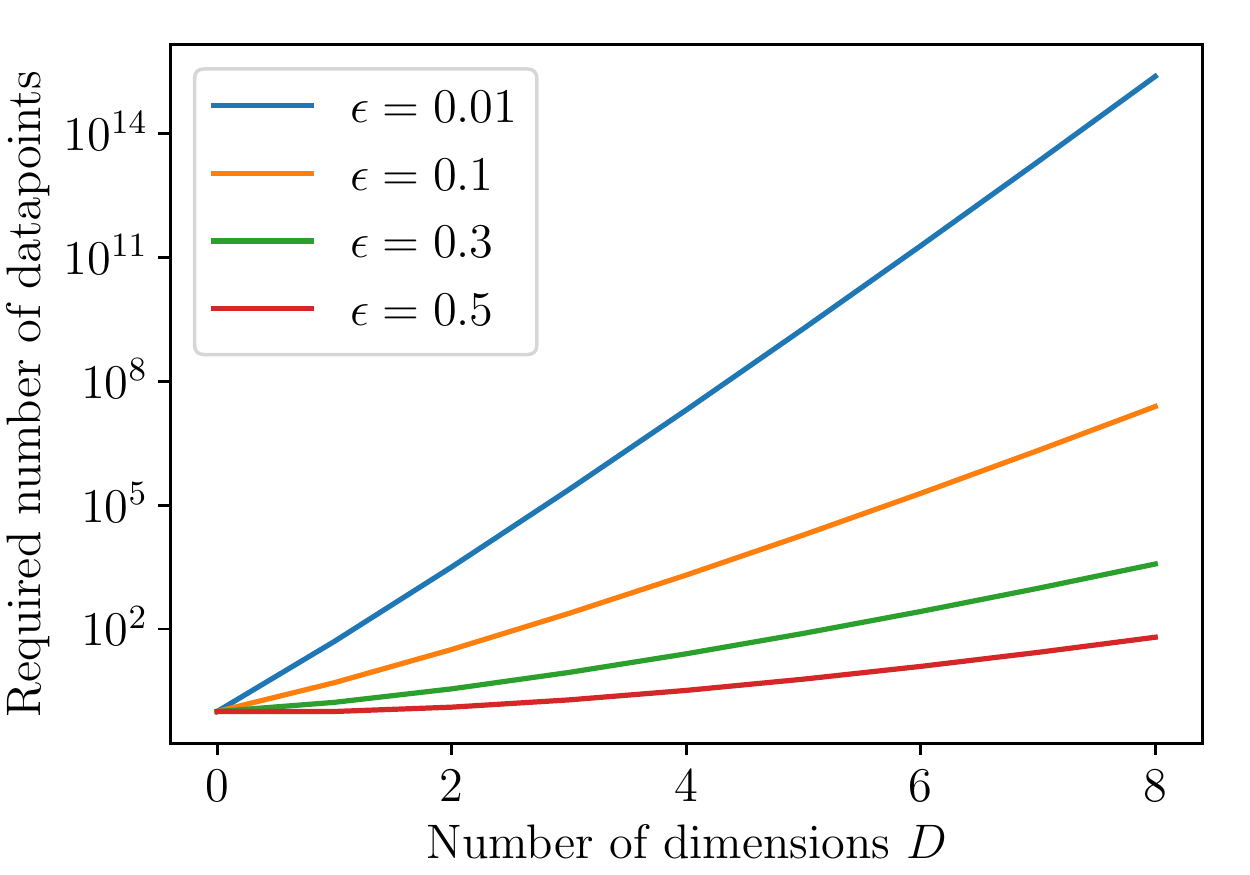}}
\label{fig:maf:balls}
\caption{Illustration of the curse of dimensionality for density estimation.}
\end{figure}

Figure \ref{fig:maf:balls:a} illustrates the shrinkage in volume of a $D$-dimensional ball as $D$ increases. Figure \ref{fig:maf:balls:b} shows the expected number of datapoints we would need to generate from the process until one of them falls into the ball, which is no less than the inverse of the ball's volume. As we can see, for a ball of radius $\epsilon=0.01$ and in $D=8$ dimensions, we would need a dataset of size at least a \emph{thousand trillion} datapoints!

In practice, in order to scale density estimation to high dimensions without requiring astronomical amounts of data, we make assumptions about the densities we wish to estimate and encode these assumptions into our models. With careful model design, it has been possible to train good density models on high-dimensional data such as high-resolution images \citep{Menick:2019:pixel_nets, Kingma:2018:glow, Dinh:2017:rnvp, Salimans:2017:pixelcnnpp} and audio streams \citep{VanDenOord:2016:WaveNet, Kim:2018:FloWaveNet, Prenger:2018:WaveGlow}. Fortunately, we can often make assumptions that are generic enough to apply to broad domains, while still making density estimation practical. In the following, I will examine a few generic assumptions that often guide model design.

\subsubsection{Smoothness}

We often assume that the density function varies smoothly, that is, if $\norm{\vect{x}_A -\vect{x}_B}$ is small, then $\abs{\prob{\vect{x}_A} -\prob{\vect{x}_B}}$ is also small. The assumption of smoothness encourages the model to interpolate over small regions that happen to have no training data due to sampling noise, rather than assign zero density to them. In practice, we enforce smoothness either by limiting the flexibility of the density model, or by regularizing it.

\subsubsection{Low intrinsic dimensionality}

We often assume that the world has fewer degrees of freedom than the measurements we make to describe it. For example, an image of a natural scene could only vary in certain semantically meaningful ways (e.g.~location of objects in the scene, direction of lighting, etc.), whereas each pixel of the image can't vary arbitrarily. This means that such natural images would approximately lie on a manifold of low intrinsic dimensionality embedded in the $D$-dimensional space of pixel intensities --- see also discussions by \citet{Basri:2003:lambertian} and \citet{Hinton:1997:manifolds} on the low-dimensional manifold structure of simple images. In practice, low-dimensional manifold structure can be modelled by limiting the degrees of freedom that the model can represent (e.g.~by introducing information bottlenecks in the model's structure).

\subsubsection{Symmetries and invariances}

Real data often have symmetries, some examples of which are listed below.
\begin{enumerate}[label=(\roman*)]
\item \emph{Translation symmetry}. An image of a car moved a few pixels to the right is still an image of the same car.
\item \emph{Scale symmetry}. An audio stream of music played twice as fast may still be a plausible piece of music.
\item \emph{Mirror symmetry}. An image flipped horizontally may still be a plausible image.
\item \emph{Order symmetry}. A dataset whose datapoints have been reordered is still the same dataset.
\end{enumerate}
Such symmetries in the data are often reflected in model design. For example, top-performing image models employ convolutions and multi-scale architectures \citep{Menick:2019:pixel_nets, Kingma:2018:glow, Dinh:2017:rnvp}, which equip the model with a degree of translation and scale symmetry. Additionally, symmetries in the data can be used for data augmentation: for instance, an image model can be trained with all training images flipped horizontally as additional training data \citep{Papamakarios:2017:maf, Dinh:2017:rnvp}.

\subsubsection{Independencies and loose dependencies}

Sometimes we know or suspect that certain measurements are either independent of or at least loosely dependent on other measurements. For example, in an audio stream, one could assume that the current audio intensity is only loosely dependent on audio intensities of more than a few seconds before it. This could be encoded in a model by limiting the range (or the receptive field) of each variable in the data, for example in the case of audio \citep{VanDenOord:2016:WaveNet} or pixel intensities \citep{VanDenOord:2016:PixelCNN}.

\section{Methods for density estimation}
\label{sec:maf:methods}

Methods for density estimation can broadly be classified as either \emph{parametric} or \emph{non-parametric}. Parametric methods model the density function as a specified functional form with a fixed number of tunable parameters. Non-parametric methods are those that don't fit the above description: typically, they specify a model whose complexity grows with the number of training datapoints. In this section I will discuss and compare some standard density-estimation methods from both categories, and I will introduce the idea of \emph{neural density estimation}.

\subsection{Simple parametric models and mixture models}
\label{sec:maf:methods:simple_parametric}

In parametric density estimation, we first specify a density model $q_{\bm{\phi}}\br{\vect{x}}$ with a fixed number of tunable parameters $\bm{\phi}$, and then we try to find a setting of $\bm{\phi}$ that makes $q_{\bm{\phi}}\br{\vect{x}}$ as similar as possible to the true density $\prob{\vect{x}}$. A straightforward approach is to choose $q_{\bm{\phi}}\br{\vect{x}}$ to be in a simple parametric family, for example the \emph{Gaussian family}:
\begin{equation}
q_{\bm{\phi}}\br{\vect{x}} = \frac{1}{\abs{\det\br{2\pi\bm{\Sigma}}}^{\nicefrac{1}{2}}}\exp\br{-\frac{1}{2}\br{\vect{x}-\bm{\mu}}^T\bm{\Sigma}^{-1}\br{\vect{x}-\bm{\mu}}}
\quad\text{where}\quad
\bm{\phi} = \set{\bm{\mu}, \bm{\Sigma}}.
\end{equation}
The parameters of the Gaussian family are a real $D$-dimensional vector $\bm{\mu}$ and a $D\times D$ symmetric positive-definite matrix $\bm{\Sigma}$. There are also special cases of the Gaussian family that further restrict the form of $\bm{\Sigma}$, such as the probabilistic versions of \emph{principal-components analysis} \citep{Tipping:1999:ppca}, \emph{minor-components analysis} \citep{Williams:2002:pmca}, and \emph{extreme-components analysis} \citep{Welling:2004:eca}.

The problem with simple parametric families such as the above is that the set of density functions they can represent is limited. For example, the Gaussian family can't represent density functions with more than one mode. One way to increase the expressivity of parametric models is to combine a number of models into one \emph{mixture model}. Let $q^{\br{1}}_{\bm{\phi}_1}\br{\vect{x}}, \ldots, q^{\br{K}}_{\bm{\phi}_K}\br{\vect{x}}$ be $K$ parametric models from the same or different families. A mixture model is a parametric model defined as:
\begin{equation}
q_{\bm{\phi}}\br{\vect{x}} = \sum_k \alpha_k \,q^{\br{k}}_{\bm{\phi}_k}\br{\vect{x}}
\quad\text{where}\quad
\sum_k \alpha_k = 1
\quad\text{and}\quad
\alpha_k \ge 0\,\,\text{ for all }k.
\end{equation}
The parameters of a mixture model are $\bm{\phi} = \set{\alpha_1, \bm{\phi}_1, \ldots, \alpha_K, \bm{\phi}_K}$.
Mixture models where all $q^{\br{k}}_{\bm{\phi}_k}\br{\vect{x}}$ are Gaussian are usually referred to as \emph{Gaussian mixture models}. Gaussian mixture models are a strong density-estimation baseline: with sufficiently many components they can approximate any density arbitrarily well \citep{McLachlan:1988:mixture_models}. However, they may require a large number of components to approximate density functions that can be expressed compactly in a different form (e.g.~a uniform density in the unit cube would require a large number of narrow Gaussians to approximate its steep boundary).

Parametric density models are typically estimated by maximum likelihood. Given a set of training datapoints $\set{\vect{x}_1, \ldots, \vect{x}_N}$ that have been independently and identically generated by a process with density $\prob{\vect{x}}$, we seek a setting of the model's parameters $\bm{\phi}$ that maximize the average log likelihood on the training data:
\begin{equation}
L\br{\bm{\phi}} = \frac{1}{N}\sum_n\log q_{\bm{\phi}}\br{\vect{x}_n}.
\end{equation}
From the strong law of large numbers, as $N\rightarrow\infty$ we have that $L\br{\bm{\phi}}$ converges almost surely to $\avgx{\log q_{\bm{\phi}}\br{\vect{x}}}{\prob{\vect{x}}}$. Hence, for a large enough training set, maximizing $L\br{\bm{\phi}}$ is equivalent to minimizing $\kl{\prob{\vect{x}}}{q_{\bm{\phi}}\br{\vect{x}}}$, since:
\begin{equation}
\kl{\prob{\vect{x}}}{q_{\bm{\phi}}\br{\vect{x}}} = -\avgx{\log q_{\bm{\phi}}\br{\vect{x}}}{\prob{\vect{x}}} + \mathrm{const}.
\end{equation}
Some of the merits of maximum-likelihood estimation and of KL-divergence minimization have already been discussed in section \ref{sec:maf:intro:why_densities}.

For certain simple models, the optimizer of $L\br{\bm{\phi}}$ has a closed-form solution. For example, the maximum-likelihood parameters of a Gaussian model are the empirical mean and covariance of the training data:
\begin{equation}
\bm{\mu}^* = \frac{1}{N}\sum_n{\vect{x}_n}
\quad\text{and}\quad
\bm{\Sigma}^* = \frac{1}{N}\sum_n{\br{\vect{x}_n - \bm{\mu}^*}\br{\vect{x}_n - \bm{\mu}^*}^T}.
\end{equation}
For mixture models based on simple parametric families, such as Gaussian mixture models, $L\br{\bm{\phi}}$ can be (locally) maximized using the \emph{expectation-maximization algorithm} \citep{Dempster:1977:em} or its online variant \citep{Cappe:2008:online_em}. More generally, if $L\br{\bm{\phi}}$ is differentiable with respect to $\bm{\phi}$, it can be (locally) maximized with gradient-based methods, as I will discuss in section \ref{sec:maf:methods:nde}.

\subsection{Histograms}

The histogram is one of the simplest and most widely used methods for density estimation. The idea of the histogram is to partition the data space into a set of non-overlapping bins $\set{B_1, \ldots, B_K}$, and estimate $\Prob{\vect{x}'\in B_k}$ by the fraction of training datapoints in $B_k$. Then, the density in $B_k$ is approximated by the estimate of $\Prob{\vect{x}'\in B_k}$ divided by the volume of $B_k$.

Histograms are often described as non-parametric models \citep[e.g.][section 2.5]{Bishop:2006:prml}. However, given the definition of a parametric model I gave earlier, I would argue that histograms are better described as parametric models that are trained with maximum likelihood. Given a partition of the data space into $K$ non-overlapping bins, a histogram is the following parametric model:
\begin{equation}
q_{\bm{\phi}}\br{\vect{x}} = \prod_k \pi_k^{I\br{\vect{x}\in B_k}}
\quad\text{where}\quad
\sum_k \pi_k\abs{B_k}=1
\quad\text{and}\quad
\pi_k \ge 0 \,\,\text{ for all }k.
\end{equation}
In the above, $I\br{\cdot}$ is the indicator function, which takes a logical statement and outputs $1$ if the statement is true and $0$ otherwise.
Each $\pi_k$ represents the density in bin $B_k$. The parameters of the histogram are $\bm{\phi} = \set{\pi_1, \ldots, \pi_K}$.

The average log likelihood of the histogram on training data $\set{\vect{x}_1, \ldots, \vect{x}_N}$ is:
\begin{equation}
L\br{\bm{\phi}} = \frac{1}{N}\sum_k N_k \log \pi_k,
\end{equation}
where $N_k = \sum_n I\br{\vect{x}_n\in B_k}$ is the number of training datapoints in $B_k$. Taking into account the equality constraint $\sum_k \pi_k\abs{B_k}=1$, we can write the Lagrangian of the maximization problem as:
\begin{equation}
\mathcal{L}\br{\bm{\phi}, \lambda} = L\br{\bm{\phi}} - \lambda\br{\sum_k \pi_k\abs{B_k}-1},
\end{equation}
where $\lambda$ is a Lagrange multiplier enforcing the equality constraint. Taking derivatives of the Lagrangian with respect to $\pi_k$ and $\lambda$ and jointly solving for zero, we find that the maximum-likelihood optimizer is:
\begin{equation}
\pi_k^* = \frac{N_k}{N\abs{B_k}},
\end{equation}
which also satisfies $\pi_k \ge 0$. As expected, the maximum-likelihood density in $B_k$ is the fraction of the training data in $B_k$ divided by the volume of $B_k$.

In practice, to construct the bins we typically grid up each axis between two extremes, and take the $D$-dimensional hyperrectangles formed this way to be the bins. How fine or coarse we grid up the space determines the volume of the bins and the granularity of the histogram. There is a bias-variance tradeoff controlled by bin volume: a histogram with too many bins of small volume may overfit, whereas a histogram with too few bins may underfit.

A drawback of histograms is that they suffer from the curse of dimensionality. To illustrate why, suppose we are trying to estimate a uniform density in the $D$-dimensional unit cube $\left[0, 1\right]^D$, and that we'd like a granularity of $K$ equally sized bins per axis. The total number of bins will be $K^D$, which is also the expected number of datapoints until one of them falls in a given bin. Hence, the amount of training data we'd need to populate the histogram scales exponentially with dimensionality. In practice, histograms are often used in low dimensions if there are enough datapoints (e.g.~for visualization purposes) but rarely in more than two or three dimensions.

\subsection{Kernel density estimation}

Kernel density estimation is a non-parametric method for estimating densities. A kernel density estimator can be thought of as a smoothed version of the empirical distribution of the training data. Given training data $\set{\vect{x}_1, \ldots, \vect{x}_N}$, their \emph{empirical distribution} $q_0\br{\vect{x}}$ is an equally weighted mixture of $N$ delta distributions located at training datapoints:
\begin{equation}
q_0\br{\vect{x}} = \frac{1}{N}\sum_n \dirac{\vect{x}-\vect{x}_n}.
\end{equation}
We can smooth out the empirical distribution and turn it into a density by replacing each delta distribution with a \emph{smoothing kernel}.
A smoothing kernel $k_{\epsilon}\br{\vect{u}}$ is a density function defined by:
\begin{equation}
k_{\epsilon}\br{\vect{u}} = \frac{1}{\epsilon^D}\,k_{1}\br{\frac{\vect{u}}{\epsilon}},
\end{equation}
where $\epsilon > 0$, and $k_1\br{\vect{u}}$ is a density function bounded from above. The parameter $\epsilon$ controls the ``width'' of the kernel; as $\epsilon \rightarrow 0$,  $k_{\epsilon}\br{\vect{u}}$ approaches $\dirac{\vect{u}}$. Given a smoothing kernel $k_{\epsilon}\br{\vect{u}}$, the \emph{kernel density estimator} is defined as:
\begin{equation}
q_{\epsilon}\br{\vect{x}} = \frac{1}{N}\sum_n k_{\epsilon}\br{\vect{x}-\vect{x}_n}.
\end{equation}
In practice, common choices of kernel include the Gaussian kernel:
\begin{equation}
k_{1}\br{\vect{u}} = 
\frac{1}{\br{2\pi}^{\nicefrac{D}{2}}}\exp\br{-\frac{1}{2}\norm{\vect{u}}^2},
\end{equation}
or the multiplicative Epanechnikov kernel:
\begin{equation}
k_{1}\br{\vect{u}} = \begin{cases}
\br{\frac{3}{4}}^D\prod_d \br{1-{u_d^2}} 
& \abs{u_d} \le 1\text{ for all }d \\
0 & \text{otherwise.}
\end{cases}
\end{equation}
The multiplicative Epanechnikov kernel is the most efficient among decomposable kernels, in the sense that asymptotically it achieves the lowest mean squared error \citep{Epanechnikov:1969:kernel}.

In the limit $\epsilon \rightarrow 0$, the kernel density estimator is \emph{unbiased}: it is equal to the true density in expectation. This is because as $\epsilon\rightarrow 0$ we have $q_{\epsilon}\br{\vect{x}}\rightarrow q_{0}\br{\vect{x}}$, and
\begin{equation}
\avg{q_{0}\br{\vect{x}}} = \frac{1}{N}\sum_n \avgx{\dirac{\vect{x}-\vect{x}_n}}{\prob{\vect{x}_n}}
=\avgx{\dirac{\vect{x}-\vect{x}'}}{\prob{\vect{x}'}}
=\prob{\vect{x}}.
\end{equation}
Moreover, the kernel density estimator is \emph{consistent}: it approaches the true density for small $\epsilon$ and large $N$, provided $\epsilon$ doesn't shrink too fast with $N$. To show this, we first upper-bound the variance of the estimator:
\begin{align}
\var{q_{\epsilon}\br{\vect{x}}} &= \frac{1}{N^2}\sum_n \varx{k_{\epsilon}\br{\vect{x}-\vect{x}_n}}{\prob{\vect{x}_n}}
= \frac{1}{N}\varx{k_{\epsilon}\br{\vect{x}-\vect{x}'}}{\prob{\vect{x}'}}\\
&\le \frac{1}{N}\avgx{k^2_{\epsilon}\br{\vect{x}-\vect{x}'}}{\prob{\vect{x}'}} 
= \frac{1}{N\epsilon^{2D}}\integralx{k^2_{1}\br{\frac{\vect{x}-\vect{x}'}{\epsilon}}\prob{\vect{x}'}}{\vect{x}'}{\R^D}\\
&\le \frac{\sup_{\vect{u}}{k^2_{1}\br{\vect{u}}}}{N\epsilon^{2D}}.
\end{align}
We can see that the variance approaches zero as $N\epsilon^{2D}$ approaches infinity. Hence, $q_{\epsilon}\br{\vect{x}}$ converges in probability to $\prob{\vect{x}}$ as $N$ approaches infinity, provided that $\epsilon$ approaches zero at a rate less than $N^{-\nicefrac{1}{2D}}$.

In practice, the width parameter $\epsilon$ controls the degree of smoothness, and trades off bias for variance: if $\epsilon$ is too low the model may overfit, whereas if $\epsilon$ is too high the model may underfit. In general, we want $\epsilon$ to be smaller the more data we have and larger the higher the dimension is; there are rules of thumb for setting $\epsilon$ based on $N$ and $D$ such as \emph{Scott's rule} \citep{Scott:1992:mde} or \emph{Silverman's rule} \citep{Silverman:1986:density}.

Sometimes it is not possible to find a value for $\epsilon$ that works equally well everywhere. For instance, a lower value may be more appropriate in regions with high concentration of training data than in regions with low concentration. One possible solution, known as the \emph{method of nearest neighbours}, is to choose a different $\epsilon$ for each location $\vect{x}$, such that the effective number of training datapoints contributing to the density at $\vect{x}$ is constant. However, the method of nearest neighbours doesn't always result in a normalizable density \citep[section 2.5.2]{Bishop:2006:prml}.

The kernel density estimator is widely used and a strong baseline in low dimensions due to its flexibility and good asymptotic properties. However it suffers from the curse of dimensionality in high dimensions. To illustrate why, consider estimating the uniform density in the unit cube $\left[0, 1\right]^D$ using the multiplicative Epanechnikov kernel, whose support is a $D$-dimensional hyperrectangle of side $2\epsilon$. The volume of space covered by kernels is at most $N\br{2\epsilon}^D$, which approaches zero as $D$ grows large for any $\epsilon < \nicefrac{1}{2}$. Hence, to avoid covering only a vanishing amount of space, we must either make the support of the kernel at least as large as the support of the entire density, or have the number of training datapoints grow at least exponentially with dimensionality.

Compared to parametric methods, kernel density estimation and non-parametric methods in general have the advantage that they don't require training: there is no need to search for a model because the training data \emph{is} the model. However, the memory cost of storing the model and the computational cost of evaluating the model grow linearly with $N$, which can be significant for large datasets. In contrast, parametric models have fixed memory and evaluation costs.

\subsection{Neural density estimation}
\label{sec:maf:methods:nde}

Neural density estimation is a parametric method for density estimation that uses neural networks to parameterize a density model. A \emph{neural density estimator} is a neural network with parameters $\bm{\phi}$ that takes as input a datapoint $\vect{x}$ and returns a real number $f_{\bm{\phi}}\br{\vect{\vect{x}}}$ such that:
\begin{equation}
\integralx{\exp\br{f_{\bm{\phi}}\br{\vect{\vect{x}}}}}{\vect{x}}{\R^D} = 1.
\end{equation}
The above constraint is enforced by construction, that is, the architecture of the neural network is such that $\exp\br{f_{\bm{\phi}}\br{\vect{x}}}$ integrates to $1$ for all settings of $\bm{\phi}$ (I will discuss how this can be achieved in section \ref{sec:maf:paper}). Since $\exp\br{f_{\bm{\phi}}\br{\vect{x}}}$ meets all the requirements of a density function, the neural network can be used as density model $q_{\bm{\phi}}\br{\vect{x}} = \exp\br{f_{\bm{\phi}}\br{\vect{x}}}$.

Given training data $\set{\vect{x}_1, \ldots, \vect{x}_N}$, neural density estimators are typically trained by maximizing the average log likelihood:
\begin{equation}
L\br{\bm{\phi}} = \frac{1}{N}\sum_n \log q_{\bm{\phi}}\br{\vect{x}_n} =  \frac{1}{N}\sum_n f_{\bm{\phi}}\br{\vect{x}_n}.
\end{equation}
The maximization of $L\br{\bm{\phi}}$ is typically done using a variant of \emph{stochastic-gradient ascent} \citep{Bottou:2012:sgd}. First, the parameters $\bm{\phi}$ are initialized to some arbitrary value. The algorithm proceeds in a number of iterations, in each of which $\bm{\phi}$ is updated. In each iteration, a subset of $M$ training datapoints $\set{\vect{x}_{n_1}, \ldots, \vect{x}_{n_M}}$, known as a \emph{minibatch}, is selected at random. The selection is usually done without replacement; if no more datapoints are left, all datapoints are put back in. Then, the gradient with respect to $\bm{\phi}$ of the average log likelihood on the minibatch is computed:
\begin{equation}
\nabla_{\bm{\phi}}\hat{L}\br{\bm{\phi}} = \frac{1}{M}\sum_m \nabla_{\bm{\phi}}f_{\bm{\phi}}\br{\vect{x}_{n_m}}.
\end{equation}
Each $\nabla_{\bm{\phi}}f_{\bm{\phi}}\br{\vect{x}_{n_m}}$ can be computed in parallel using \emph{reverse-mode automatic differentiation}, also known as \emph{backpropagation} in the context of neural networks \citep[section 6.5]{Goodfellow:2016:deeplearningbook}. The gradient $\nabla_{\bm{\phi}}\hat{L}\br{\bm{\phi}}$ is an unbiased estimator of $\nabla_{\bm{\phi}}L\br{\bm{\phi}}$, and is known as a \emph{stochastic gradient}. Finally, an \emph{ascent direction} $\vect{d}$ is computed based on its previous value, the total number of iterations so far, the current stochastic gradient, and possibly a window (or running aggregate) of previous stochastic gradients, and the parameters are updated by
$\bm{\phi} \leftarrow \bm{\phi} + \vect{d}$. There are various strategies for computing $\vect{d}$, such as \emph{momentum} \citep{Qian:1999:momentum},  \emph{AdaGrad} \citep{Duchi:2011:adagrad}, \emph{AdaDelta} \citep{Zeriler:2012:adadelta}, \emph{Adam} \citep{Kingma:2015:adam}, and \emph{AMSGrad} \citep{Reddi:2018:amsgrad}.

Stochastic-gradient ascent is a general algorithm for optimizing differentiable functions that can be written as averages of multiple terms. Due to its generality, it has the advantage that it decouples the task of modelling the density from the task of optimizing the training objective. Due to its use of stochastic gradients instead of full gradients, it scales well to large datasets, and it can be used with training datasets of infinite size (such as data produced by a generative process on the fly). Finally, there is some preliminary evidence that the stochasticity of the gradients may contribute in finding parameter settings that generalize well \citep{Keskar:2017:gen_gap}.

The question that remains to be answered is how we can design neural networks such that their exponentiated output integrates to $1$ by construction. This is one of the main contributions of this thesis, and it will be the topic of section \ref{sec:maf:paper}. As we shall see, it is possible to design neural density estimators that, although parametric, are flexible enough to approximate complex densities in thousands of dimensions.

\section{The paper}
\label{sec:maf:paper}

This section presents the paper \emph{Masked Autoregressive Flow for Density Estimation}, which is the main contribution of this chapter. The paper discusses state-of-the-art methods for constructing neural density estimators, and proposes a new method which we term \emph{Masked Autoregressive Flow}. We show how Masked Autoregressive Flow can increase the flexibility of previously proposed neural density estimators, and demonstrate MAF's performance in high-dimensional density estimation.

The paper was initially published as a preprint on arXiv in May 2017. Then, it was accepted for publication at the conference \emph{Advances in Neural Information Processing Systems (NeurIPS)} in December 2017. It was featured as an oral presentation at the conference; of the $3{,}240$ papers submitted to NeurIPS in 2017, $678$ were accepted for publication, of which $40$ were featured as oral presentations.

\subsubsection{Author contributions}

The paper is co-authored by me, Theo Pavlakou and Iain Murray. As the leading author, I conceived and developed Masked Autoregressive Flow, performed the experiments, and wrote the paper. Theo Pavlakou prepared the UCI datasets used in section 4.2 of the paper; his earlier work on density estimation using the UCI datasets served as a guide and point of reference for the experiments in the paper. Iain Murray supervised the project, offered suggestions, and helped revise the final version.

\subsubsection{Differences in notation}

The previous sections  used $\prob{\vect{x}}$ to mean the true density of the generative process and $q_{\bm{\phi}}\br{\vect{x}}$ to mean the density represented by a parametric model with parameters $\bm{\phi}$. The paper uses $\prob{\vect{x}}$  both for the true density and for the model density depending on context. In section 3.2 and appendix A of the paper, where disambiguation between the two densities is required, we use $\pi_x\br{\vect{x}}$ for the true density and $p_x\br{\vect{x}}$ for the model density.

\subsubsection{Corrections from original version}

About a year after the initial publication of the paper, we discovered that the experimental results on conditional density estimation with Masked Autoregressive Flow were incorrect due to an error in the code. Upon discovery, we corrected the results and issued a replacement of the paper on arXiv. The version included in this chapter is the corrected version of the paper, which was published on arXiv in June 2018. The particular results that were updated from earlier versions are indicated with a footnote in this version.

\includepdf[pages=-]{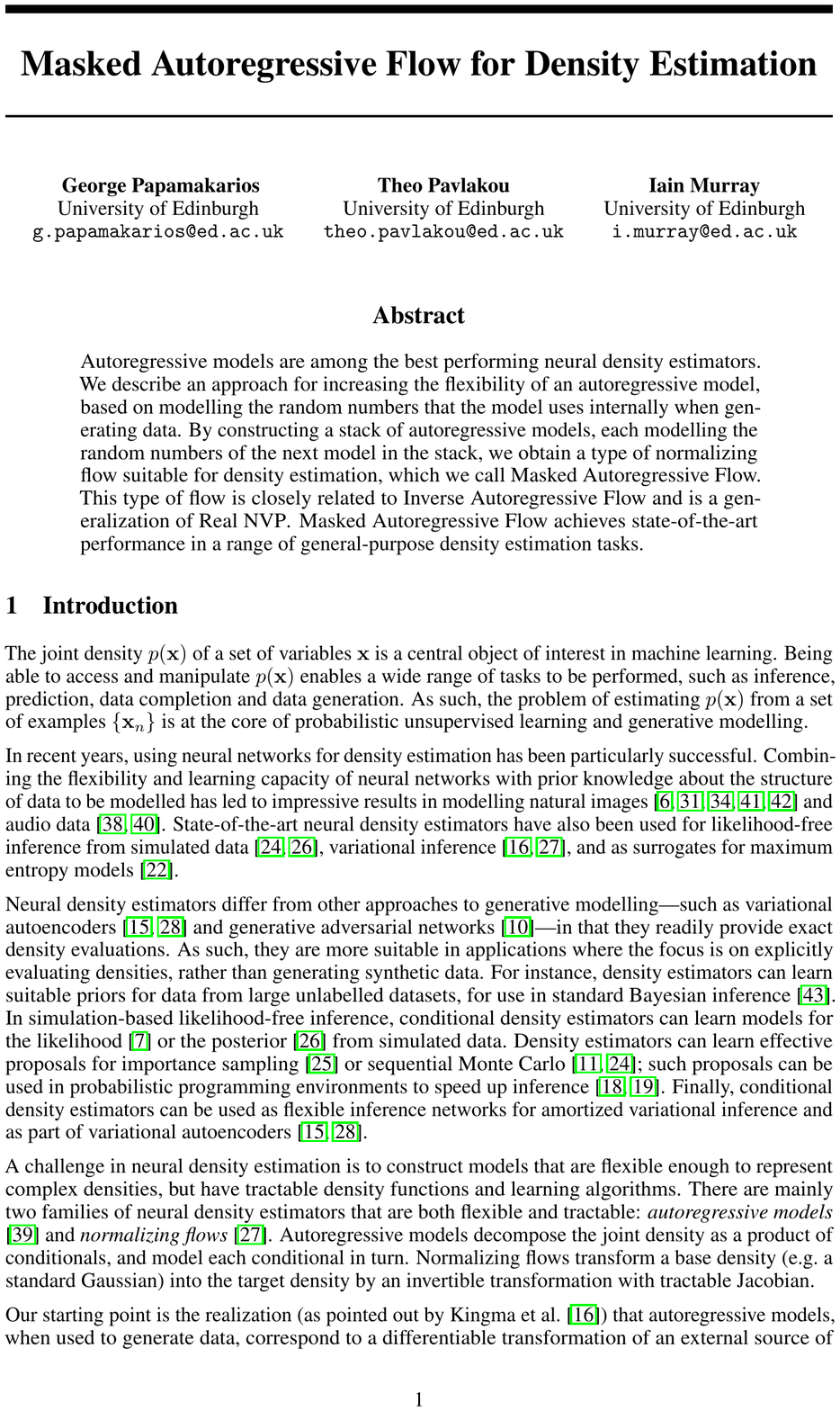}

\section{Contribution and impact}
\label{sec:maf:contrib}

The paper \emph{Masked Autoregressive Flow for Density Estimation} showed that we can increase the flexibility of simple neural density estimators by explicitly modelling their internal randomness. By reparameterizing a density model in terms of its internal randomness, we obtain an invertible transformation that can be composed into a normalizing flow. The resulting model can be more expressive than the original, while it remains tractable to train, evaluate and sample from. \emph{Masked Autoregressive Flow} is a specific implementation of this idea that uses masked autoregressive models with Gaussian conditionals as building blocks.

In addition to introducing a new density estimator, the paper contributed to the understanding of the relationship between MADE, MAF, IAF and Real NVP, and clarified the computational tradeoffs involved with using these models for density estimation and variational inference. Specifically, the paper explained the following relationships:
\begin{enumerate}[label=(\roman*)]
\item A MAF with one layer is a MADE with Gaussian conditionals.
\item An IAF is a MAF with its transformation inverted (and vice versa).
\item A Real NVP is a MAF/IAF with one autoregressive step instead of $D$ autoregressive steps.
\item Fitting a MAF to the training data by maximum likelihood can be viewed as fitting an implicit IAF to the base density by stochastic variational inference.
\end{enumerate}
Finally, the paper clarified the following computational tradeoffs:
\begin{enumerate}[label=(\roman*)]
\item MAF is fast to evaluate but slow to sample from.
\item IAF is slow to evaluate but fast to sample from.
\item Real NVP is fast to both evaluate and sample from, at the cost of decreased flexibility compared to MAF/IAF\@.
\end{enumerate}

According to Google Scholar, the paper has received $66$ citations as of \thedate. MAF has been used as a prior and/or decoder for variational autoencoders \citep{Alemi:2018:broken_elbo, Dillon:2017:tfd, Choi:2019:waic, Bauer:2019:resampled, Tran:2018:distr_pp, Vikram:2019:loracs}, for modelling state-action pairs in imitation learning \citep{Schroecker:2019:imitation}, and for estimating likelihoods or likelihood ratios in likelihood-free inference \citep{Brehmer:2018:mining, Papamakarios:2019:snl}. The five datasets we used for our experiments on unconditional density estimation (namely POWER, GAS, HEPMASS, MINIBOONE and BSDS300) have been made available online \citep{Papamakarios:2018:maf_datasets} and have been used by other researchers as density-estimation benchmarks \citep{Huang:2018:naf, Grathwohl:2018:ffjord, Oliva:2018:tan, Li:2018:glow, DeCao:2019:bnaf, Nash:2019:AEM}. Finally, MAF has been implemented as part of \emph{TensorFlow Probability} \citep{Dillon:2017:tfd}, a software library for probabilistic modelling and inference which forms part of TensorFlow \citep{Tensorflow:2015:whitepaper}.

There are two ways in which MAF is limited. First, MAF can be slow to sample from in high dimensions, as the computational cost of generating $D$-dimensional data from MAF scales linearly with $D$. This is true in general for autoregressive models, such as WaveNet \citep{VanDenOord:2016:WaveNet} or PixelCNN \citep{VanDenOord:2016:PixelCNN, Salimans:2017:pixelcnnpp}. Slow sampling limits the applicability of MAF as a neural sampler or as a variational posterior. Second, it is still unknown whether MAF is a universal density approximator, i.e.~whether it can model any well-behaved density arbitrarily well given enough layers and hidden units. In the next section, I will review advances in normalizing flows since the publication of the paper, and I will discuss further the tradeoffs between efficiency and expressivity in existing models.

\section{Further advances in normalizing flows}
\label{sec:maf:flows}

Since the publication of the paper \emph{Masked Autoregressive Flow for Density Estimation}, there has been a lot of research interest in normalizing flows. In this section, I review advances in normalizing flows after the publication of the paper, and I show how the various approaches are related to each other. Some of the advances are based on Masked Autoregressive Flow, and others are independent threads of research.

\subsection{Non-affine autoregressive layers}

MAF, IAF and Real NVP are all composed of \emph{affine autoregressive layers}, i.e.~autoregressive layers where each variable is scaled and shifted as a function of previous variables. (Since the coupling layers used by Real NVP are special cases of autoregressive layers, I won't make a distinction between the two from now on.) As we discussed in section~\ref{sec:maf:paper}, an affine autoregressive layer transforms each noise variable $u_i$ into a data variable $x_i$ as follows:
\begin{equation}
x_i = \alpha_i u_i + \beta_i,
\end{equation}
where $\alpha_i$ and $\beta_i$ are functions of $\vect{x}_{1:i-1}$ or $\vect{u}_{1:i-1}$. Restricting the transformation from $u_i$ to $x_i$ to be affine allowed us to invert it and compute the determinant of its Jacobian efficiently.

However, one could increase the expressivity of an autoregressive layer by allowing more general transformations from $u_i$ to $x_i$ of the form:
\begin{equation}
x_i = g_{\bm{\psi}_i}\br{u_i},
\end{equation}
where $\bm{\psi}_i$ is a function of $\vect{x}_{1:i-1}$ or $\vect{u}_{1:i-1}$ that parameterizes the transformation. As long as $g_{\bm{\psi}_i}\!$ is taken to be smooth and invertible, the resulting flow is a \emph{non-affine autoregressive flow}. The absolute determinant of the Jacobian of such a flow is:
\begin{equation}
\abs{\det\br{\deriv{f^{-1}}{\vect{x}}}} = \prod_i{\abs{\deriv{g_{\bm{\psi}_i}}{u_i}}^{-1}}.
\end{equation}

\subsubsection{Neural autoregressive flows}

\emph{Neural autoregressive flows} \citep{Huang:2018:naf} are non-affine autoregressive flows that use monotonically-increasing neural networks to parameterize $g_{\bm{\psi}_i}$. A feedforward neural network with one input $u_i$ and one output $x_i$ can be made monotonically increasing if (a) all its activation functions are monotonically increasing (sigmoid or leaky-ReLU activation functions have this property), and (b) all its weights are strictly positive. \citet{Huang:2018:naf} propose neural architectures that follow this principle, termed \emph{deep sigmoidal flows} and \emph{deep dense sigmoidal flows}. The derivative of $g_{\bm{\psi}_i}\!$ with respect to its input (needed for the computation of the Jacobian determinant above) can be obtained by automatic differentiation. A special case of a neural autoregressive flow is \emph{Flow++} \citep{Ho:2019:flowpp}, which parameterizes $g_{\bm{\psi}_i}\!$ as a mixture of logistic CDFs, and is equivalent to a neural network with positive weights and one hidden layer of logistic-sigmoid units.

The advantage of neural autoregressive flows is their expressivity. \citet{Huang:2018:naf} show that with a sufficiently flexible transformation $g_{\bm{\psi}_i}$, a single neural autoregressive layer can approximate any well-behaved density arbitrarily well. This is because if $g_{\bm{\psi}_i}\!$ becomes equal to the inverse CDF of the conditional $\prob{x_i\g \vect{x}_{1:i-1}}$, the neural autoregressive layer transforms the joint density $\prob{\vect{x}}$ into a uniform density in the unit cube \citep{Hyvarinen:1999:nonlinear_ica}.

The disadvantage of neural autoregressive flows is that in general they are not analytically invertible. That is, even though the inverse of  $g_{\bm{\psi}_i}\!$ exists, it's not always available in closed form. In order to invert the flow, one would have to resort to numerical methods. A neural autoregressive flow can still be used to estimate densities if it is taken to parameterize the transformation from $\vect{x}$ to $\vect{u}$, but if the transformation from $\vect{u}$ to $\vect{x}$ is not available analytically, it would not be possible to sample from the trained model efficiently.

\subsubsection{Non-linear squared flow}

If we take $g_{\bm{\psi}_i}\!$ to be non-affine but restrict it to have an analytic inverse, we would have a non-affine autoregressive flow that we could sample from. An example is the \emph{non-linear squared flow} \citep{Ziegler:2019:latent_flows}, which adds an inverse-quadratic perturbation to the affine transformation as follows:
\begin{equation}
g_{\bm{\psi}_i}\br{u_i} = \alpha_i u_i + \beta_i + \frac{\gamma_i}{1 + \br{\delta_i u_i + \epsilon_i}^2}
\quad\text{where}\quad
\bm{\psi}_i = \set{\alpha_i, \beta_i, \gamma_i, \delta_i, \epsilon_i}.
\end{equation}
The above transformation is not generally invertible, but it can be made monotonically increasing if we restrict $\alpha_i>\frac{9}{8\sqrt{3}}\abs{\gamma_i}\delta_i$ and $\delta_i>0$.
For $\gamma_i = 0$, the non-linear squared flow reduces to an affine autoregressive flow. 
Given $x_i$, the equation $x_i = g_{\bm{\psi}_i}\br{u_i}$ is a cubic polynomial with respect to $u_i$, so it can be solved analytically. The above transformation is more expressive than an affine transformation, but not as expressive as a general neural autoregressive flow.

\subsubsection{Piecewise-polynomial autoregressive flows}

Another approach to creating non-affine but analytically invertible autoregressive flows is to parameterize $g_{\bm{\psi}_i}\!$ as a piecewise-linear or piecewise-quadratic monotonically-increasing function \citep{Muller:2018:nis}. In this case, the parameters $\bm{\psi}_i$ correspond to the locations of the segments and their shape (i.e.~slope and curvature). In order to invert $g_{\bm{\psi}_i}\!$ for a given $x_i$, one would need to first identify which segment this $x_i$ corresponds to (which can be done by binary search since the segments are sorted), and then invert that segment (which is easy for a linear or quadratic segment). The more segments we have, the more flexible the transformation becomes.

\subsection{Invertible convolutional layers}

If we are interested in modelling image data, we may want to design a flow that contains invertible convolutional layers. For this discussion, we will assume that the data is an image of shape $H\times W\times C$, where $H$ is the height, $W$ is the width, and $C$ is the number of channels. A convolutional layer transforms noise of shape $H\times W\times C$ into data via a convolution with filter $k$. Let $\vect{x}$ and $\vect{u}$ represent the vectorized image and noise respectively. Since convolution is a linear operation, we can write it as the following matrix multiplication:
\begin{equation}
{\vect{x}} = \mat{W}_{k} \,\vect{u},
\end{equation}
where $\mat{W}_{k}$ is a matrix of shape $HWC\times HWC$ whose entries depend on the filter ${k}$. If $\mat{W}_{k}$ is invertible then the convolution is invertible, and its Jacobian has absolute determinant:
\begin{equation}
\abs{\det\br{\deriv{f^{-1}}{\vect{x}}}} = \abs{\det\br{\mat{W}_{k}}}^{-1}.
\end{equation}
Nonetheless, naively inverting $\mat{W}_{k}$ or calculating its determinant has a cost of $\bigo{\br{HWC}^3}$, so more scalable solutions have to be found in practice.

\subsubsection{Invertible $1\times 1$ convolutions and Glow}

\citet{Kingma:2018:glow} introduced \emph{invertible $1\times 1$ convolutions} in their model \emph{Glow}. An invertible $1\times 1$ convolution is essentially a linear transformation where each pixel of size $C$ (with one value for each channel) is multiplied by the same matrix $\mat{V}$ of shape $C\times C$. The equivalent matrix $\mat{W}_{k}$ can be obtained by:
\begin{equation}
\mat{W}_{k} = \mat{V} \otimes \mat{I},
\end{equation}
where $\mat{I}$ is the identity matrix of shape $HW\times HW$ and $\otimes$ is the Kronecker product. The inverse and determinant of $\mat{W}_{k}$ have a cost of  $\bigo{C^3}$, which may not be prohibitive for moderate $C$. To further reduce the cost, \citet{Kingma:2018:glow} suggest parameterizing $\mat{V}$ as follows:
\begin{equation}
\mat{V} = \mat{P}\mat{L}\mat{U},
\end{equation}
where $\mat{P}$ is a fixed permutation matrix, $\mat{L}$ is a lower triangular matrix with ones in its diagonal, and $\mat{U}$ is an upper triangular matrix. In that case, the absolute determinant of $\mat{W}_{k}$ becomes:
\begin{equation}
\abs{\det\br{\mat{W}_{k}}} = HW\prod_c \abs{U_{cc}},
\end{equation}
where $U_{cc}$ is the $c$-th element of $\mat{U}$'s diagonal. If we further restrict every $U_{cc}$ to be positive, we guarantee that the transformation is always invertible. In addition to modelling images, invertible $1\times 1$ convolutions have been used in modelling audio by \emph{WaveGlow} \citep{Prenger:2018:WaveGlow} and \emph{FloWaveNet} \citep{Kim:2018:FloWaveNet}.

\subsubsection{Autoregressive and emerging convolutions}

One way of obtaining scalable invertible convolutions without restricting the receptive field to be $1\times 1$ is via \emph{autoregressive convolutions} \citep{Hoogeboom:2019:emerging}. In an autoregressive convolution, pixels are assumed to be ordered, and part of the filter is zeroed out so that output pixel $i$ only depends on input pixels $1$ to $i-1$\@. An autoregressive convolution corresponds to a triangular matrix $\mat{W}_{k}$, and hence its determinant can be calculated at a cost of $\bigo{HWC}$. A convolution that is not restricted to be autoregressive can be obtained by composing an autoregressive convolution whose matrix $\mat{W}_{{k}_1}$ is upper triangular with an autoregressive convolution whose matrix $\mat{W}_{{k}_2}$ is lower triangular; the result is equivalent to a non-autoregressive convolution with matrix $\mat{W}_{{k}_2}\mat{W}_{{k}_1}$. This is analogous to parameterizing an LU decomposition of the convolution matrix, and is termed \emph{emerging convolution} \citep{Hoogeboom:2019:emerging}.

\subsubsection{Periodic convolutions in the Fourier domain}

Finally, another way of scaling up invertible convolutions is via the Fourier domain. According to the convolution theorem, the convolution between a filter $k$ and a signal $u$ is equal to:
\begin{equation}
{k} * {u} = \mathcal{F}^{-1}\br{\mathcal{F}\br{{k}}\,\mathcal{F}\br{{u}}},
\end{equation}
where $\mathcal{F}$ is the Fourier transform and $\mathcal{F}^{-1}$ is its inverse. Since the Fourier transform is a unitary linear operator, its discrete version corresponds to multiplication with a particular unitary matrix $\mat{F}$. Hence, the convolution of a discrete signal $\vect{u}$ can be written in vectorized form as:
\begin{equation}
\mat{W}_k\,\vect{u} = \mat{F}^T\br{\mat{F}\vect{k} \odot \mat{F}\vect{u}} =
\br{\mat{F}^T \mat{D}_k\,\mat{F}}\,\vect{u}
\end{equation}
where $\odot$ is elementwise multiplication, and $\mat{D}_k$ is a diagonal matrix whose diagonal is $\mat{F}\vect{k}$. Since the absolute determinant of $\mat{F}$ is $1$, the absolute determinant of $\mat{W}_k$ is:
\begin{equation}
\abs{\det\br{\mat{W}_{k}}} = \prod_i \abs{D_{k,ii}},
\end{equation}
where $D_{k,ii}$ is the $i$-th element of the filter $k$ expressed in the Fourier domain.

In a typical convolution layer with $C$ filters and $C$ input channels, we perform a total of $C^2$ convolutions (one for each combination of filter and input-channel), and the resulting $C^2$ output maps are summed across input channels to obtain $C$ output channels. If we express the entire convolution layer in the Fourier domain as we did above, we obtain a block-diagonal matrix $\mat{D}_k$ of shape $HWC \times HWC$ whose diagonal contains $HW$ matrices of shape $C\times C$ \citep{Hoogeboom:2019:emerging}. Hence, calculating the determinant of a convolution layer using its Fourier representation has a cost of $\bigo{HWC^3}$, which may be acceptable for moderate $C$. This type of convolution layer is termed \emph{periodic convolution} \citep{Hoogeboom:2019:emerging}.

\subsection{Invertible residual layers}

A residual layer \citep{He:2016:resnet} is a transformation of the following form:
\begin{equation}
\vect{x} = f\br{\vect{u}} =  \vect{u} + g\br{\vect{u}}.
\end{equation}
Residual layers are designed to avoid vanishing gradients in deep neural networks. Since the Jacobian of a residual layer is:
\begin{equation}
\deriv{f}{\vect{u}} = \mat{I} + \deriv{g}{\vect{u}},
\end{equation}
the propagated gradient doesn't vanish even if the Jacobian of $g$ does.

\subsubsection{Sylvester flow}

Residual networks are not generally invertible, but can be made to be if $g$ is restricted accordingly. One such example is the \emph{Sylvester flow} \citep{Berg:2018:sylvester}, where $g$ is taken to be:
\begin{equation}
g\br{\vect{u}} = \mat{Q}\mat{R}\,h\br{\mat{\tilde{R}}\mat{Q}^T\vect{u} + \vect{b}}.
\end{equation}
In the above, $\mat{Q}$ is a $D\times M$ matrix whose columns form an orthonormal basis, $\mat{R}$ and $\mat{\tilde{R}}$ are $M\times M$ upper-triangular matrices, $\vect{b}$ is an $M$-dimensional bias, $h\br{\cdot}$ is a smooth activation function applied elementwise, and $M\le D$. The above transformation can be thought of as a feedforward neural network with one hidden layer of $M$ units, whose weight matrices have been parameterized in a particular way.

Theorem 2 of \citet{Berg:2018:sylvester} gives sufficient conditions for the invertibility of the Sylvester flow. However, their proof requires $\tilde{\mat{R}}$ to be invertible, which as I show here is not necessary. In theorem \ref{theorem:maf:sylvester_invertibility} below, I give a more succinct proof of the invertibility of the Sylvester flow, which doesn't require the invertibility of $\tilde{\mat{R}}$.

\vspace{1em}
\begin{theorem}[Invertibility of the Sylvester flow]
\label{theorem:maf:sylvester_invertibility}
The Sylvester flow is invertible if $h\br{\cdot}$ is monotonically increasing with bounded derivative (e.g.~a sigmoid activation function has this property), and if for all $m$ we have:
\begin{equation}
\tilde{R}_{mm}R_{mm} > -\frac{1}{\sup_z{h'\br{z}}}.
\label{eq:maf:sylvester_inv}
\end{equation}
\end{theorem}

\begin{proof}
Using the \emph{matrix-determinant lemma}, the determinant of the Jacobian of $f$ can be written as:
\begin{equation}
\det\br{\deriv{f}{\vect{u}}} = \det\br{\mat{I} + \mat{A}\mat{\tilde{R}}\mat{R}},
\end{equation}
where $\mat{A}$ is an $M\times M$ diagonal matrix whose diagonal is $h'\br{\mat{\tilde{R}}\mat{Q}^T\vect{u} + \vect{b}}$. Since $\mat{I} + \mat{A}\mat{\tilde{R}}\mat{R}$ is an $M\times M$ upper-triangular matrix, its determinant is the product of its diagonal elements, hence:
\begin{equation}
\det\br{\deriv{f}{\vect{u}}} = \prod_m
\br{1 + A_{mm}\tilde{R}_{mm}R_{mm}}.
\end{equation}
Condition \eqref{eq:maf:sylvester_inv} ensures that the Jacobian determinant is positive everywhere. Hence, from the \emph{inverse-function theorem} it follows that $f$ is invertible.
\end{proof}

Sylvester flows are related to other normalizing flows. For $M=1$, the Sylvester flow becomes a special case of the \emph{planar flow} \citep{Rezende:2015:flows}. Furthermore, if $M=D$, $\mat{Q}$ is taken to be a reverse-permutation matrix and $\mat{R}$ is strictly upper triangular, $g$ becomes a MADE \citep{Germain:2015:made} with one hidden layer of $D$ units. In that case, the Sylvester flow becomes a special case of the \emph{Inverse Autoregressive Flow} \citep{Kingma:2016:iaf}, where the scaling factor is $1$ and the shifting factor is $g\br{\vect{u}}$.

Even though the Sylvester flow is invertible and has a tractable Jacobian, its inverse is not available analytically. This means that the Sylvester flow can calculate efficiently only the density of its own samples, so it can be used as a variational posterior. Alternatively, if we take it to parameterize the transformation from $\vect{x}$ to $\vect{u}$, we can use it as a density estimator but we won't be able to sample from it efficiently.

\subsubsection{Contractive residual layers and iResNet}

In general, a residual layer is invertible if $g$ is a \emph{contraction}, i.e.~its Lipschitz constant is less than~$1$. A residual layer with this property is termed an \emph{iResNet} \citep{Behrmann:2018:iresnet}. To see why an iResNet is invertible, fix a value for $\vect{x}$ and consider the sequence $\vect{u}_1, \vect{u}_2, \ldots$ obtained by:
\begin{equation}
\vect{u}_{k + 1} = \vect{x} - g\br{\vect{u}_k}.
\label{eq:maf:iresnet_fixed_point}
\end{equation}
The map $\vect{u}_k\mapsto\vect{u}_{k+1}$ is a contraction, so by the \emph{Banach fixed-point theorem} it follows that the sequence converges for any choice of $\vect{u}_1$ to the same fixed point $\vect{u}_{\infty}$. That fixed point is the unique value that satisfies $\vect{x} = f\br{\vect{u}_{\infty}}$, which proves the invertibility of $f$.

One way to construct an iResNet is to parameterize $g$ to be a feedforward neural network with contractive activation functions (such as sigmoids or ReLUs), and with weight matrices of spectral norm less than $1$. Even though its inverse is not analytically available in general, an iResNet can be numerically inverted using the iterative procedure of equation \eqref{eq:maf:iresnet_fixed_point}.

Directly calculating the Jacobian determinant of an iResNet costs $\bigo{D^3}$. Alternatively, following \citet{Behrmann:2018:iresnet},
the log absolute determinant of the Jacobian can be first expanded into a power series as follows:
\begin{equation}
\log{\abs{\det\br{\deriv{f}{\vect{u}}}}} = 
\sum_{k=1}^{\infty}\frac{\br{-1}^{k+1}}{k}\trace\br{\mat{J}_g^k}
\quad\text{where}\quad
\mat{J}_g = \deriv{g}{\vect{u}},
\end{equation}
and then it can be approximated by truncating the power series at a desired accuracy. Further, $\trace\br{\mat{J}_g^k}$ can be approximated by the following unbiased stochastic estimator, known as the \emph{Hutchinson estimator} \citep{Hutchinson:1990:estimator}:
\begin{equation}
\trace\br{\mat{J}_g^k} \approx \vect{v}^T\mat{J}_g^k\vect{v}
\quad\text{where}\quad
\vect{v}\sim\gaussian{\vect{0}}{\mat{I}}.
\label{eq:maf:hutchinson}
\end{equation}
Calculating $\vect{v}^T\mat{J}_g^k\vect{v}$ doesn't require explicitly computing the Jacobian, as it can be done by backpropagating through $g$ a total of $k$ times, once for each Jacobian-vector product.

The main advantage of iResNets over other types of invertible residual networks is the flexibility in constructing $g$. Unlike Sylvester flows where $g$ is restricted to one hidden layer and to no more than $D$ hidden units, an iResNet can have any number of layers and hidden units. However, unlike other normalizing flows, it is expensive both to sample and to calculate exact densities under an iResNet, which limits its applicability in practice.

\subsection{Infinitesimal flows}

So far I have discussed normalizing flows consisting of a fixed number of layers. We can imagine a flow where the number of layers grows larger and larger, but at the same time the effect of each layer becomes smaller and smaller. In the limit of infinitely many layers each of which has an infinitesimal effect, we obtain an flow where $\vect{u}$ is transformed continuously, rather than in discrete steps. We call such flows \emph{infinitesimal flows}.

\subsubsection{Deep diffeomorphic flow}

One type of infinitesimal flow is the \emph{deep diffeomorphic flow} \citep{Salman:2018:diffeomorphic}. We start with a residual layer mapping $\vect{u}_t$ to $\vect{u}_{t+1}$:
\begin{equation}
\vect{u}_{t+1} = \vect{u}_t + g_t\br{\vect{u}_t},
\end{equation}
and then extend it to a variable-sized step as follows:
\begin{equation}
\vect{u}_{t+dt} = \vect{u}_t + dt\, g_t\br{\vect{u}_t}.
\label{eq:maf:resnet_dt}
\end{equation}
In the limit $dt\rightarrow 0$, we obtain the following ordinary differential equation:
\begin{equation}
\tderiv{\vect{u}_t}{t} = g_t\br{\vect{u}_t}.
\end{equation}
In the above ODE, we can interpret $g_t$ as a time-varying velocity field. For small enough $dt$, we can interpret a residual layer of the form of equation \eqref{eq:maf:resnet_dt} as the \emph{Euler integrator} of this ODE\@. In that case, we can approximately invert the flow by running the integrator backwards:
\begin{equation}
\vect{u}_{t} \approx \vect{u}_{t+dt} - dt\, g_t\br{\vect{u}_{t+dt}},
\end{equation}
which becomes exact for $dt\rightarrow 0$. We can also think of a diffeomorphic flow as a special case of an iResNet, where $dt$ is small enough such that the transformation is contractive. In that case, each step of running the integrator backwards corresponds to a single iteration of the iResNet-inversion algorithm, where $\mathbf{u}_t$ is initialized with $\mathbf{u}_{t+dt}$, the value to be inverted at each step.

Since the deep diffeomorphic flow is a special case of an iResNet, calculating its Jacobian determinant can be done similarly. Using a small (but not infinitesimal) $dt$, we begin by writing the log absolute determinant of each step of the Euler integrator as a power series:
\begin{equation}
\log{\abs{\det\br{\deriv{\vect{u}_{t+dt}}{\vect{u}_t}}}} = 
\sum_{k=1}^{\infty}\frac{\br{-1}^{k+1}}{k}\br{dt}^k\trace\br{\mat{J}_{g_t}^k}
\quad\text{where}\quad
\mat{J}_{g_t} = \deriv{g_t}{\vect{u}_t},
\end{equation}
and then approximate it by truncating the power series at a desired level of accuracy. The smaller we take $dt$ the more we can afford to truncate. Finally, we can estimate $\trace\br{\mat{J}_{g_t}^k}$ using the Hutchinson estimator in equation \eqref{eq:maf:hutchinson}.

The main drawback of the deep diffeomorphic flow is that inverting it and computing its Jacobian determinant are approximate operations that can introduce approximation error. To make the approximation more accurate, one needs to reduce $dt$, which in turn can make the flow prohibitively deep, since the number of layers scales as $\bigo{\nicefrac{1}{dt}}$.

\subsubsection{Neural ODEs and FFJORD}

The drawbacks of the deep diffeomorphic flow stem from the fact that (a) it backpropagates through the integrator, (b) the Euler integrator it uses is not exactly invertible, and (c) the calculation of the Jacobian determinant is approximate. A different approach that avoids these issues is \emph{Neural ODEs} \citep{Chen:2018:neural_odes}. Instead of using the Euler integrator and backpropagating through it, Neural ODEs defines an additional ODE that describes the evolution of the gradient of a loss as it backpropagates through the flow. Hence, instead of backpropagating through the integrator in order to train the flow, the gradients with respect to its parameters can be obtained as the solution to this additional ODE\@. Both the ODE for the forward pass and the ODE for the backward pass can be solved with any integrator, which allows one to use an integrator that is exactly invertible. Furthermore, the integrator can choose the number of integration steps (or layers of the flow) adaptively.

Instead of approximating the Jacobian determinant by truncating its power series, we can define directly how the log density evolves through the flow via a third ODE\@. We start by rewriting the power series of the log absolute determinant as follows:
\begin{equation}
\log{\abs{\det\br{\deriv{\vect{u}_{t+dt}}{\vect{u}_t}}}} = 
dt\,\trace\br{\mat{J}_{g_t}} + \bigo{{dt}^2}.
\end{equation}
Using the above, we can express the evolution of the log density from $t$ to $t+dt$ as:
\begin{equation}
\log q_{t+dt}\br{\vect{u}_{t+dt}} = \log q_{t}\br{\vect{u}_{t}}
-dt\,\trace\br{\mat{J}_{g_t}} + \bigo{{dt}^2}.
\end{equation}
Taking the limit $dt\rightarrow 0$, we obtain:
\begin{equation}
\frac{d}{dt}\log q_{t}\br{\vect{u}_{t}} = -\trace\br{\mat{J}_{g_t}} = -\nabla\cdot g_t\br{\vect{u}_t}.
\end{equation}
The above ODE is a special case of the \emph{Fokker--Planck equation} where the diffusion is zero. As before, $\trace\br{\mat{J}_{g_t}}$ can be approximated using the Hutchinson estimator. The above ODE can be solved together with the ODEs for forward and backward propagation through the flow with the integrator of our choice. The resulting flow is termed \emph{FFJORD}
\citep{Grathwohl:2018:ffjord}. A special case of FFJORD is the \emph{Monge--Amp\`{e}re flow}
\citep{Zhang:2018:monge_ampere}, which takes $g_t$ to be the gradient of a scalar field (in which case $g_t$ has zero rotation everywhere).

The advantage of Neural ODEs and FFJORD is their flexibility. Unlike other normalizing flows where the architecture has to be restricted in some way or another, $g_t$ can be parameterized by any neural network, which can depend on $\vect{u}_t$ and $t$ in arbitrary ways. On the other hand, their disadvantage is that they necessitate using an ODE solver for sampling and calculating the density under the model, which can be slow compared to a single pass through a neural network. Additionally, unlike a neural network which has a fixed evaluation time, the evaluation time of an ODE solver may vary depending on the value of its input (i.e.~its initial state).

\section{Generative models without tractable densities}
\label{sec:maf:gen}

So far in this chapter I have focused on neural density estimators, that is, models whose density can be calculated efficiently. However, neural density estimation is not the only approach to generative modelling. Before I conclude the chapter, I will take a step back and briefly review generative models without tractable densities. I will discuss how such models relate to neural density estimators, how they differ from them, and what alternative capabilities they may offer.

\subsection{Energy-based models}

In section \ref{sec:maf:methods:nde}, we defined a neural density estimator to be a neural network that takes a vector $\vect{x}$ and returns a real number $f_{\bm{\phi}}\br{\vect{x}}$ such that for any parameter setting $\bm{\phi}$:
\begin{equation}
\integralx{\exp\br{f_{\bm{\phi}}\br{\vect{\vect{x}}}}}{\vect{x}}{\R^D} = 1.
\end{equation}
The above property allows us to interpret $q_{\phi}\br{\vect{x}} = \exp\br{f_{\bm{\phi}}\br{\vect{\vect{x}}}}$ as a density function. In order to enforce this property, we had to restrict the architecture of the neural network, which as we saw can hurt the flexibility of the model.

We can relax this restriction by requiring only that the integral of $\exp\br{f_{\bm{\phi}}\br{\vect{\vect{x}}}}$ be finite. That is, for every parameter setting $\bm{\phi}$:
\begin{equation}
\integralx{\exp\br{f_{\bm{\phi}}\br{\vect{\vect{x}}}}}{\vect{x}}{\R^D} = Z_{\bm{\phi}} < \infty.
\end{equation}
If that's the case, we can still define a valid density function as follows:
\begin{equation}
q_{\bm{\phi}}\br{\vect{x}} = \frac{1}{Z_{\bm{\phi}}}\exp\br{f_{\bm{\phi}}\br{\vect{\vect{x}}}}.
\end{equation}
Models defined this way are called \emph{energy-based models}. The quantity $-f_{\bm{\phi}}\br{\vect{\vect{x}}}$ is known as the \emph{energy}, and $Z_{\bm{\phi}}$ is known as the \emph{normalizing constant}. Under the above definitions, a neural density estimator is an energy-based model whose normalizing constant is always $1$.

The main advantage of energy-based modelling is the increased flexibility in specifying $f_{\bm{\phi}}\br{\vect{x}}$. However, evaluating the density under an energy-based model is intractable, since calculating the normalizing constant involves a high-dimensional integral. There are ways to estimate the normalizing constant, e.g.~via \emph{annealed importance sampling} \citep{Salakhutdinov:2008:ais} or \emph{model distillation} \citep{Papamakarios:2015:distilling}, but they are approximate and often expensive.

In addition to their density being intractable, energy-based models typically don't provide a mechanism of generating samples. Instead, one must resort to approximate methods such as Markov-chain Monte Carlo \citep{Murray:2007:mcmc, Neal:1993:mcmc} to generate samples. In contrast, some neural density estimators such \emph{Masked Autoregressive Flow} or \emph{Real NVP} are capable of generating exact independent samples directly.

Finally, the intractability of $q_{\bm{\phi}}\br{\vect{x}}$ makes likelihood-based training of energy-based models less straightforward than  of neural density estimators. In particular, the gradient of $\log{q_{\bm{\phi}}\br{\vect{x}}}$ with respect to $\bm{\phi}$ is:
\begin{equation}
\deriv{}{\bm{\phi}}\log{q_{\bm{\phi}}\br{\vect{x}}} = \deriv{}{\bm{\phi}}f_{\bm{\phi}}\br{\vect{x}} - \deriv{}{\bm{\phi}}\log Z_{\bm{\phi}}.
\end{equation}
By substituting in the definition of the normalizing constant and by exchanging the order of differentiation and integration, we can write the gradient of $\log Z_{\bm{\phi}}$ as follows:
\begin{align}
\deriv{}{\bm{\phi}}\log Z_{\bm{\phi}} 
&= \frac{1}{Z_{\bm{\phi}}}\deriv{}{\bm{\phi}}\integralx{\exp\br{f_{\bm{\phi}}\br{\vect{\vect{x}}}}}{\vect{x}}{\R^D}\\
&= \integralx{\frac{1}{Z_{\bm{\phi}}}\exp\br{f_{\bm{\phi}}\br{\vect{\vect{x}}}}\,\deriv{}{\bm{\phi}}f_{\bm{\phi}}\br{\vect{x}}}{\vect{x}}{\R^D}\\
&= \integralx{q_{\bm{\phi}}\br{\vect{x}}\,\deriv{}{\bm{\phi}}f_{\bm{\phi}}\br{\vect{x}}}{\vect{x}}{\R^D}\\
&= \avgx{\deriv{}{\bm{\phi}}f_{\bm{\phi}}\br{\vect{x}}}{q_{\bm{\phi}}\br{\vect{x}}}.
\end{align}
The above expectation is intractable, hence direct gradient-based optimization of the average log likelihood is impractical. There are ways to either approximate the above expectation or avoid computing it entirely, such as \emph{contrastive divergence} \citep{Hinton:2002:products_experts}, \emph{score matching} \citep{Hyvarinen:05:score_matching}, or \emph{noise-contrastive estimation} \citep{Gutmann:2012:nce}. In contrast, for neural density estimators the above expectation is always zero by construction, which makes likelihood-based training by backpropagation readily applicable.

\subsection{Latent-variable models and variational autoencoders}
\label{sec:maf:gen:vaes}

One way of increasing the expressivity of a simple generative model is by introducing \emph{latent variables}. Latent variables are auxiliary variables $\vect{z}$ that are used to augment the data $\vect{x}$. For the purposes of this discussion we will assume that $\vect{z}$ is continuous (i.e.~$\vect{z}\in\R^K$), but discrete latent variables are also possible.

Having decided on how many latent variables to introduce, one would typically define a joint density model $q_{\bm{\phi}}\br{\vect{x}, \vect{z}}$. This is often done by separately defining a prior model $q_{\bm{\phi}}\br{\vect{z}}$ and a conditional model $q_{\bm{\phi}}\br{\vect{x}\g \vect{z}}$. Then, the density of $\vect{x}$ is obtained by:
\begin{equation}
q_{\bm{\phi}}\br{\vect{x}} = \integralx{q_{\bm{\phi}}\br{\vect{x}, \vect{z}}}{\vect{z}}{\R^K}
 = \integralx{q_{\bm{\phi}}\br{\vect{x}\g \vect{z}}q_{\bm{\phi}}\br{\vect{z}}}{\vect{z}}{\R^K}.
\end{equation}

The motivation for introducing latent variables and then integrating them out is that the marginal $q_{\bm{\phi}}\br{\vect{x}}$ can be complex even if the prior $q_{\bm{\phi}}\br{\vect{z}}$ and the conditional $q_{\bm{\phi}}\br{\vect{x}\g \vect{z}}$ are not. Hence, one can obtain a complex model of the data despite using simple modelling components. Another way to view $q_{\bm{\phi}}\br{\vect{x}}$ is as a mixture of infinitely many components $q_{\bm{\phi}}\br{\vect{x}\g \vect{z}}$ indexed by $\vect{z}$, weighted by $q_{\bm{\phi}}\br{\vect{z}}$. In fact, mixture models of finitely many components such as those discussed in section \ref{sec:maf:methods:simple_parametric} can be thought of as latent-variable models with discrete latent variables.

Like energy-based models, the density of a latent-variable model is intractable to calculate since it involves a high-dimensional integral. If that integral is tractable, the model is essentially a neural density estimator and can be treated as such. So, for the purposes of this discussion, we'll assume that the integral is always intractable by definition.

Even though the model density is intractable, it's possible to lower-bound $\log{q_{\bm{\phi}}\br{\vect{x}}}$ by introducing an auxiliary density model $r_{\bm{\psi}}\br{\vect{z}\g\vect{x}}$ with parameters $\bm{\vect{\psi}}$ as follows:
\begin{align}
\log{q_{\bm{\phi}}\br{\vect{x}}} &= \log\integralx{q_{\bm{\phi}}\br{\vect{x}\g \vect{z}}q_{\bm{\phi}}\br{\vect{z}}}{\vect{z}}{\R^K}\\
&= \log\,{\avgx{\frac{q_{\bm{\phi}}\br{\vect{x}\g \vect{z}}q_{\bm{\phi}}\br{\vect{z}}}{r_{\bm{\psi}}\br{\vect{z}\g\vect{x}}}}{r_{\bm{\psi}}\br{\vect{z}\g\vect{x}}}}\\
&\ge \avgx{\log q_{\bm{\phi}}\br{\vect{x}\g \vect{z}} + \log q_{\bm{\phi}}\br{ \vect{z}} - \log r_{\bm{\psi}}\br{\vect{z}\g\vect{x}}}{r_{\bm{\psi}}\br{\vect{z}\g\vect{x}}},
\end{align}
where we used \emph{Jensen's inequality} to exchange the logarithm with the expectation. The above lower bound is often called \emph{evidence lower bound} or \emph{ELBO}\@. The ELBO is indeed a lower bound given any choice of conditional density $r_{\bm{\psi}}\br{\vect{z}\g\vect{x}}$, and it becomes equal to $\log{q_{\bm{\phi}}\br{\vect{x}}}$ if:
\begin{equation}
r_{\bm{\psi}}\br{\vect{z}\g\vect{x}} =
q_{\bm{\phi}}\br{\vect{z}\g\vect{x}}
= \frac{q_{\bm{\phi}}\br{\vect{x}\g\vect{z}}q_{\bm{\phi}}\br{\vect{z}}}{q_{\bm{\phi}}\br{\vect{x}}}.
\end{equation}

In practice, the ELBO can be used as a tractable objective to train the model. Maximizing the average ELBO over training data with respect to $\bm{\phi}$ is equivalent to maximizing a lower bound to the average log likelihood. Moreover, maximizing the average ELBO with respect to $\bm{\psi}$ makes the ELBO a tighter lower bound. Stochastic gradients of the average ELBO with respect to $\bm{\phi}$ and $\bm{\psi}$ can be obtained by reparameterizing $r_{\bm{\psi}}\br{\vect{z}\g\vect{x}}$ with respect to its internal random variables (similarly to how we reparameterized MADE in section \ref{sec:maf:paper}), and by estimating the expectation over the reparameterized $r_{\bm{\psi}}\br{\vect{z}\g\vect{x}}$ via Monte Carlo. A latent-variable model trained this way is known as a \emph{variational autoencoder} \citep{Kingma:2014:vae, Rezende:2014:vae}. In the context of variational autoencoders, $r_{\bm{\psi}}\br{\vect{z}\g\vect{x}}$ is often referred to as the \emph{encoder} and $q_{\bm{\phi}}\br{\vect{x}\g \vect{z}}$ as the \emph{decoder}.

Apart from increasing the flexibility of the model, another motivation for introducing latent variables is \emph{representation learning}. Consider for example a variational autoencoder that is used to model images of handwritten characters. Assume that the decoder $q_{\bm{\phi}}\br{\vect{x}\g \vect{z}}$ is deliberately chosen to factorize over the elements of $\vect{x}$, i.e.~the pixels in an image. Clearly, a factorized distribution over pixels is unable to model global structure, which in our case corresponds to factors of variation such as the identity of a character, its shape, the style of handwriting, and so on. Therefore, information about global factors of variation such as the above must be encoded into the latent variables $\vect{z}$ if the model is to represent the data well. If the dimensionality of $\vect{z}$ is deliberately chosen to be smaller than that of $\vect{x}$ (the number of pixels), the latent variables correspond to a compressed encoding of the global factors of variation in an image. Given an image $\vect{x}$, plausible such encodings can be directly sampled from $r_{\bm{\psi}}\br{\vect{z}\g\vect{x}}$.

Finally, we can further encourage the latent variables to acquire semantic meaning by engineering the structure of the model appropriately. Examples include: hierarchically-structured latent variables that learn to represent datasets \citep{Edwards:2017:statistician}, exchangeable groups of latent variables that learn to represent objects in a scene \citep{Nash:2017:multi_entity_vae, Eslami:2016:air}, and latent variables with a Markovian structure that learn to represent environment states \citep{Gregor:2019:tdvae, Buesing:2018:fast_gen_models}.

\subsection{Neural samplers and generative adversarial networks}

A \emph{neural sampler} is a neural network $f_{\bm{\phi}}$ with parameters $\bm{\phi}$ that takes random unstructured noise $\vect{u}\in \R^{K}$ and transforms it into data $\vect{x}\in \R^D$:
\begin{equation}
\vect{x} = f_{\bm{\phi}}\br{\vect{u}}
\quad\text{where}\quad
\vect{u}\sim \pi_u\br{\vect{u}}.
\end{equation}
In the above, $\pi_u\br{\vect{u}}$ is a fixed source of noise (not necessarily a density model). Under this definition, normalizing flows with tractable sampling, such as Masked Autoregressive Flow or Real NVP, are also neural samplers. A typical variational autoencoder is also a neural sampler, in which case $\vect{u}$ are the internal random numbers needed to sample $\vect{z}$ from the prior and $\vect{x}$ from the decoder.

Typically, the main purpose of a neural sampler is to be used as a data generator and not as a density estimator. Hence, a neural sampler doesn't need to have the restrictions of a normalizing flow: $f_{\bm{\phi}}$ doesn't have to be bijective and $K$ (the dimensionality of the noise) doesn't have to match $D$ (the dimensionality of the data). As a result, a neural sampler offers increased flexibility in specifying the neural network $f_{\bm{\phi}}$ compared to a normalizing flow.

A consequence of the freedom of specifying $f_{\bm{\phi}}$ and $K$ is that a neural sampler may no longer correspond to a valid density model of the data. For example, if $f_{\bm{\phi}}$ is not injective, a whole subset of $\R^K$ of non-zero measure under $\pi_u\br{\vect{u}}$ can be mapped to a single point in $\R^D$, where the model density will be infinite. Moreover, if $K<D$, the neural sampler can only generate data on a $K$-dimensional manifold embedded in $\R^D$. Hence, the model density will be zero almost everywhere, except for a manifold of zero volume where the model density will be infinite.

Training a neural sampler is typically done by generating a set of samples from the model, selecting a random subset of the training data, and calculating a measure of discrepancy between the two sets. If that discrepancy measure is differentiable, its gradient with respect to the model parameters $\bm{\phi}$ can be obtained by backpropagation. Then, the model can be trained by stochastic-gradient descent, with stochastic gradients obtained by repeating the above procedure.

One way to obtain a differentiable discrepancy measure is via an additional neural network, known as the \emph{discriminator}, whose task is to discriminate generated samples from training data. The classification performance of the discriminator (as measured by e.g.~cross entropy) can be used as a differentiable discrepancy measure. In practice, the discriminator is trained at the same time as the neural sampler. Neural samplers trained this way are known as \emph{generative adversarial networks} \citep{Goodfellow:2014:gan}. Other possible discrepancy measures between training data and generated samples include the maximum mean discrepancy \citep{Dziugaite:2015:mmdgan}, or the Wasserstein distance \citep{Arjovsky:2017:wgan}.

Generative adversarial networks often excel at data-generation performance. When trained as image models, they are able to generate high-resolution images that are visually indistinguishable from real images \citep{Brock:2018:biggan, Karras:2018:progan, Karras:2018:stylegan}. However, their good performance as data generators doesn't necessarily imply good performance in terms of average log likelihood. \citet{Grover:2018:flowgan} and \citet{Danihelka:2017:rnvp_gan} have observed that a Real NVP trained as a generative adversarial network can generate realistic-looking samples, but it may perform poorly as a density estimator.

\section{Summary and conclusions}

The probability density of a generative process at a location $\vect{x}$ is how often the process generates data near $\vect{x}$ per unit volume. Density estimation is the task of modelling the density of a process given training data generated by the process. In this chapter, I discussed how to perform density estimation with neural networks, and why this may be a good idea.

Instead of densities, we could directly estimate other statistical quantities associated with the process, such as expectations of various functions under the process, or we could model the way data is generated by the process. Why estimate densities then? In section \ref{sec:maf:intro:why_densities}, I identified the following reasons as to why estimating densities can be useful:
\begin{enumerate}[label=(\roman*)]
\item Bayesian inference is a calculus over densities.
\item Accurate densities imply good data compression.
\item The density of a model is a principled training and evaluation objective.
\item Density models are useful components in other algorithms.
\end{enumerate}
That said, I would argue that in practice we should first identify what task we are interested in, and then use the tool that is more appropriate for the task. If the task calls for a density model, then a density model should be used. Otherwise, a different solution may be more appropriate.

Estimating densities is hard, and it becomes dramatically harder as the dimensionality of the data increases. As I discussed in section \ref{sec:maf:intro:high_dim}, the reason is the curse of dimensionality: a dataset becomes sparser in higher dimensions, and naive density estimation doesn't scale. In section \ref{sec:maf:methods}, I argued that standard density estimation methods based on histograms or smoothing kernels are particularly susceptible to the curse of dimensionality.

Is high-dimensional density estimation hopeless? In section \ref{sec:maf:intro:high_dim} I argued that it's not, as long as we are prepared to make assumptions about the densities we want to estimate, and encode these assumptions into our models. I identified the following assumptions that are often appropriate to make for real-world densities:
\begin{enumerate}[label=(\roman*)]
\item They vary smoothly.
\item Their intrinsic dimensionality is smaller than that of the data.
\item They have symmetries.
\item They involve variables that are loosely dependent.
\end{enumerate}
Neural density estimators scale to high dimensions partly because their architecture can encode the above assumptions. In contrast, it's not clear how these assumptions can be encoded into models such as parametric mixtures, histograms, or smoothing kernels.

The main contribution of this chapter is \emph{Masked Autoregressive Flow}, a normalizing flow that consists of a stack of reparameterized autoregressive models. In section \ref{sec:maf:contrib}, I discussed the impact Masked Autoregressive Flow has had since its publication, how it contributed to further research, and what its main limitations have been.

Since the publication of Masked Autoregressive Flow in December 2017, the field of normalizing flows has experienced considerable research interest. In section \ref{sec:maf:flows}, I reviewed advances in the field since the publication of the paper, and I examined their individual strengths and weaknesses. Even though normalizing flows have improved significantly, no model yet exists that has all the following characteristics:
\begin{enumerate}[label=(\roman*)]
\item The density under the model can be exactly evaluated in one neural-network pass (like Masked Autoregressive Flow)\@.
\item Exact independent samples can be generated in one neural-network pass (like Inverse Autoregressive Flow)\@.
\item The modelling performance is on par with the best autoregressive models.
\end{enumerate}
Of all the models that satisfy both (i) and (ii), the best-performing ones are Real NVP and its successor Glow. However, as we saw in section \ref{sec:maf:paper}, the coupling layers used by Real NVP and Glow are not as expressive as the autoregressive layers used by MAF and IAF\@. On the other hand, satisfying (iii) often requires sacrificing either (i) or (ii). Nonetheless, there is a lot of research activity currently in normalizing flows, and several new ideas such as Neural ODEs and FFJORD are beginning to emerge, so it's reasonable to expect further advances in the near future.

Neural density estimation is only one approach to generative modelling, with variational autoencoders and generative adversarial networks being the main alternatives when exact density evaluations are not required. As I discussed in section \ref{sec:maf:gen}, the advantage of variational autoencoders is their ability to learn structured representations, whereas the strength of generative adversarial networks is their performance in generating realistic data. As I argued earlier, choosing among a neural density estimator, a variational autoencoder and a generative adversarial network should primarily depend on the task we are interested in.

That said, normalizing flows, variational autoencoders and generative adversarial networks are all essentially the same model: a neural network that takes unstructured noise and turns it into data. The difference is in the requirements the different models must satisfy (a normalizing flow must be invertible), in the capabilities they offer (a normalizing flow provides explicit densities in addition to sampling), and in the way they are trained. The fact that generative adversarial networks can take white noise and turn it into convincingly realistic images should be seen as proof that (a) this task can be achieved by a smooth function, and (b) such a function is learnable from data. Since this task is demonstrably possible, it should serve as a performance target for neural density estimators too.

\part{Likelihood-free inference}
\label{part:lfi}
\chapter{Fast $\epsilon$-free Inference of Simulator Models with Bayesian Conditional Density Estimation}
\label{chapter:efree}

This chapter is about likelihood-free inference, i.e.~Bayesian inference when the likelihood is intractable, and how neural density estimation can be used to address this challenge. We start by explaining what likelihood-free inference is, in what situations the likelihood may be intractable, and why it is important to address this challenge (section \ref{sec:efree:intro}). We then review \emph{approximate Bayesian computation}, a family of methods for likelihood-free inference based on simulations (section \ref{sec:efree:abc}). The main contribution of this chapter is the paper \emph{Fast $\epsilon$-free Inference of Simulator Models with Bayesian Conditional Density Estimation}, which introduces a new method for likelihood-free inference based on neural density estimation (section \ref{sec:efree:paper}). We conclude the chapter with a discussion of the contributions, impact and limitations of the paper, and a comparison of the proposed method with previous and following work (section \ref{sec:efree:discussion}).

\section{Likelihood-free inference: what and why?}
\label{sec:efree:intro}

Suppose we have a stationary process that generates data and is controlled by parameters $\bm{\theta}$. Every time the process is run forward, it stochastically generates a datapoint $\vect{x}$ whose distribution depends on $\bm{\theta}$. We will assume that $\vect{x}$ and $\bm{\theta}$ are continuous vectors, but many of the techniques discussed in this thesis can be adapted for discrete $\vect{x}$ and/or $\bm{\theta}$. We will further assume that for every setting of $\bm{\theta}$, the corresponding distribution over $\vect{x}$ admits a finite density everywhere; in other words, the process defines a conditional density function $\prob{\vect{x}\g\bm{\theta}}$.

Given a datapoint $\vect{x}_o$ known to have been generated by the process, we are interested in inferring plausible parameter settings that could have generated $\vect{x}_o$. In particular, given prior beliefs over parameters described by a density $\prob{\bm{\theta}}$, we are interested in computing the posterior density $\prob{\bm{\theta}\g\vect{x}=\vect{x}_o}$ obtained by \emph{Bayes' rule}:
\begin{equation}
\prob{\bm{\theta}\g\vect{x}=\vect{x}_o} = \frac{\prob{\vect{x}_o \g \bm{\theta}}\,\prob{\bm{\theta}}}{Z\br{\vect{x}_o}}
\quad\text{where}\quad
Z\br{\vect{x}_o}=\integral{\prob{\vect{x}_o \g \bm{\theta}'}\,\prob{\bm{\theta}'}}{\bm{\theta}'}.
\label{eq:efree:bayes_rule}
\end{equation}
We assume that the normalizing constant $Z\br{\vect{x}_o}$ is finite for every $\vect{x}_o$, so that the posterior density is always well-defined.

\subsection{Density models and simulator models}
\label{sec:efree:intro:sim_models}

The choice of inference algorithm depends primarily on how the generative process is modelled. In this section, I discuss and compare two types of models: density models and simulator models.

A \emph{density model}, also known as an \emph{explicit model}, describes the conditional density function of the process. Given values for $\vect{x}$ and $\bm{\theta}$, a density model returns the value of the conditional density $\prob{\vect{x}\g\bm{\theta}}$ (or an approximation of it if the model is not exact). With a density model, we can easily evaluate the posterior density $\prob{\bm{\theta}\g\vect{x}=\vect{x}_o}$ up to a normalizing constant using Bayes' rule (equation \ref{eq:efree:bayes_rule}). Even though the normalizing constant $Z\br{\vect{x}_o}$ is typically intractable, we can sample from the posterior using e.g.~Markov-chain Monte Carlo \citep{Murray:2007:mcmc, Neal:1993:mcmc}, or approximate the posterior with a more convenient distribution using e.g.~variational inference \citep{Ranganath:2014:bbvi, Kucukelbir:2015:autovistan, Blei:2017:vi} or knowledge distillation \citep{Papamakarios:2015:distilling, Papamakarios:2015:distilling_model_knowledge}. We refer to such methods as \emph{likelihood-based inference methods}, as they explicitly evaluate the likelihood $\prob{\vect{x}_o \g \bm{\theta}}$.

On the other hand, a \emph{simulator model}, also known as an \emph{implicit model}, describes how the process generates data. For any parameter setting $\bm{\theta}$, a simulator model can be run forward to generate exact independent samples from $\prob{\vect{x}\g\bm{\theta}}$ (or approximate independent samples if the model is not perfect). Unlike a density model, we can't typically evaluate the density under a simulator model. In order to perform inference in a simulator model we need methods that make use of simulations from the model but don't require density evaluations. We refer to such methods as \emph{likelihood-free inference methods}.

In general, likelihood-free methods are less efficient than likelihood-based methods. As we will see in sections \ref{sec:efree:abc} and \ref{sec:efree:paper}, likelihood-free methods can require a large number of simulations from the model to produce accurate results, even for fairly simple models. One of the main topics of this thesis, and of research in likelihood-free inference in general, is how to reduce the required number of simulations without sacrificing inference quality.

Since likelihood-free methods are less efficient than likelihood-based methods, why should we bother with simulator models at all, and not just require all models to be density models instead? The answer is that simulators can often be more natural and more interpretable modelling tools than density functions. A simulator model is a direct specification of how a process of interest generates data, which can be closer to our understanding of the process as a working mechanism. In contrast, a density function is a mathematical abstraction that is often hard to interpret, and as a result it can be an unintuitive description of real-world phenomena.

Simulator models are particularly well-suited for scientific applications. In high-energy physics for instance, simulators are used to model collisions in particle accelerators and the interaction of produced particles with particle detectors \citep{Agostinelli:2003:geant, Sjostrand:2008:pythia}. Inference in such models can be used to infer parameters of physical theories and potentially discover new physics \citep{Brehmer:2018:eft, Brehmer:2018:eft_guide}. Other scientific applications of simulator models and likelihood-free inference methods include
astronomy and cosmology \citep{Schafer:2012:lfi_cosmology, Alsing:2018:lfi_cosmology, Alsing:2019:lfi_cosmology},
systems biology \citep{Wilkinson:2011:systems_bio},
evolution and ecology \citep{Pritchard:1999:pop_growth, Ratmann:2007:protein_nets, Wood:2010:sl, Beaumont:2010:abc_evo_eco},
population generics \citep{Beaumont:2002:abc_pop_gen},
and computational neuroscience \citep{Pospischil:2008:hh_models, Markram:2015:neurosim, Lueckmann:2017:snpe, Sterratt:2011:neuro}

In addition to their scientific applications, simulator models can be useful in engineering and artificial intelligence. For instance, one can build an image model using a computer-graphics engine: the engine takes a high-level description of a scene (which corresponds to the model parameters $\bm{\theta}$) and renders an image of that scene (which corresponds to the generated data~$\vect{x}$). Computer vision and scene understanding can be cast as inference of scene parameters given an observed image \citep{Mansinghka:2013:graphics, Kulkarni:2015:picture, Romaszko:2017:vision_as_graphics}. Other examples of applications of simulator models in artificial intelligence include perceiving physical properties of objects using a physics engine \citep{Wu:2015:galileo}, and inferring the goals of autonomous agents using probabilistic programs \citep{Cusumano:2017:agent_goals}.

\subsection{When is the likelihood intractable?}

In the previous section, I argued that simulator models are useful modelling tools. Nonetheless, this doesn't necessitate the use of likelihood-free inference methods: given a simulator model, we could try to calculate its likelihood based on the simulator's internal workings, and then use likelihood-based inference methods instead. In this section, I will discuss certain cases in which it may be too expensive or even impossible to calculate the likelihood of a simulator, and thus likelihood-free methods may be preferable or necessary.

\subsubsection{Integrating over the simulator's latent state}

Typically, the stochasticity of a simulator model comes from internal generation of random numbers, which form the simulator's latent state. Consider for example a simulator whose latent state consists of variables $\vect{z}_1, \ldots, \vect{z}_K$, generated from distributions $p_{\bm{\psi}_1}, \ldots, p_{\bm{\psi}_K}\!$ whose parameters $\bm{\psi}_1, \ldots, \bm{\psi}_K$ can be arbitrary functions of all variables preceding them. For simplicity, assume that all variables are continuous, and that $K$ is fixed (and not random). Such a simulator can be expressed as the following sequence of statements:
\begin{align}
\bm{\psi}_1 &= f_1\br{\bm{\theta}}\\
\vect{z}_1 &\sim p_{\bm{\psi}_1}\\
&\ldots\notag\\
\bm{\psi}_K &= f_K\br{\vect{z}_{1:K-1}, \bm{\theta}}\\
\vect{z}_K &\sim p_{\bm{\psi}_K}\\
\bm{\phi} &= g\br{\vect{z}_{1:K}, \bm{\theta}}\\
\vect{x} &\sim p_{\bm{\phi}}.
\end{align}
If the distributions $p_{\bm{\psi}_1}, \ldots, p_{\bm{\psi}_K}\!$ and $p_{\bm{\phi}}$ admit tractable densities, the joint density of data $\vect{x}$ and latent state $\vect{z}_{1:K}$ is also tractable, and can be calculated by:
\begin{equation}
\prob{\vect{x}, \vect{z}_{1:K}\g\bm{\theta}} = p_{\bm{\phi}}\br{\vect{x}}\prod_k p_{\bm{\psi}_k}\br{\vect{z}_k}.
\end{equation}
Given observed data $\vect{x}_o$, the likelihood of the simulator is:
\begin{equation}
\prob{\vect{x}_o\g\bm{\theta}} = \integral{\prob{\vect{x}_o, \vect{z}_{1:K}\g\bm{\theta}}}{\vect{z}_{1:K}}.
\end{equation}
Since the likelihood of the simulator involves integration over the latent state, it is generally intractable.

Even when integration over the latent state makes the likelihood intractable, the joint likelihood $\prob{\vect{x}_o, \vect{z}_{1:K}\g\bm{\theta}}$ may still be tractable (as in the above example). We could then sample from the joint posterior $\prob{\vect{z}_{1:K},\bm{\theta}\g\vect{x} = \vect{x}_o}$ using e.g.~MCMC, and simply discard the latent-state samples. However, this may be infeasible when the latent state is large. Consider for example a simulator of a physical process involving Brownian motion of a large set of particles. Inferring the latent state of the simulator would entail sampling from the space of all possible particle trajectories that are consistent with the data. If both the observed data $\vect{x}_o$ and the parameters of interest $\bm{\theta}$ correspond to macroscopic variables associated with the process, likelihood-free inference methods that don't involve inferring the microscopic state of the process may be preferable.

\subsubsection{Deterministic transformations}

Sometimes, the likelihood of a simulator model is intractable because the output of the simulator is a deterministic transformation of the simulator's latent state. Consider for example the following simulator model:
\begin{align}
\vect{z} &\sim \prob{\vect{z}\g\bm{\theta}}\\
\vect{x} &= f\br{\vect{z}, \bm{\theta}},
\end{align}
where $\vect{z}$ is continuous and the density $\prob{\vect{z}\g\bm{\theta}}$ is tractable. The joint distribution of $\vect{x}$ and $\vect{z}$ does not admit a proper density, but can be expressed using a delta distribution as follows:
\begin{equation}
\prob{\vect{x}, \vect{z}\g\bm{\theta}} = \dirac{\vect{x} - f\br{\vect{z}, \bm{\theta}}}\,\prob{\vect{z}\g\bm{\theta}}.
\end{equation}
The likelihood of the simulator is:
\begin{equation}
\prob{\vect{x}_o\g\bm{\theta}} = \integral{\dirac{\vect{x}_o - f\br{\vect{z}, \bm{\theta}}}\,\prob{\vect{z}\g\bm{\theta}}}{\vect{z}}.
\end{equation}
Calculating the above integral requires a change of variables from $\vect{z}$ to $\vect{x}$ which is not tractable in general. Moreover, the distribution $\prob{\vect{x}\g\bm{\theta}}$ may not admit a proper density even if $\prob{\vect{z}\g\bm{\theta}}$ does.

Sampling from the joint posterior of $\vect{z}$ and $\bm{\theta}$ in order to avoid calculating the likelihood also poses difficulties. Given observed data $\vect{x}_o$, the joint posterior $\prob{\vect{z},\bm{\theta}\g\vect{x} = \vect{x}_o}$ is constrained on a manifold in the  $\pair{\vect{z}}{\bm{\theta}}$-space defined by $\vect{x}_o = f\br{\vect{z}, \bm{\theta}}$. 
One approach to sampling from this posterior is to replace the joint distribution $\prob{\vect{x}, \vect{z}\g\bm{\theta}}$ with the following approximation:
\begin{equation}
p_{\epsilon}\br{\vect{x}, \vect{z}\g\bm{\theta}} = \gaussianx{\vect{x}}{f\br{\vect{z}, \bm{\theta}}}{\epsilon^2\mat{I}}\,\prob{\vect{z}\g\bm{\theta}},
\end{equation}
where $\epsilon$ is a small positive constant. This is equivalent to replacing the original simulator model with the following one:
\begin{align}
\vect{z} &\sim \prob{\vect{z}\g\bm{\theta}}\\
\vect{z}' &\sim \gaussian{\vect{0}}{\mat{I}}\\
\vect{x} &= f\br{\vect{z}, \bm{\theta}} + \epsilon\vect{z}'.
\end{align}
An issue with this approach is that for small $\epsilon$ methods such as MCMC may struggle, whereas for large $\epsilon$ the posterior no longer corresponds to the original model. MCMC methods on constrained manifolds exist and may be used instead, but make assumptions about the manifold which may not hold for a simulator model in general \citep{Brubaker:2012:constrained_mcmc, Graham:2017:differentiable_simulators}.

\subsubsection{Inaccessible internal workings}

Even if it is theoretically possible to use a likelihood-based method with a simulator model, this may be difficult in practice if the simulator is a black box whose internal workings we don't have access to. For example, the simulator may be provided as an executable program, or it may be written in a low-level language (e.g.~for running efficiency), or it may be an external library routine. I would argue that a general-purpose inference engine ought to have inference algorithms that support such cases.

Finally, the ``simulator'' may not be a model written in computer code, but an actual process in the real world. For example, the ``simulator`` could be an experiment in a lab, asking participants to perform a task, or a measurement of a physical phenomenon. If a model doesn't exist, likelihood-free inference methods could be used directly in the real world.

\section{Approximate Bayesian computation}
\label{sec:efree:abc}

In this section, I will discuss a family of likelihood-free inference methods broadly known as \emph{approximate Bayesian computation} or \emph{ABC}\@. In general, ABC methods draw approximate samples from the parameter posterior by repeatedly simulating the model and checking whether simulated data match the observed data. I will discuss three types of ABC methods: \emph{rejection ABC}, \emph{Markov-chain Monte Carlo ABC}, and \emph{sequential Monte Carlo ABC}.

\subsection{Rejection ABC}
\label{sec:efree:abc:rej_abc}

Let $B_{\epsilon}\br{\vect{x}_o}$ be a neighbourhood of $\vect{x}_o$, defined as the set of datapoints whose distance from $\vect{x}_o$ is no more than $\epsilon$:
\begin{equation}
B_{\epsilon}\br{\vect{x}_o} = \set{\vect{x}: \norm{\vect{x}-\vect{x}_o}\le \epsilon}.
\end{equation}
In the above, $\norm{\cdot}$ can be the Euclidean norm, or any other norm in $\R^D$. For a small $\epsilon$, we can approximate the likelihood  $\prob{\vect{x}_o\g\bm{\theta}}$ by the average conditional density inside $B_{\epsilon}\br{\vect{x}_o}$:
\begin{equation}
\prob{\vect{x}_o\g\bm{\theta}}\approx \frac{\Prob{\norm{\vect{x}-\vect{x}_o}\le \epsilon\g\bm{\theta}}}{\abs{B_{\epsilon}\br{\vect{x}_o}}},
\end{equation}
where $\abs{B_{\epsilon}\br{\vect{x}_o}}$ is the volume of $B_{\epsilon}\br{\vect{x}_o}$. Substituting the above likelihood approximation into Bayes' rule (equation \ref{eq:efree:bayes_rule}), we obtain the following approximation to the posterior:
\begin{equation}
\prob{\bm{\theta}\g\vect{x}=\vect{x}_o} \approx 
\frac{\Prob{\norm{\vect{x}-\vect{x}_o}\le \epsilon\g\bm{\theta}}\,\prob{\bm{\theta}}}{\integral{\Prob{\norm{\vect{x}-\vect{x}_o}\le \epsilon\g\bm{\theta}'}\,\prob{\bm{\theta}'}}{\bm{\theta}'}}
= \prob{\bm{\theta}\g\norm{\vect{x}-\vect{x}_o} \le \epsilon}.
\end{equation}
The approximate posterior $\prob{\bm{\theta}\g\norm{\vect{x}-\vect{x}_o} \le \epsilon}$ can be thought of as the exact posterior under an alternative observation, namely that generated data $\vect{x}$ falls inside the neighbourhood $B_{\epsilon}\br{\vect{x}_o}$.
As $\epsilon$ approaches zero, $B_{\epsilon}\br{\vect{x}_o}$ becomes infinitesimally small, and the approximate posterior approaches the exact posterior $\prob{\bm{\theta}\g\vect{x}=\vect{x}_o}$ \citep[for a formal proof under suitable regularity conditions, see e.g.][supplementary material, theorem 1]{Prangle:2017:adapting_distance}. Conversely, as $\epsilon$ approaches infinity, $B_{\epsilon}\br{\vect{x}_o}$ covers the whole space, and the approximate posterior approaches the prior $\prob{\bm{\theta}}$.

\emph{Rejection ABC} \citep{Pritchard:1999:pop_growth} is a rejection-sampling method for obtaining independent samples from the approximate posterior $\prob{\bm{\theta}\g\norm{\vect{x}-\vect{x}_o} \le \epsilon}$. It works by first sampling parameters from the prior $\prob{\bm{\theta}}$, then simulating the model using the sampled parameters, and only accepting the sample if the simulated data is no further than a distance $\epsilon$ away from  $\vect{x}_o$. The parameter $\epsilon$ controls the tradeoff between approximation quality and computational efficiency: as $\epsilon$ becomes smaller, the accepted samples follow more closely the exact posterior, but the algorithm accepts less often.
Rejection ABC is shown in algorithm \algref{alg:efree:rej_abc}.

\begin{figure}[h]
\alglabel{alg:efree:rej_abc}
\begin{algorithm}[H]
\DontPrintSemicolon
\setstretch{1.2}
\Repeat{$\norm{\vect{x} - \vect{x}_o}\le \epsilon$}{
Sample $\bm{\theta}\sim\prob{\bm{\theta}}$\;
Simulate $\vect{x}\sim\prob{\vect{x}\g\bm{\theta}}$\;
}
\Return $\bm{\theta}$\;
\caption{Rejection ABC\label{alg:efree:rej_abc}}
\end{algorithm}
\end{figure}

An issue with rejection ABC is that it can become prohibitively expensive for small $\epsilon$, especially when the data is high-dimensional. As an illustration, consider the case where $\norm{\cdot}$ is the maximum norm, and hence the acceptance region $B_{\epsilon}\br{\vect{x}_o}$ is a cube of side $2\epsilon$ centred at $\vect{x}_o$. For small $\epsilon$, the acceptance probability can be approximated as follows:
\begin{equation}
\Prob{\norm{\vect{x}-\vect{x}_o}\le \epsilon\g\bm{\theta}} \approx \prob{\vect{x}_o\g\bm{\theta}}\abs{B_{\epsilon}\br{\vect{x}_o}}
= \prob{\vect{x}_o\g\bm{\theta}}\,\br{2\epsilon}^D,
\end{equation}
where $D$ is the dimensionality of $\vect{x}$. As $\epsilon\rightarrow 0$, the acceptance probability approaches zero at a rate of $\bigo{\epsilon^D}$. Moreover, as $D$ grows large, the acceptance probability approaches zero for any $\epsilon<\nicefrac{1}{2}$. To put that into perspective, if $\prob{\vect{x}_o\g\bm{\theta}}=1$, $\epsilon=0.05$ and $D=15$, we would need on average about a thousand trillion simulations for each accepted sample!

In practice, in order to maintain the acceptance probability at manageable levels, we typically transform the data into a small number of \emph{features}, also known as \emph{summary statistics}, in order to reduce the dimensionality $D$. If the summary statistics are sufficient for inferring $\bm{\theta}$, then turning data into summary statistics incurs no loss of information about $\bm{\theta}$ \citep[for a formal definition of sufficiency, see e.g.][definition 9.32]{Wasserman:2010:stats}. However, it can be hard to find sufficient summary statistics, and so in practice summary statistics are often chosen based on domain knowledge about the inference task. A large section of the ABC literature is devoted to constructing good summary statistics \citep[see e.g.][and references therein]{Fearnhead:2012:semi_automatic_abc, Blum:2013:dim_reduction_review, Prangle:2014:semi_automatic_abc_model_choice, Charnock:2018:IMNN}.

So far I've described rejection ABC as a rejection-sampling algorithm for the approximate posterior obtained by replacing  $\prob{\vect{x}_o\g\bm{\theta}}$ by the average conditional density near $\vect{x}_o$. An alternative way to view rejection ABC is as a kernel density estimate of the joint density $\prob{\bm{\theta}, \vect{x}}$ that has been conditioned on $\vect{x}_o$. Let $\set{\pair{\bm{\theta}_1}{\vect{x}_1}, \ldots, \pair{\bm{\theta}_N}{{\vect{x}_N}}}$ be a set of joint samples from $\prob{\bm{\theta}, \vect{x}}$, obtained as follows:
\begin{equation}
\bm{\theta}_n \sim \prob{\bm{\theta}}
\quad\text{and}\quad
\vect{x}_n \sim \prob{\vect{x}\g\bm{\theta}_n}.
\end{equation}
Let $k_{\epsilon}\br{\cdot}$ be a smoothing kernel, and let $\dirac{\cdot}$ be the delta distribution. We form the following approximation to the joint density, which is a kernel density estimate in $\vect{x}$ and an empirical distribution in $\bm{\theta}$:
\begin{equation}
\hat{p}\br{\bm{\theta}, \vect{x}} = \frac{1}{N}\sum_n k_{\epsilon}\br{\vect{x} - \vect{x}_n}\dirac{\bm{\theta} - \bm{\theta}_n}.
\end{equation}
By conditioning on $\vect{x}_o$ and normalizing, we obtain an approximate posterior in the form of a weighted empirical distribution:
\begin{equation}
\hat{p}\br{\bm{\theta}\g\vect{x}=\vect{x}_o} = \frac{\hat{p}\br{\bm{\theta}, \vect{x}_o}}{\integral{\hat{p}\br{\bm{\theta}, \vect{x}_o}}{\bm{\theta}}}
= \frac{\sum_n k_{\epsilon}\br{\vect{x}_o - \vect{x}_n}\dirac{\bm{\theta} - \bm{\theta}_n}}{\sum_n k_{\epsilon}\br{\vect{x}_o - \vect{x}_n}}.
\label{eq:efree:smooth_rej_abc}
\end{equation}
If we use the following uniform smoothing kernel:
\begin{equation}
k_{\epsilon}\br{\vect{u}} \propto I\br{\norm{\vect{u}} \le \epsilon} = \begin{cases}
1& \norm{\vect{u}} \le \epsilon\\
0& \mathrm{otherwise},
\end{cases}
\end{equation}
then $\hat{p}\br{\bm{\theta}\g\vect{x}=\vect{x}_o}$ is simply the empirical distribution of samples returned by rejection ABC\@. The above can be seen as a way to generalize rejection ABC to \emph{smooth-rejection ABC}, where we weight each prior sample $\bm{\theta}_n$ by a smoothing kernel $k_{\epsilon}\br{\vect{x}_o - \vect{x}_n}$ instead of accepting or rejecting it \citep{Beaumont:2002:abc_pop_gen}.

Additionally, we can interpret $\hat{p}\br{\bm{\theta}\g\vect{x}=\vect{x}_o}$ as the empirical distribution of the exact posterior of a different model, as given by importance sampling. Consider an alternative simulator model, obtained by adding noise $\vect{u}\sim k_{\epsilon}\br{\vect{u}}$ to the output $\vect{x}$ of the original simulator:
\begin{align}
\vect{x} &\sim \prob{\vect{x}\g\bm{\theta}}\\
\vect{u} &\sim k_{\epsilon}\br{\vect{u}}\\
\vect{x}' &= \vect{x} + \vect{u}.
\end{align}
Suppose we observed $\vect{x}'=\vect{x}_o$, and want to perform inference on $\bm{\theta}$. The posterior $\prob{\bm{\theta}, \vect{x}\g\vect{x}'=\vect{x}_o}$ can be written as:
\begin{equation}
\prob{\bm{\theta}, \vect{x}\g\vect{x}'=\vect{x}_o} \propto k_{\epsilon}\br{\vect{x}_o - \vect{x}}\,\prob{\bm{\theta}, \vect{x}}.
\end{equation}
We will form a weighted empirical distribution of the above posterior using importance sampling, with $\prob{\bm{\theta}, \vect{x}}$ as the proposal.
First, we sample $\set{\pair{\bm{\theta}_1}{\vect{x}_1}, \ldots, \pair{\bm{\theta}_N}{{\vect{x}_N}}}$ from $\prob{\bm{\theta}, \vect{x}}$, and then we form the following importance-weighted empirical distribution:
\begin{equation}
\hat{p}\br{\bm{\theta}, \vect{x}\g\vect{x}'=\vect{x}_o} = \frac{\sum_n k_{\epsilon}\br{\vect{x}_o - \vect{x}_n}\dirac{\bm{\theta} - \bm{\theta}_n}\dirac{\vect{x} - \vect{x}_n}}{\sum_n k_{\epsilon}\br{\vect{x}_o - \vect{x}_n}}.
\end{equation}
By marginalizing out $\vect{x}$ (i.e.~by discarding samples $\vect{x}_1, \ldots, \vect{x}_N$), we obtain the weighted empirical distribution $\hat{p}\br{\bm{\theta}\g\vect{x}=\vect{x}_o}$ as given by equation \eqref{eq:efree:smooth_rej_abc}. This shows that smooth-rejection ABC, and by extension rejection ABC, can be understood as sampling from the exact posterior of an alternative model. In this interpretation, $\epsilon$ controls the extent to which the alternative model is different to the original one. This interpretation of rejection ABC as sampling from a different model was proposed by \citet{Wilkinson:2013:abc_wrong_model} using rejection sampling; here I provide an alternative perspective using importance sampling instead.

\subsection{Markov-chain Monte Carlo ABC}
\label{sec:efree:abc:mcmc_abc}

As we saw in the previous section, rejection ABC proposes parameters from the prior $\prob{\bm{\theta}}$, and only accepts parameters that are likely under the approximate posterior $\prob{\bm{\theta}\g\norm{\vect{x}-\vect{x}_o}\le\epsilon}$. If the approximate posterior is significantly narrower than the prior, as is often the case in practice, a large fraction of proposed parameters will be rejected.

An alternative approach that can lead to fewer rejections is to propose parameters using Markov-chain Monte Carlo in the form of \emph{Metropolis--Hastings}. Instead of proposing parameters independently from the prior, Metropolis--Hastings proposes new parameters by e.g.~perturbing previously accepted parameters. If the perturbation is not too large, the proposed parameters are likely to be accepted too, which can lead to fewer rejections compared to proposing from the prior.

A likelihood-based Metropolis--Hastings algorithm that targets $\prob{\bm{\theta}\g\norm{\vect{x}-\vect{x}_o}\le\epsilon}$ is shown in algorithm \algref{alg:efree:mh}. The algorithm uses a proposal distribution $q\br{\bm{\theta}'\g\bm{\theta}}$ from which it proposes new parameters $\bm{\theta}'$ based on previously accepted parameters $\bm{\theta}$. Then, the following quantity is calculated, known as the \emph{acceptance ratio}:
\begin{equation}
\alpha = \frac
{\Prob{\norm{\vect{x}-\vect{x}_o}\le \epsilon\g\bm{\theta}'}\,\prob{\bm{\theta}'}\,q\br{\bm{\theta}\g\bm{\theta}'}}
{\Prob{\norm{\vect{x}-\vect{x}_o}\le \epsilon\g\bm{\theta}}\,\prob{\bm{\theta}}\,q\br{\bm{\theta}'\g\bm{\theta}}}.
\end{equation}
Finally, the algorithm outputs the proposed parameters $\bm{\theta}'$ with probability $\min\br{1, \alpha}$, otherwise it outputs the previous parameters $\bm{\theta}$.
It can be shown that this procedure leaves the approximate posterior invariant; that is, if $\bm{\theta}$ is a sample from $\prob{\bm{\theta}\g\norm{\vect{x}-\vect{x}_o}\le\epsilon}$, then the algorithm's output is also a sample from $\prob{\bm{\theta}\g\norm{\vect{x}-\vect{x}_o}\le\epsilon}$ \citep[for a proof, see][section 11.2.2]{Bishop:2006:prml}.

\begin{figure}[h]
\alglabel{alg:efree:mh}
\begin{algorithm}[H]
\DontPrintSemicolon
\setstretch{1.2}
Propose $\bm{\theta}'\sim q\br{\bm{\theta}'\g\bm{\theta}}$\;
Calculate acceptance ratio $\alpha = \frac
{\Prob{\norm{\vect{x}-\vect{x}_o}\le \epsilon\g\bm{\theta}'}\,\prob{\bm{\theta}'}\,q\br{\bm{\theta}\g\bm{\theta}'}}
{\Prob{\norm{\vect{x}-\vect{x}_o}\le \epsilon\g\bm{\theta}}\,\prob{\bm{\theta}}\,q\br{\bm{\theta}'\g\bm{\theta}}}$\;
Sample $u \sim \uniform{0}{1}$\;
\If{$u\le \alpha$}{\Return $\bm{\theta}'$\;}
\Else{\Return $\bm{\theta}$\;}
\caption{Likelihood-based Metropolis--Hastings\label{alg:efree:mh}}
\end{algorithm}
\end{figure}

In the likelihood-free setting, we cannot evaluate the approximate likelihood $\Prob{\norm{\vect{x}-\vect{x}_o}\le \epsilon\g\bm{\theta}}$, but we can estimate it as the fraction of simulated data that are at most distance $\epsilon$ away from the observed data $\vect{x}_o$:
\begin{equation}
\Prob{\norm{\vect{x}-\vect{x}_o}\le \epsilon\g\bm{\theta}} \approx \frac{1}{N}\sum_n I\br{\norm{\vect{x}_n-\vect{x}_o}\le \epsilon}
\quad\text{where}\quad
\vect{x}_n\sim \prob{\vect{x}\g\bm{\theta}}.
\label{eq:efree:pmh_estimate}
\end{equation}
The above estimate is unbiased for any number of simulations $N$, that is, its expectation is equal to $\Prob{\norm{\vect{x}-\vect{x}_o}\le \epsilon\g\bm{\theta}}$. It can be shown that Metropolis--Hastings leaves the target distribution invariant even if we use an unbiased estimate of the target distribution instead of its exact value in the acceptance ratio, as long as the estimate itself is made part of the Markov chain
\citep[for a proof, see][]{Andrieu:2009:pseudomarginal}. This means that we can use the above unbiased estimate of the approximate likelihood in the acceptance ratio and obtain a valid likelihood-free Metropolis--Hastings algorithm. This algorithm, which I refer to as \emph{pseudo-marginal Metropolis--Hastings}, is shown in algorithm \algref{alg:efree:pmh}.

\begin{figure}[h]
\alglabel{alg:efree:pmh}
\begin{algorithm}[H]
\DontPrintSemicolon
\setstretch{1.2}
Propose $\bm{\theta}'\sim q\br{\bm{\theta}'\g\bm{\theta}}$\;
\For{$n=1:N$}{Simulate $\vect{x}_n\sim\prob{\vect{x}\g\bm{\theta'}}$\;}
Calculate approximate-likelihood estimate $L' = \frac{1}{N}\sum_n I\br{\norm{\vect{x}_n-\vect{x}_o}\le \epsilon}$\;
Calculate acceptance ratio $\alpha = \frac
{L'\,\prob{\bm{\theta}'}\,q\br{\bm{\theta}\g\bm{\theta}'}}
{L\,\prob{\bm{\theta}}\,q\br{\bm{\theta}'\g\bm{\theta}}}$\;
Sample $u \sim \uniform{0}{1}$\;
\If{$u\le \alpha$}{\Return $\pair{\bm{\theta}'}{L'}$\;}
\Else{\Return $\pair{\bm{\theta}}{L}$\;}
\caption{Pseudo-marginal Metropolis--Hastings\label{alg:efree:pmh}}
\end{algorithm}
\end{figure}

The estimate in equation \eqref{eq:efree:pmh_estimate} is unbiased for any number of simulations $N$, hence pseudo-marginal Metropolis--Hastings is valid for any $N$. For larger $N$, the approximate-likelihood estimate has lower variance, but the runtime of the algorithm increases as more simulations are required per step. Overall, as discussed by \citet{Bornn:2017:single_pseudosample}, the efficiency of the algorithm doesn't necessarily increase with $N$, and so $N=1$ is usually good enough. In the case of $N=1$, i.e.~where only one datapoint $\vect{x}$ is simulated from $\prob{\vect{x}\g\bm{\theta}'}$ per step, and assuming $\bm{\theta}$ is a previously accepted sample, the acceptance ratio simplifies into the following:
\begin{equation}
\alpha = \begin{cases}
\frac
{\prob{\bm{\theta}'}\,q\br{\bm{\theta}\g\bm{\theta}'}}
{\prob{\bm{\theta}}\,q\br{\bm{\theta}'\g\bm{\theta}}} & \text{if }\norm{\vect{x} - \vect{x}_o} \le \epsilon,\\[0.4em]
0 & \text{otherwise.}
\end{cases}
\end{equation}
For $N=1$, pseudo-marginal Metropolis--Hastings  can be simplified into algorithm \algref{alg:efree:mcmc_abc}, which is known as \emph{Markov-chain Monte Carlo ABC}\@. MCMC-ABC was proposed by \citet{Marjoram:2003:mcmc_abc} prior to the introduction of pseudo-marginal MCMC by \citet{Andrieu:2009:pseudomarginal}. Here I provide an alternative perspective, where I describe MCMC-ABC as pseudo-marginal MCMC\@.

\begin{figure}[h]
\alglabel{alg:efree:mcmc_abc}
\begin{algorithm}[H]
\DontPrintSemicolon
\setstretch{1.2}
Propose $\bm{\theta}'\sim q\br{\bm{\theta}'\g\bm{\theta}}$\;
Simulate $\vect{x}\sim\prob{\vect{x}\g\bm{\theta'}}$\;
\If{$\norm{\vect{x}-\vect{x}_o}\le \epsilon$}{
Calculate acceptance ratio $\alpha = \frac
{\prob{\bm{\theta}'}\,q\br{\bm{\theta}\g\bm{\theta}'}}
{\prob{\bm{\theta}}\,q\br{\bm{\theta}'\g\bm{\theta}}}$\;
Sample $u \sim \uniform{0}{1}$\;
\If{$u\le \alpha$}{\Return $\bm{\theta}'$\;}
\Else{\Return $\bm{\theta}$\;}
}
\Else{\Return $\bm{\theta}$\;}
\caption{Markov-chain Monte Carlo ABC\label{alg:efree:mcmc_abc}}
\end{algorithm}
\end{figure}

Compared to rejection ABC, which proposes from the prior and thus can lead to most proposals being rejected, MCMC-ABC tends to accept more often. As discussed earlier, we can think of the proposal $q\br{\bm{\theta}'\g\bm{\theta}}$ as randomly `perturbing' accepted parameters $\bm{\theta}$ in order to propose new parameters $\bm{\theta}'$; if the perturbation is small, the proposed parameters are also likely to be accepted. However, unlike rejection ABC which produces independent samples, MCMC-ABC produces dependent samples, which can lead to low effective sample size.

Similarly to rejection ABC, the acceptance probability of MCMC-ABC vanishes as $\epsilon$ becomes small or as the data dimensionality becomes large. In practice, the data is typically transformed into summary statistics in order to reduce the data dimensionality. Another issue with MCMC-ABC is the initialization of the Markov chain: if the initial parameters $\bm{\theta}$ are unlikely to generate data near $\vect{x}_o$, the chain may be stuck at its initial state for a long time. In practice, one possibility is to run rejection ABC until one parameter sample is accepted, and initialize the Markov chain with that parameter sample.

\subsection{Sequential Monte Carlo ABC}
\label{sec:efree:abc:smc_abc}

As we have seen, the low acceptance rate of rejection ABC is partly due to parameters being proposed from the prior $\prob{\bm{\theta}}$. Indeed, if the approximate posterior $\prob{\bm{\theta}\g\norm{\vect{x}-\vect{x}_o}\le\epsilon}$ is significantly narrower than the prior, it's highly unlikely that a prior sample will be accepted. In the previous section, I discussed an approach for increasing the acceptance rate based on perturbing accepted parameters using Markov-chain Monte Carlo. In this section, I will discuss an alternative approach based on importance sampling.

To begin with, consider a variant of rejection ABC where instead of proposing parameters from the prior $\prob{\bm{\theta}}$, we propose parameters from an alternative distribution $\tilde{p}\br{\bm{\theta}}$. Parameter samples accepted under this procedure will be distributed according to the following distribution:
\begin{equation}
\tilde{p}\br{\bm{\theta}\g \norm{\vect{x}-\vect{x}_o}\le\epsilon}\propto
\Prob{\norm{\vect{x}-\vect{x}_o}\le\epsilon\g\bm{\theta}}\,\tilde{p}\br{\bm{\theta}}.
\end{equation}
Assuming $\tilde{p}\br{\bm{\theta}}$ is non-zero in the support of the prior, we can approximate $\prob{\bm{\theta}\g\norm{\vect{x}-\vect{x}_o}\le\epsilon}$ with a population of $N$ importance-weighted samples $\set{\pair{w_1}{\bm{\theta}_1}, \ldots, \pair{w_N}{\bm{\theta}_N}}$, where each $\bm{\theta}_n$ is obtained using the above procedure, and $\set{w_1, \ldots, w_N}$ are importance weights that are normalized to sum to $1$ and are calculated by:
\begin{equation}
w_n \propto \frac{\prob{\bm{\theta}_n\g\norm{\vect{x}-\vect{x}_o}\le\epsilon}}{\tilde{p}\br{\bm{\theta}_n\g \norm{\vect{x}-\vect{x}_o}\le\epsilon}}
\propto\frac{\Prob{\norm{\vect{x}-\vect{x}_o}\le\epsilon\g\bm{\theta}_n}\,\prob{\bm{\theta}_n}}{\Prob{\norm{\vect{x}-\vect{x}_o}\le\epsilon\g\bm{\theta}_n}\,\tilde{p}\br{\bm{\theta}_n}}
 = \frac{\prob{\bm{\theta}_n}}{\tilde{p}\br{\bm{\theta}_n}}.
\end{equation}
The above procedure, which I refer to as \emph{importance-sampling ABC}, is shown in algorithm \algref{alg:efree:is_abc}. Importance-sampling ABC reduces to rejection ABC if $\tilde{p}\br{\bm{\theta}} = \prob{\bm{\theta}}$.

\begin{figure}[h]
\alglabel{alg:efree:is_abc}
\begin{algorithm}[H]
\DontPrintSemicolon
\setstretch{1.2}
\For{$n=1:N$}{
\Repeat{$\norm{\vect{x} - \vect{x}_o}\le \epsilon$}{
Sample $\bm{\theta}_n\sim\tilde{p}\br{\bm{\theta}}$\;
Simulate $\vect{x}\sim\prob{\vect{x}\g\bm{\theta}_n}$\;
}
}
Calculate importance weights $w_n \propto \frac{\prob{\bm{\theta}_n}}{\tilde{p}\br{\bm{\theta}_n}}$ such that they sum to $1$\;
\Return $\set{\pair{w_1}{\bm{\theta}_1}, \ldots, \pair{w_N}{\bm{\theta}_N}}$\;
\caption{Importance-sampling ABC\label{alg:efree:is_abc}}
\end{algorithm}
\end{figure}

The efficiency of importance-sampling ABC depends heavily on the choice of the proposal $\tilde{p}\br{\bm{\theta}}$. If the proposal is too broad the acceptance rate will be low, whereas if the proposal is too narrow the importance weights will have high variance \citep[see][for a discussion on the optimality of the proposal]{Alsing:2018:optimal_abc_proposal}.
In practice, one approach for constructing a proposal is to perform a preliminary run of importance-sampling ABC with $\epsilon' >\epsilon$, obtain a population of importance-weighted parameter samples $\set{\pair{w_1}{\bm{\theta}_1}, \ldots, \pair{w_N}{\bm{\theta}_N}}$, and take $\tilde{p}\br{\bm{\theta}}$ to be a mixture distribution defined by:
\begin{equation}
\tilde{p}\br{\bm{\theta}} = \sum_n w_n \,q\br{\bm{\theta}\g\bm{\theta}_n}.
\end{equation}
Similar to MCMC-ABC, $q\br{\bm{\theta}\g\bm{\theta}_n}$ can be thought of as a way to perturb $\bm{\theta}_n$. For instance, a common choice is to take $q\br{\bm{\theta}\g\bm{\theta}_n}$ to be a Gaussian distribution centred at $\bm{\theta}_n$, which corresponds to adding zero-mean Gaussian noise to $\bm{\theta}_n$.

Based on the above idea, we can run a sequence of $T$ rounds of importance-sampling ABC with $\epsilon_1 > \ldots > \epsilon_T$, and construct the proposal for round $t$ using the importance-weighted population obtained in round $t-1$. The first round can use the prior $\prob{\bm{\theta}}$ as proposal and a large enough $\epsilon_1$ so that the acceptance rate is not too low. This procedure is known as \emph{sequential Monte Carlo ABC}, and is shown in algorithm \algref{alg:efree:smc_abc}. Further discussion on SMC-ABC and different variants of the algorithm are provided by \citet{Sisson:2007:smc_abc, Bonassi:2015:smc_abc, Beaumont:2009:smc_abc, Toni:2009:smc_abc}.

\begin{figure}[h]
\alglabel{alg:efree:smc_abc}
\begin{algorithm}[H]
\DontPrintSemicolon
\setstretch{1.2}
Run rejection ABC with tolerance $\epsilon_1$\;
Obtain initial population $\set{\bm{\theta}_1^{\br{1}}, \ldots, \bm{\theta}_N^{\br{1}}}$\;
Set initial weights $w_n^{\br{1}}\propto 1$\;
\For{$t=2:T$}{
Set proposal $\tilde{p}\br{\bm{\theta}} = \sum_n w_n^{\br{t-1}} \,q\br{\bm{\theta}\g\bm{\theta}_n^{\br{t-1}}}$\;
Run importance-sampling ABC with tolerance $\epsilon_t$ and proposal $\tilde{p}\br{\bm{\theta}}$\;
Obtain weighted population $\set{\pair{w_1^{\br{t}}}{\bm{\theta}_1^{\br{t}}}, \ldots, \pair{w_N^{\br{t}}}{\bm{\theta}_N^{\br{t}}}}$\;
Estimate effective sample size $\hat{S} = \br{\sum_n w_n^2}^{-1}$\;
\If{$\hat{S}<S_\mathrm{min}$}{
Resample $\set{\bm{\theta}_1^{\br{t}}, \ldots, \bm{\theta}_N^{\br{t}}}$ with probabilities $\set{w_1^{\br{t}}, \ldots, w_N^{\br{t}}}$\;
Set weights $w_n^{\br{t}}\propto 1$\;
}
}
\Return $\set{\pair{w_1^{\br{T}}}{\bm{\theta}_1^{\br{T}}}, \ldots, \pair{w_N^{\br{T}}}{\bm{\theta}_N^{\br{T}}}}$\;
\caption{Sequential Monte Carlo ABC\label{alg:efree:smc_abc}}
\end{algorithm}
\end{figure}

An issue with sequential Monte Carlo more broadly (and not just in the context of ABC) is \emph{sample degeneracy}, which occurs when only few samples have non-negligible weights. If that's the case, the effective sample size, i.e.~the number of independent samples that would be equivalent to the weighted population, is significantly less than the population size $N$. One strategy for avoiding sample degeneracy, which is shown in algorithm \algref{alg:efree:smc_abc}, is to estimate the effective sample size after each round of importance-sampling ABC; if the estimated effective sample size is less than a threshold (e.g.~$N/2$), we resample the population with probabilities given by the weights in order to obtain a new population with equal weights. In the context of ABC in particular, the above resampling may not be necessary, since sampling the next population from the proposal distribution already diversifies its samples. In any case, it's a good idea to monitor the effective sample size of each population as the algorithm progresses, as a consistently low sample size can serve as a warning of potentially poor performance. In practice, the effective sample size $S$ of a weighted population $\set{\pair{w_1}{\bm{\theta}_1}, \ldots, \pair{w_N}{\bm{\theta}_N}}$ can be estimated by:
\begin{equation}
\hat{S} = \frac{1}{\sum_n w_n^2}.
\end{equation}
A justification of the above estimate is provided by \citet{Nowozin:2015:ess}.

SMC-ABC is a strong baseline for likelihood-free inference; its acceptance rate is typically higher than rejection ABC, and it can be easier to tune  than MCMC-ABC\@. One way to use SMC-ABC in practice is to start with a tolerance $\epsilon_1$ that is large enough for rejection ABC to be practical, and exponentially decay it after each round in order to obtain increasingly accurate posterior samples.
Nevertheless, even though SMC-ABC can produce a good proposal distribution after a few rounds, it doesn't solve the problem of the acceptance probability eventually vanishing as $\epsilon$ becomes small. In practice,  as $\epsilon$ decreases in every round, the required number of simulations increases, and can thus become prohibitively large after a few rounds.

\section{The paper}
\label{sec:efree:paper}

This section presents the paper \emph{Fast $\epsilon$-free Inference of Simulator Models with Bayesian Conditional Density Estimation}, which is the main contribution of this chapter. In the paper, we cast likelihood-free inference as density-estimation problem: we estimate the posterior density using a neural density estimator trained on data simulated from the model. We employ a Bayesian mixture-density network as our neural density estimator, and we describe how to use the network in order to guide simulations and thus accelerate the training process.

The paper was first published as a preprint on arXiv in May 2016. Afterwards, it was accepted for publication at the conference \emph{Advances in Neural Information Processing Systems (NeurIPS)} in December 2016. There were about $2{,}500$ submissions to NeurIPS in 2016, of which $568$ were accepted for publication.

\subsubsection{Author contributions}

The paper is co-authored by me and Iain Murray. We jointly conceived and developed the method. As the leading author, I wrote the code, performed the experiments, and drafted the paper. Iain Murray supervised the project, and revised the final draft.

\subsubsection{Differences in notation}

For precision, in this chapter I use $\prob{\cdot}$ for probability densities and $\Prob{\cdot}$ for probability measures. The paper doesn't make this distinction; instead, it uses the term `probability' for both probability densities and probability measures, and denotes both of them by $\prob{\cdot}$ or similar. This overloaded notation, common in the machine-learning literature, should be interpreted appropriately depending on context.

Throughout this chapter I use $\norm{\vect{x}-\vect{x}_o}\le\epsilon$ to denote the acceptance condition of ABC; the paper uses $\norm{\vect{x}-\vect{x}_o}<\epsilon$ instead. The difference is immaterial, since the set $\set{\vect{x}: \norm{\vect{x}-\vect{x}_o}=\epsilon}$ has zero probability measure by assumption.

Finally, the paper uses $\langle f\br{\vect{x}}\rangle_{\prob{\vect{x}}}$ to denote the expectation of a function $f\br{\vect{x}}$ with respect to a distribution $\prob{\vect{x}}$, whereas the rest of the thesis uses $\avgx{f\br{\vect{x}}}{\prob{\vect{x}}}$.

\subsubsection{Minor corrections}

Section 2.1 of the paper states that $\vect{x}=\vect{x}_o$ corresponds to setting $\epsilon=0$ in $\norm{\vect{x}-\vect{x}_o}<\epsilon$. Strictly speaking, this statement is incorrect, since $\set{\vect{x}: \norm{\vect{x}-\vect{x}_o}<0}$ is the empty set. For this reason, in this chapter I use $\norm{\vect{x}-\vect{x}_o}\le\epsilon$, such that setting $\epsilon$ to zero correctly implies $\vect{x}=\vect{x}_o$.

Section 4 of the paper misrepresents the work of \citet{Fan:2013} as a synthetic likelihood method that trains ``a separate density model of
$\vect{x}$ for each $\bm{\theta}$ by repeatedly simulating the model for fixed $\bm{\theta}$''. This is incorrect; in reality, \citet{Fan:2013} use a single conditional-density model of $\vect{x}$ given $\bm{\theta}$ that is trained on pairs $\pair{\bm{\theta}}{\vect{x}}$ independently simulated from a joint density $r\br{\bm{\theta}}\,\prob{\vect{x}\g\bm{\theta}}$, where $r\br{\bm{\theta}}$ is some reference density.

\includepdf[pages=-]{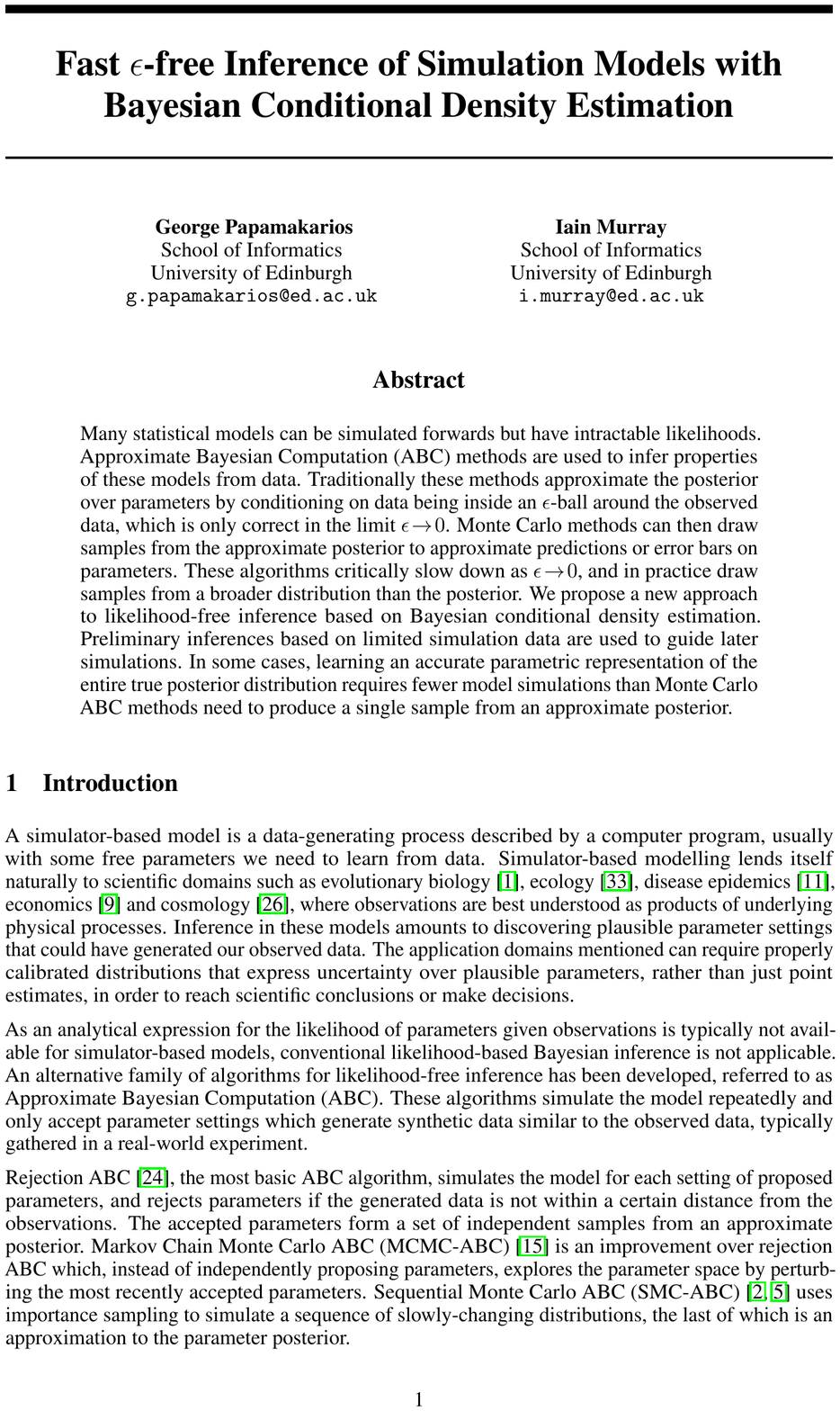}

\section{Discussion}
\label{sec:efree:discussion}

I conclude this chapter with a discussion of the paper \emph{Fast $\epsilon$-free Inference of Simulator Models with Bayesian Conditional Density Estimation}, which was presented in the previous section. In what follows, I will discuss the paper's contribution and impact so far, I will compare the proposed method with regression adjustment, I will criticize the paper's limitations, and I will discuss further advances inspired by the proposed method.

\subsection{Contribution and impact}

The paper establishes neural density estimation as a viable method for likelihood-free inference. Inference is viewed as a density-estimation problem, where the unknown posterior is the target density to estimate, and the simulator is the source of training data. Unlike ABC methods such as those discussed in section \ref{sec:efree:abc}, likelihood-free inference via neural density estimation doesn't involve accepting/rejecting parameter samples, and consequently it doesn't require specifying a tolerance $\epsilon$. The paper demonstrates that it is possible for neural density models to estimate joint posterior densities accurately, and without the need to replace the exact posterior $\prob{\bm{\theta}\g\vect{x}=\vect{x}_o}$ by an approximate posterior $\prob{\bm{\theta}\g\norm{\vect{x}-\vect{x}_o}\le \epsilon}$ as typically done in ABC\@.

Furthermore, the paper demonstrates that likelihood-free inference via neural density estimation can lead to massive computational savings in simulation cost compared to traditional ABC\@. This is achieved by a sequential training procedure, where preliminary estimates of the posterior are used as proposals for obtaining more training data. This sequential procedure is similar to the one used in sequential Monte Carlo ABC, where new parameters are proposed by perturbing previous posterior samples obtained with a larger $\epsilon$ (as explained in section \ref{sec:efree:abc:smc_abc}). As the paper shows, the proposed method can estimate the joint density of the entire posterior with roughly the same number of simulations as SMC-ABC needs to produce (the equivalent of) a single independent posterior sample.

According to Google Scholar, the paper has received $41$ citations as of \thedate. The proposed method has directly influenced subsequent developments in likelihood-free inference, such as \emph{Sequential Neural Posterior Estimation} \citep{Lueckmann:2017:snpe}, which will be discussed in more detail in section \ref{sec:efree:discussion:snpe}, \emph{Sequential Neural Likelihood} \citep{Papamakarios:2019:snl}, which is the topic of next chapter, and \emph{adaptive Gaussian-copula ABC} \citep{Chen:2019:copula_abc}. Finally, variants of the proposed method have been used in applications of likelihood-free inference in cosmology \citep{Alsing:2018:lfi_cosmology, Alsing:2019:lfi_cosmology} and in neuroscience \citep{Lueckmann:2017:snpe}.

\subsection{Comparison with regression adjustment}

As mentioned in section 4 of the paper, \emph{regression adjustment} is closely related to our work, and can be seen as a predecessor to likelihood-free inference via neural density estimation. Here, I discuss in more depth what regression adjustment is, in what way it is similar to neural density estimation, and also how it differs.

Regression adjustment is typically performed as a post-processing step after rejection ABC, to correct for the approximation incurred by using a non-zero tolerance $\epsilon$. Suppose we run rejection ABC with some tolerance $\epsilon$, until we obtain a set of $N$ joint samples $\set{\pair{\bm{\theta}_1}{\vect{x}_1}, \ldots, \pair{\bm{\theta}_N}{\vect{x}_N}}$. As we have seen, each pair $\pair{\bm{\theta}_n}{\vect{x}_n}$ is an independent joint sample from the following distribution:
\begin{equation}
\prob{\bm{\theta},\vect{x}\g\norm{\vect{x}-\vect{x}_o}\le\epsilon}\propto
I\br{\norm{\vect{x}-\vect{x}_o}\le\epsilon}\,\prob{\vect{x}\g\bm{\theta}}\,\prob{\bm{\theta}}.
\end{equation}
Marginally, each $\bm{\theta}_n$ is an independent sample from the approximate posterior $\prob{\bm{\theta}\g\norm{\vect{x}-\vect{x}_o}\le\epsilon}$. If $\epsilon =\infty$ then all simulations are accepted, in which case each $\bm{\theta}_n$ is an independent sample from the prior $\prob{\bm{\theta}}$. The goal of regression adjustment is to transform each sample $\bm{\theta}_n$ into an `adjusted' sample $\bm{\theta}'_n$, such that $\bm{\theta}'_n$ is an independent sample from the exact posterior $\prob{\bm{\theta}\g\vect{x}=\vect{x}_o}$.

The first step of regression adjustment is to estimate the stochastic mapping from $\vect{x}$ to $\bm{\theta}$, where $\vect{x}$ and $\bm{\theta}$ are obtained by rejection ABC and are jointly distributed according to $\prob{\bm{\theta},\vect{x}\g\norm{\vect{x}-\vect{x}_o}\le\epsilon}$ as described above. Regression adjustment models this mapping with a parametric regressor of the following form:
\begin{equation}
\bm{\theta} = f_{\bm{\phi}}\br{\vect{u}, \vect{x}},
\end{equation}
where $\vect{u}$ is random noise from some distribution $\pi_u\br{\vect{u}}$ that doesn't depend on $\vect{x}$, and $\bm{\phi}$ are the parameters of the regressor. The random noise $\vect{u}$ accounts for the stochasticity of the mapping from $\vect{x}$ to $\bm{\theta}$. We don't need to specify the noise distribution $\pi_u\br{\vect{u}}$, but we need to design the function $f_{\bm{\phi}}\br{\cdot, \vect{x}}$ to be invertible, so that given $\bm{\theta}$ and $\vect{x}$ we can easily solve for $\vect{u}$. For example, \citet{Beaumont:2002:abc_pop_gen} used the following linear regressor:
\begin{equation}
\bm{\theta} = \mat{A}\vect{x} + \bm{\beta} + \vect{u}
\quad\text{with}\quad
\bm{\phi} = \set{\mat{A}, \bm{\beta}},
\label{eq:efree:reg_abc_linear}
\end{equation}
whereas \citet{Blum:2010:reg_abc} used the following non-linear regressor:
\begin{equation}
\bm{\theta} = g_{\bm{\phi}_1}\br{\vect{x}} + h_{\bm{\phi}_2}\br{\vect{x}} \odot \vect{u}
\quad\text{with}\quad
\bm{\phi} = \set{\bm{\phi}_1, \bm{\phi}_2}.
\label{eq:efree:reg_abc_nonlinear}
\end{equation}
In the above, $g_{\bm{\phi}_1}$ and $h_{\bm{\phi}_2}$ are neural networks parameterized by $\bm{\phi}_1$ and $\bm{\phi}_2$, and $\odot$ denotes the elementwise product.

In practice, the regressor $f_{\bm{\phi}}$ is trained on the set of samples $\set{\pair{\bm{\theta}_1}{\vect{x}_1}, \ldots, \pair{\bm{\theta}_N}{\vect{x}_N}}$ obtained by rejection ABC\@. Assuming the noise $\vect{u}$ has zero mean, the linear regressor in equation \eqref{eq:efree:reg_abc_linear} can be trained by minimizing the average squared error on the parameters as follows:
\begin{equation}
\br{\mat{A}\!^*, \bm{\beta}^*} = \argmin_{\mat{A}, \bm{\beta}}\,\frac{1}{N}\sum_n \norm{\bm{\theta}_n -  \mat{A}\vect{x}_n - \bm{\beta}}^2.
\end{equation}
The above optimization problem has a unique solution in closed form.
The non-linear regressor in equation \eqref{eq:efree:reg_abc_nonlinear} can be trained in two stages. In the first stage, we assume that $\vect{u}$ has zero mean, and we fit $\bm{\phi}_1$ by minimizing the average squared error on the parameters:
\begin{equation}
\bm{\phi}_1^* = \argmin_{\bm{\phi}_1}\,\frac{1}{N}\sum_n \norm{\bm{\theta}_n -  g_{\bm{\phi}_1}\br{\vect{x}_n}}^2.
\end{equation}
In the second stage, we assume that $\log{\abs{\vect{u}}}$ has zero mean, and we fit $\bm{\phi}_2$ by minimizing the average squared error on the log absolute residuals:
\begin{equation}
\bm{\phi}_2^* = \argmin_{\bm{\phi}_2}\,\frac{1}{N}\sum_n \norm{\log\abs{\bm{\theta}_n - g_{\bm{\phi}_1^*}\br{\vect{x}_n}} -  \log\abs{h_{\bm{\phi}_2}\br{\vect{x}_n}}}^2,
\end{equation}
where the log and absolute-value operators are meant to be understood elementwise.
The above two optimization problems don't have a known closed-form solution, but can be approximately solved by gradient methods.

The second step of regression adjustment is to use the trained regressor $f_{\bm{\phi}^*}$ to adjust the parameter samples $\set{\bm{\theta}_1, \ldots, \bm{\theta}_N}$ obtained by rejection ABC\@. First, we can obtain $N$ independent samples $\set{\vect{u}_1, \ldots, \vect{u}_N}$ from $\pi_u\br{\vect{u}}$ by solving $\bm{\theta}_n = f_{\bm{\phi}^*}\br{\vect{u}_n, \vect{x}_n}$ for $\vect{u}_n$ separately for each $n$. Then, assuming $f_{\bm{\phi}^*}$ captures the exact stochastic mapping from $\vect{x}$ to $\bm{\theta}$, we can generate $N$ samples $\set{\bm{\theta}'_1, \ldots, \bm{\theta}'_N}$ from the exact posterior $\prob{\bm{\theta}\g\vect{x}=\vect{x}_o}$ by simply plugging $\vect{u}_n$ and $\vect{x}_o$ back into the regressor as follows:
\begin{equation}
\bm{\theta}'_n = f_{\bm{\phi}^*}\br{\vect{u}_n, \vect{x}_o}.
\end{equation}
For example, using the linear regressor in equation \eqref{eq:efree:reg_abc_linear} we obtain:
\begin{equation}
\bm{\theta}'_n = \mat{A}\!^*\,\br{\vect{x}_o - \vect{x}_n} + \bm{\theta}_n,
\end{equation}
whereas using the non-linear regressor in equation \eqref{eq:efree:reg_abc_nonlinear} we obtain:
\begin{equation}
\bm{\theta}'_n = g_{\bm{\phi}_1^*}\br{\vect{x}_o} + h_{\bm{\phi}_2^*}\br{\vect{x}_o}\odot
\frac{\bm{\theta}_n - g_{\bm{\phi}_1^*}\br{\vect{x}_n}}{h_{\bm{\phi}_2^*}\br{\vect{x}_n}},
\end{equation}
where division is meant to be understood elementwise.

In practice, $f_{\bm{\phi}^*}$ will be an approximation to the exact stochastic mapping from $\vect{x}$ to $\bm{\theta}$, hence the adjusted samples $\set{\bm{\theta}'_1, \ldots, \bm{\theta}'_N}$ will only approximately follow the exact posterior. For the adjusted samples to be a good description of the exact posterior, $f_{\bm{\phi}^*}$ needs to be a good model of the stochastic mapping from $\vect{x}$ to $\bm{\theta}$ at least within a distance $\epsilon$ away from $\vect{x}_o$. Consequently, a less flexible regressor (such as the linear regressor described above) requires a smaller $\epsilon$, whereas a more flexible regressor can be used with a larger $\epsilon$.

Similar to regression adjustment, likelihood-free inference via neural density estimation is based on modelling the stochastic relationship between $\bm{\theta}$ and $\vect{x}$. The difference is in the type of model employed and in the way the model is used. Neural density estimation uses a conditional-density model $q_{\bm{\phi}}\br{\bm{\theta}\g\vect{x}}$ (which in the paper is taken to be a mixture-density network) with the following characteristics:
\begin{enumerate}[label=(\roman*)]
\item We can calculate the conditional density under $q_{\bm{\phi}}\br{\bm{\theta}\g\vect{x}}$ for any $\bm{\theta}$ and $\vect{x}$.
\item We can obtain exact independent samples from $q_{\bm{\phi}}\br{\bm{\theta}\g\vect{x}}$ for any  $\vect{x}$.
\end{enumerate}
The first property allows us to train $q_{\bm{\phi}}\br{\bm{\theta}\g\vect{x}}$ by maximum likelihood on simulated data, and thus estimate the density of the posterior. The second property allows us to obtain as many samples from the estimated posterior as we wish.

In contrast, regression adjustment models the stochastic relationship between $\bm{\theta}$ and $\vect{x}$ as follows:
\begin{equation}
\bm{\theta} = f_{\bm{\phi}}\br{\vect{u}, \vect{x}}
\quad\text{where}\quad
\vect{u}\sim\pi_u\br{\vect{u}}.
\label{eq:efree:reg_abc_model}
\end{equation}
We can think of the above as a reparameterized version of a neural density estimator in terms of its internal random numbers (the chapter on Masked Autoregressive Flow explores this idea in depth). The crucial difference is that regression adjustment only requires fitting the regressor $f_{\bm{\phi}}$, and it entirely avoids modelling the noise distribution $\pi_u\br{\vect{u}}$. As a consequence, regression adjustment can neither calculate densities under the model, nor generate samples from the model at will, as these operations require knowing $\pi_u\br{\vect{u}}$.

In principle,  we can use a neural density estimator $q_{\bm{\phi}}\br{\bm{\theta}\g\vect{x}}$ to perform regression adjustment, as long as we can write $q_{\bm{\phi}}\br{\bm{\theta}\g\vect{x}}$ as an invertible transformation of noise as in equation \eqref{eq:efree:reg_abc_model}. A mixture-density network cannot be easily expressed this way, but a model such as Masked Autoregressive Flow (or any other normalizing flow) is precisely an invertible transformation of noise by design. Hence, we can use a normalizing flow  to estimate the conditional density of $\bm{\theta}$ given $\vect{x}$, and then obtain posterior samples by adjusting simulated data. 
More precisely, given a simulated pair $\pair{\bm{\theta}_n}{\vect{x}_n}$, we can use the trained normalizing flow to compute the noise $\vect{u}_n$ that would 
have generated $\bm{\theta}_n$ given $\vect{x}_n$, and then run the flow with noise $\vect{u}_n$ but this time conditioned on $\vect{x}_o$ to obtain a sample $\bm{\theta}'_n$ from the exact posterior $\prob{\bm{\theta}\g\vect{x}=\vect{x}_o}$.

Nevertheless, I would argue that using a neural density estimator such as a normalizing flow for regression adjustment is an unnecessarily roundabout way of generating posterior samples. If we assume that the flow is a good model of the stochastic relationship between $\bm{\theta}$ and $\vect{x}$ 
(which is what regression adjustment requires), then we may as well obtain posterior samples by sampling the noise directly from $\pi_u\br{\vect{u}}$ (recall that the flow explicitly models the noise distribution). 
In other words, with neural density estimation there is no need to adjust simulated data to obtain posterior samples; we can obtain as many posterior samples as we like by sampling from the neural density model directly.

\subsection{Limitations and criticism}
\label{sec:efree:discussion:lims}

So far I've discussed the contributions of the paper and the impact the paper has had in the field of likelihood-free inference. In this section, I examine the paper from a critical perspective: I discuss the limitations of the proposed method, and identify its weaknesses.

Recall that the proposed method estimates the posterior in a series of rounds. In each round, parameter samples are proposed from a distribution $\tilde{p}\br{\bm{\theta}}$, then the neural density estimator $q_{\bm{\phi}}\br{\bm{\theta}\g\vect{x}}$ is trained on simulated data, and finally the posterior is estimated as follows:
\begin{equation}
\hat{p}\br{\bm{\theta}\g\vect{x}=\vect{x}_o}\propto\frac{\prob{\bm{\theta}}}{\tilde{p}\br{\bm{\theta}}}\,q_{\bm{\phi}}\br{\bm{\theta}\g\vect{x}_o}.
\label{eq:efree:posterior_est_snpe_a}
\end{equation}
The above posterior estimate is then used as the proposal for the next round, and the algorithm can thus progress for any number of rounds. In the first round, we may use the prior as proposal.

In order to ensure that the computation of the posterior estimate in equation \eqref{eq:efree:posterior_est_snpe_a} is analytically tractable, the proposed method restricts $q_{\bm{\phi}}\br{\bm{\theta}\g\vect{x}}$ to be a conditional Gaussian mixture model (i.e.~a mixture-density network), and assumes that the prior $\prob{\bm{\theta}}$ and the proposal $\tilde{p}\br{\bm{\theta}}$ are also Gaussian. The Gaussian assumption can be relaxed to some extent; equation \eqref{eq:efree:posterior_est_snpe_a} remains analytically tractable if we replace Gaussians with any other exponential family. However, the proposed algorithm is still tied to mixture-density networks and can't employ more advanced neural density estimators, such as Masked Autoregressive Flow. In addition, the algorithm can't be used with arbitrary priors, which limits its applicability.

Even though mixture-density networks can be made arbitrarily flexible, I would argue that, in principle, tying an inference algorithm to a particular neural density estimator should be viewed as a weakness. Having the inference algorithm and the neural density model as separate components of an inference framework encourages modularity and good software design. Furthermore, if an inference algorithm can interface with any neural density model, improvements in neural density estimation (such as those described in the previous chapter) will automatically translate to improvements in the inference algorithm.

An important limitation of the proposed method is that equation \eqref{eq:efree:posterior_est_snpe_a} is not guaranteed to yield a posterior estimate that integrates to $1$. For example, assume that $\prob{\bm{\theta}}\propto 1$, that $q_{\bm{\phi}}\br{\bm{\theta}\g\vect{x}_o}$ has a Gaussian component with covariance $\mat{S}$, and that $\tilde{p}\br{\bm{\theta}}$ is a Gaussian with covariance $\alpha\mat{S}$ for some $\alpha < 1$; in other words, the proposal is narrower than one of the components of $q_{\bm{\phi}}\br{\bm{\theta}\g\vect{x}_o}$. In this case, dividing $q_{\bm{\phi}}\br{\bm{\theta}\g\vect{x}_o}$ by the proposal yields a Gaussian component with covariance:
\begin{equation}
\br{\mat{S}^{-1} - \br{\alpha\mat{S}}^{-1}}^{-1} = \frac{\alpha}{\alpha - 1}\,\mat{S},
\end{equation}
which is negative-definite. When this happens, the algorithm has no good way of continuing to make progress; in practice it terminates early and returns the posterior estimate from the previous round. The problem of negative-definite covariances resulting from dividing one Gaussian by another is also encountered in \emph{expectation propagation} \citep{Minka:2001:EP}.

In appendix C, the paper states 
the following:
\begin{quote}
[The neural density estimator] 
$q_{\bm{\phi}}\br{\bm{\theta}\g\vect{x}_o}$ is trained on parameters sampled from $\tilde{p}\br{\bm{\theta}}$, hence, if trained properly, it tends to be narrower than $\tilde{p}\br{\bm{\theta}}$.
\end{quote}
While in practice this is often the case, I should emphasize that there is no guarantee that $q_{\bm{\phi}}\br{\bm{\theta}\g\vect{x}_o}$ will always be narrower than the proposal. Furthermore, the paper proceeds to state the following:
\begin{quote}
[The covariance] not being positive-definite rarely happens, whereas it happening is an
indication that the algorithm's parameters have not been set up properly.
\end{quote}
This was indeed my experience in the experiments of the paper. However, as I continued to use the algorithm after the publication of the paper, I found that sometimes negative-definite covariances are hard to avoid, regardless of how the algorithm's parameters are set up. In fact, in the next chapter I give two examples of inference tasks in which the algorithm fails after its second round.

Since the publication of the paper, the above two issues, i.e.~restricting the type of neural density estimator and terminating due to an improper posterior estimate, have motivated further research that aims to improve upon the proposed algorithm. In section \ref{sec:efree:discussion:snpe} and in the next chapter, I will discuss methods that were developed to address these two issues.

\subsection{Sequential Neural Posterior Estimation}
\label{sec:efree:discussion:snpe}

About a year after the publication of the paper, the proposed method was extended by \citet{Lueckmann:2017:snpe}, partly in order to address the limitations described in the previous section. \citet{Lueckmann:2017:snpe} called their method \emph{Sequential Neural Posterior Estimation}. I thought that \emph{Sequential Neural Posterior Estimation} was an excellent name indeed, and that it would be an appropriate name for our method as well, which until then had remained nameless. With Jan-Matthis Lueckmann's permission, I now use the term SNPE to refer to both methods; if I need to distinguish between them, I use SNPE-A for our method and SNPE-B for the method of \citet{Lueckmann:2017:snpe}.

Similar to SNPE-A, SNPE-B estimates the posterior density by training a neural density estimator over multiple rounds, where the posterior estimate of each round becomes the proposal for the next round. The main difference between SNPE-A and SNPE-B is the strategy employed in order to correct for sampling parameters from the proposal $\tilde{p}\br{\bm{\theta}}$ instead of the prior $\prob{\bm{\theta}}$. Recall that SNPE-A trains the neural density estimator $q_{\bm{\phi}}\br{\bm{\theta}\g\vect{x}}$ on samples from the proposal, and then analytically adjusts $q_{\bm{\phi}}\br{\bm{\theta}\g\vect{x}_o}$ by multiplying with $\prob{\bm{\theta}}$ and dividing by $\tilde{p}\br{\bm{\theta}}$, as shown in equation \eqref{eq:efree:posterior_est_snpe_a}. As we've seen, this strategy restricts the form of the neural density estimator, and may lead to posterior estimates that don't integrate to $1$.

In contrast, SNPE-B assigns importance weights to the proposed samples, and then trains the neural density estimator on the weighted samples. Let $\set{\pair{\bm{\theta}_1}{\vect{x}_1}, \ldots, \pair{\bm{\theta}_N}{\vect{x}_N}}$ be a set of $N$ independent joint samples
obtained by:
\begin{equation}
\bm{\theta}_n \sim \tilde{p}\br{\bm{\theta}}
\quad\text{and}\quad
\vect{x}_n\sim\prob{\vect{x}\g\bm{\theta}_n}.
\end{equation}
Assuming $\tilde{p}\br{\bm{\theta}}$ is non-zero in the support of the prior, SNPE-B assigns an importance weight $w_n$ to each joint sample $\pair{\bm{\theta}_n}{\vect{x}_n}$, given by:
\begin{equation}
w_n = \frac{\prob{\bm{\theta}_n}}{\tilde{p}\br{\bm{\theta}_n}}.
\end{equation}
Then, the neural density estimator is trained by maximizing the average log likelihood on the weighted dataset, given by:
\begin{equation}
L\br{\bm{\phi}} = \frac{1}{N}\sum_n w_n\log q_{\bm{\phi}}\br{\bm{\theta}_n\g\vect{x}_n}.
\end{equation}
As $N\rightarrow \infty$, due to the strong law of large numbers, $L\br{\bm{\phi}}$ converges almost surely to the following expectation:
\begin{align}
L\br{\bm{\phi}}\,\,\xrightarrow{a.s.}\,\,
\avgx{\frac{\prob{\bm{\theta}}}{\tilde{p}\br{\bm{\theta}}}\log q_{\bm{\phi}}\br{\bm{\theta}\g\vect{x}}}{\tilde{p}\br{\bm{\theta}}\prob{\vect{x}\g\bm{\theta}}}
= \avgx{\log q_{\bm{\phi}}\br{\bm{\theta}\g\vect{x}}}{\prob{\bm{\theta}, \vect{x}}}.
\end{align}
Maximizing the above expectation is equivalent to minimizing $\kl{\prob{\bm{\theta}, \vect{x}}}{q_{\bm{\phi}}\br{\bm{\theta}\g\vect{x}}\,\prob{\vect{x}}}$, which happens only if $q_{\bm{\phi}}\br{\bm{\theta}\g\vect{x}} = \prob{\bm{\theta}\g\vect{x}}$ almost everywhere. Therefore, the training procedure of SNPE-B targets the exact posterior (at least asymptotically), which means that we can estimate the exact posterior simply by:
\begin{equation}
\hat{p}\br{\bm{\theta}\g\vect{x}=\vect{x}_o} = q_{\bm{\phi}}\br{\bm{\theta}\g\vect{x}_o}.
\end{equation}

Unlike SNPE-A where the neural density estimator must be a mixture-density network and the prior and the proposal must be Gaussian, SNPE-B imposes no restrictions in the form of the neural density estimator, the proposal or the prior. This means that the proposal doesn't need to be Gaussian, but can be an arbitrary density model. Moreover, since the neural density estimator targets the posterior directly, we are guaranteed that the posterior estimate is a proper density, and the algorithm can proceed for any number of rounds without the risk of terminating early.

Nonetheless, SNPE-B has its own weaknesses. If the proposal $\tilde{p}\br{\bm{\theta}}$ is not a good match to the prior $\prob{\bm{\theta}}$ (which is likely in later rounds), the importance weights will have high variance, which means that the average log likelihood $L\br{\bm{\phi}}$ will have high variance too. As a consequence, training can be unstable, and the trained model may not be an accurate posterior estimate. As we will see in the next chapter, SNPE-B typically produces less accurate results than SNPE-A, at least when SNPE-A doesn't terminate early.

In conclusion, although SNPE-B is an improvement over SNPE-A in terms of robustness and flexibility, I would argue that there is still scope for further developments in order to improve the performance and robustness of current methods. In the next chapter, I will discuss such a development, which I refer to as \emph{Sequential Neural Likelihood}; SNL is an alternative to both SNPE-A and SNPE-B in that it targets the likelihood instead of the posterior. Nonetheless, SNPE remains an actively researched topic, so it would be reasonable to expect further improvements in the near future.

\chapter{Sequential Neural Likelihood: Fast Likelihood-free Inference with Autoregressive Flows}
\label{chapter:snl}

The previous chapter discussed the challenge of likelihood-free inference, and described \emph{Sequential Neural Posterior Estimation},  a method that estimates the unknown posterior with a neural density model. This chapter delves further into likelihood-free inference via neural density estimation, and introduces \emph{Sequential Neural Likelihood}, an alternative approach that estimates the intractable likelihood instead. We begin with the paper \emph{Sequential Neural Likelihood: Fast Likelihood-free Inference with Autoregressive Flows}, which introduces SNL and is the main contribution of this chapter (section \ref{sec:snl:paper}). We continue with a discussion of the paper, and a comparison with a related approach based on active learning (section \ref{sec:snl:discussion}). Finally,
we conclude the part on likelihood-free inference with a summary and a discussion of this and the previous chapter (section \ref{sec:snl:summary}).

\section{The paper}
\label{sec:snl:paper}

This section presents the paper \emph{Sequential Neural Likelihood: Fast Likelihood-free Inference with Autoregressive Flows}, which is the main contribution of this chapter. The paper introduces \emph{Sequential Neural Likelihood}, a new method for likelihood-free inference via neural density estimation that was designed to overcome some of the limitations of SNPE\@.

The paper was first published as a preprint on arXiv in May 2018. Afterwards, it was accepted for publication at the 22nd \emph{International Conference on Artificial Intelligence and Statistics (AISTATS)} in April 2019. There were $1{,}111$ submissions to AISTATS that year, of which $360$ were accepted for publication.

\subsubsection{Author contributions}

The paper is co-authored by me, David Sterratt and Iain Murray. As the leading author, I developed the method, implemented the code, and performed the experiments. David Sterratt implemented the simulator for the Hodgkin--Huxley model, and advised on neuroscience matters. Iain Murray supervised the project and revised the final draft. The paper was written by me, except for sections A.4 and B.4 of the appendix which were written by David Sterratt.

\includepdf[pages=-]{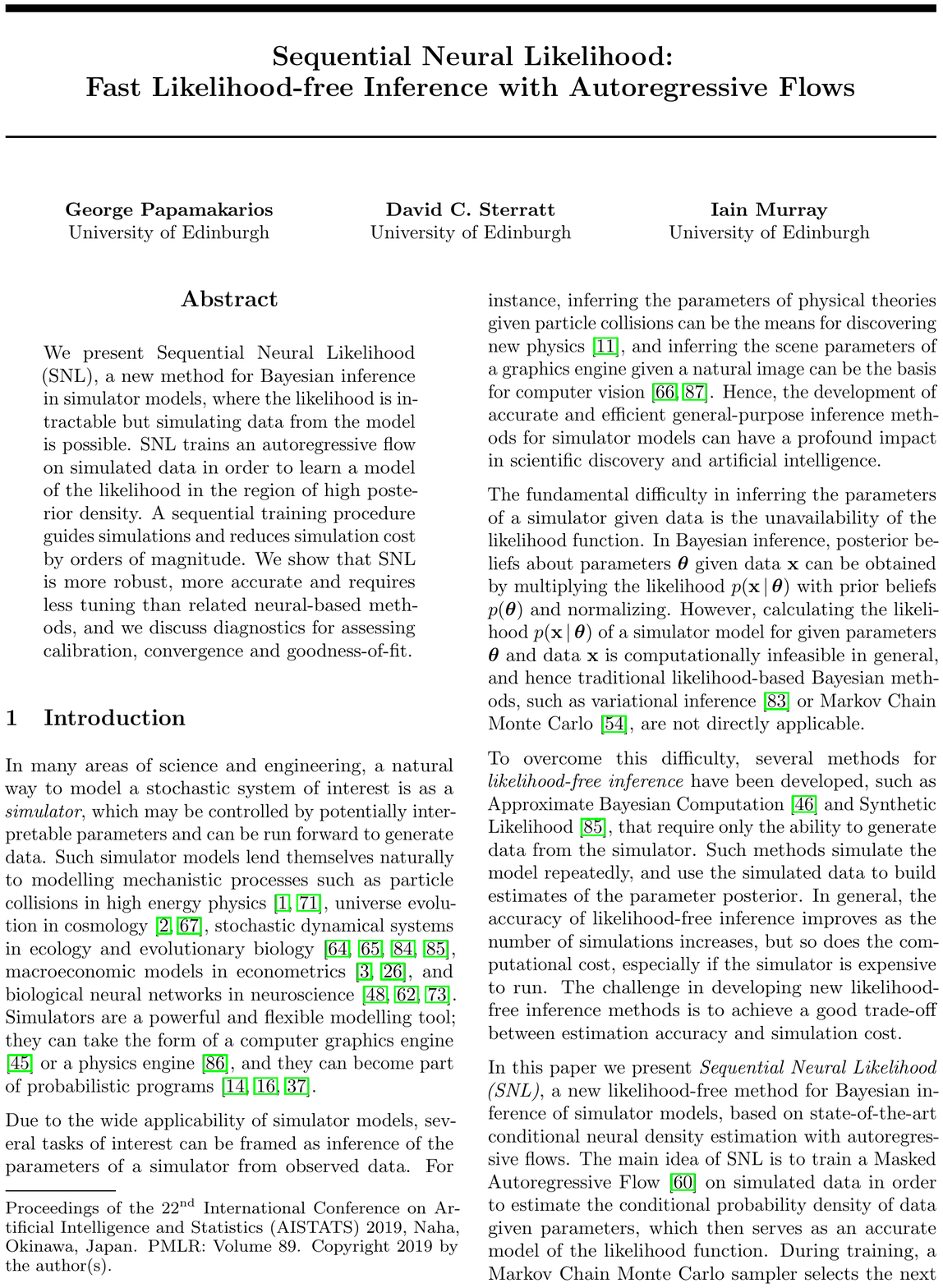}

\section{Discussion}
\label{sec:snl:discussion}

In this section, I further discuss the paper \emph{Sequential Neural Likelihood: Fast Likelihood-free Inference with Autoregressive Flows}, which was presented in the previous section. In what follows, I evaluate the contribution and impact the paper has had so far, and I compare Sequential Neural Likelihood with an alternative approach for estimating the simulator's intractable likelihood based on active learning.

\subsection{Contribution and impact}

The paper builds upon our previous work on \emph{Sequential Neural Posterior Estimation (Type A)} and \emph{Masked Autoregressive Flow}. Its main contribution is to identify the limitations of SNPE-A and SNPE-B, and to introduce a new algorithm, \emph{Sequential Neural Likelihood}, that overcomes these limitations. Compared to SNPE, SNL has the following advantages:
\begin{enumerate}[label=(\roman*)]
\item SNL is more general, since it can be used with any neural density estimator and any prior distribution. In contrast, SNPE-A can only be used with mixture-density networks and exponential-family priors.
\item SNL makes more efficient use of simulations, as it reuses simulated data from previous rounds to train the neural density estimator. In contrast, SNPE only uses the simulations from the last round for training. 
\item SNL is more robust, as it doesn't need to correct for proposing parameter samples from a distribution other than the prior. As a result, it doesn't suffer from early termination due to negative-definite covariances (such as SNPE-A) or from high variance due to importance weights (such as SNPE-B)\@.
\end{enumerate}

A disadvantage of SNL is that it requires an additional inference step to generate posterior samples. The paper uses MCMC in the form of axis-aligned slice sampling; however, in principle any likelihood-based inference method that can generate samples (such as variational inference) can be used instead. In contrast, SNPE trains a model of the posterior directly, which can generate independent samples without requiring an additional inference step.

The paper suggests using SNL with Masked Autoregressive Flow. MAF is a flexible model for general-purpose density estimation, which makes it a good default choice. However, I should emphasize that SNL is not tied to a particular model, and it can be used with any neural density estimator. In fact, even though this thesis has focused on continuous parameter and data spaces, SNL can in principle be used even when the parameters and/or the data are discrete.

It is too early to evaluate the impact of the paper, as it was published shortly before this thesis was written. As of \thedate, the paper has received $7$ citations according to Google Scholar. Along with similar methods, SNL has so far been used for likelihood-free inference in cosmology \citep{Alsing:2019:lfi_cosmology}. It remains to be seen whether SNL will prove useful in other applications, and whether it will inspire further research on the topic.

\subsection{Comparison with active-learning methods}

In each round, SNL selects parameters to simulate at by independently proposing from the posterior estimate obtained in the round before. The same strategy is used in SNPE, and a similar strategy is used in SMC-ABC\@. As we have seen, proposing parameters from the previous posterior estimate can reduce the number of simulations by orders of magnitude compared to proposing parameters from the prior, and can improve posterior-estimation accuracy significantly. Nonetheless, this strategy was motivated by intuition and empirical validation, rather than a theoretically justified argument. Therefore, it is reasonable to ask whether we can find a more principled strategy that can reduce the number of simulations even further.

In contrast with estimating the posterior, estimating the likelihood can easily  accommodate alternative parameter-acquisition strategies.
As we saw in the previous chapter, when we estimate the posterior with a neural density model it is necessary to correct for the parameter-acquisition strategy, otherwise the posterior estimate will be biased. On the other hand, as we discussed in section 3 of the paper, the parameter-acquisition strategy does not bias the estimation of the likelihood by a neural density model. Hence, estimating the likelihood instead of the posterior offers increased flexibility in selecting parameters, which enables more sophisticated strategies to be considered.

In principle, we can formulate the problem of selecting parameters as a task of decision-making under uncertainty. Ideally, we would like to select the next parameters such that the resulting simulation gives us the most information about what the posterior should be; such a scheme would select the next simulation optimally in an information-theoretic sense. In the next few paragraphs, I describe a precise theoretical formulation of the above idea.

Suppose that the conditional density of data given parameters is modelled by a Bayesian neural density estimator $q_{\bm{\phi}}\br{\vect{x}\g\bm{\theta}}$. Given a dataset of simulation results $\mathcal{D} = \set{\pair{\bm{\theta}_1}{\vect{x}_1}, \ldots, \pair{\bm{\theta}_N}{\vect{x}_N}}$, our beliefs about $\bm{\phi}$ are encoded by a distribution $\prob{\bm{\phi}\g\mathcal{D}}$. Under this Bayesian model, the predictive distribution of $\vect{x}$ given $\bm{\theta}$ is:
\begin{equation}
\prob{\vect{x}\g\bm{\theta}, \mathcal{D}} = \avgx{q_{\bm{\phi}}\br{\vect{x}\g\bm{\theta}}}{\prob{\bm{\phi}\g\mathcal{D}}},
\end{equation}
whereas the predictive distribution of $\bm{\theta}$ given $\vect{x}$ is:
\begin{equation}
\prob{\bm{\theta}\g \vect{x}, \mathcal{D}} =
\avgx{\frac{q_{\bm{\phi}}\br{\vect{x}\g\bm{\theta}}\,\prob{\bm{\theta}}}
{\integral{q_{\bm{\phi}}\br{\vect{x}\g\bm{\theta}'}\,\prob{\bm{\theta}'}}{\bm{\theta}'}}}{\prob{\bm{\phi}\g\mathcal{D}}}.
\end{equation}
Given observed data $\vect{x}_o$, the distribution $\prob{\bm{\theta}\g \vect{x}_o, \mathcal{D}}$ is the \emph{predictive posterior} under the Bayesian model. The predictive posterior encodes two distinct kinds of uncertainty:
\begin{enumerate}[label=(\roman*)]
\item Our irreducible uncertainty about what parameters might have generated $\vect{x}_o$, due to the fact that the simulator is stochastic.
\item Our uncertainty about what the posterior should be, due to not having seen enough simulation results. This type of uncertainty can be reduced by simulating more data.
\end{enumerate}
Our goal is to select the next parameters to simulate at, such that the second kind of uncertainty is maximally reduced. 

Now, suppose we obtain a new simulation result $\pair{\bm{\theta}'}{\vect{x}'}$ and hence an updated dataset $\mathcal{D}'=\mathcal{D}\cup\set{\pair{\bm{\theta}'}{\vect{x}'}}$. In that case, our beliefs about $\bm{\phi}$ will be updated by Bayes' rule:
\begin{equation}
\prob{\bm{\phi}\g\mathcal{D}'}\propto q_{\bm{\phi}}\br{\vect{x}'\g\bm{\theta}'}\,\prob{\bm{\phi}\g\mathcal{D}},
\end{equation}
and the predictive posterior will be updated to use $\prob{\bm{\phi}\g\mathcal{D}'}$ instead of $\prob{\bm{\phi}\g\mathcal{D}}$.
We want to select parameters $\bm{\theta}'$ to simulate at such that, on average, the \emph{updated} predictive posterior $\prob{\bm{\theta}\g \vect{x}_o, \mathcal{D}'}$ becomes as certain as possible. We can quantify how uncertain the updated predictive posterior is by its entropy:
\begin{equation}
H\br{\bm{\theta}\g\vect{x}_o, \mathcal{D}'} = -\avgx{\log \prob{\bm{\theta}\g \vect{x}_o, \mathcal{D}'}}{\prob{\bm{\theta}\g \vect{x}_o, \mathcal{D}'}}.
\end{equation}
The larger the entropy, the more uncertain the predictive posterior is. Hence, under the above framework, the optimal parameters $\bm{\theta}^*$ to simulate at will be those that minimize the expected predictive-posterior entropy:
\begin{equation}
\bm{\theta}^* = \argmin_{\bm{\theta}'}\,\avgx{H\br{\bm{\theta}\g\vect{x}_o, \mathcal{D}'}}{\prob{\vect{x}'\g\bm{\theta}', \mathcal{D}}}.
\end{equation}
The above expectation is taken with respect to our predictive distribution of $\vect{x}'$ given $\bm{\theta}'$, and not with respect to the actual simulator. Intuitively, this can be though of as running the simulator in our imagination rather than in reality, hence we must take into account our uncertainty over the outcome $\vect{x}'$ due to not knowing the simulator's likelihood exactly.

We can justify the above strategy in terms of the mutual information between $\bm{\theta}$ (the parameters that might have generated $\vect{x}_o$) and $\vect{x}'$ (the data generated by simulating at $\bm{\theta}'$). This mutual information can be written as:
\begin{equation}
\mathit{MI}\br{\bm{\theta}, \vect{x}'\g\bm{\theta}', \vect{x}_o, \mathcal{D}}
= H\br{\bm{\theta}\g\vect{x}_o, \mathcal{D}} - \avgx{H\br{\bm{\theta}\g\vect{x}_o, \mathcal{D}'}}{\prob{\vect{x}'\g\bm{\theta}', \mathcal{D}}}.
\end{equation}
As we can see, the mutual information between $\bm{\theta}$ and $\vect{x}'$ is equal to the expected reduction in the predictive-posterior entropy due to simulating at $\bm{\theta}'$. Hence, selecting the next parameters by minimizing the updated expected predictive-posterior entropy is equivalent to selecting the next parameters such that the simulated data  gives us on average the most information about what the parameters that produced the observed data might be.
The above strategy is a special case of the framework of \citet{Jarvenpaa:2018:efficient_acquisition}, if we take the loss function that appears in their framework to be the entropy of the updated predictive posterior.

Although optimal in an information-theoretic sense, the above strategy is intractable, as it involves a number of high-dimensional integrals and the global optimization of an objective. In order to implement the above method, we would need to:
\begin{enumerate}[label=(\roman*)]
\item train a Bayesian neural density model,
\item compute the predictive distribution of $\vect{x}$ given $\bm{\theta}$ under the model,
\item compute the predictive posterior and its entropy, and
\item globally minimize the expected predictive-posterior entropy with respect to the next parameters to simulate at.
\end{enumerate}
With current methods, each of the above steps can only be partially or approximately solved. It wouldn't be inaccurate to say that the above strategy reduces the problem of likelihood-free inference to an even harder problem.

Despite being intractable, the above strategy is useful as a guiding principle for the development of other parameter-acquisition strategies, which may approximate the optimal strategy to a certain extent. One such heuristic is the \emph{MaxVar rule} \citep{Jarvenpaa:2018:efficient_acquisition, Lueckmann:2018:maxvar}; given a Bayesian neural density estimator $q_{\bm{\phi}}\br{\vect{x}\g\bm{\theta}}$, the MaxVar rule uses the variance of the \emph{unnormalized} posterior density at parameters $\bm{\theta}'$ (which is due to the uncertainty about $\bm{\phi}$) as a proxy for the information we would gain by simulating at $\bm{\theta}'$. According to the MaxVar rule, we select the next parameter to simulate at by maximizing the above variance, that is:
\begin{equation}
\bm{\theta}^* = \argmax_{\bm{\theta}'}\,\varx{q_{\bm{\phi}}\br{\vect{x}_o\g\bm{\theta}'}\,\prob{\bm{\theta}'}}{\prob{\bm{\phi}\g\mathcal{D}}}.
\end{equation}
Assuming we can (approximately) generate samples $\set{\bm{\phi}_1, \ldots, \bm{\phi}_M}$ from $\prob{\bm{\phi}\g\mathcal{D}}$ (using e.g.~MCMC or variational inference), we can approximate the above variance using the empirical variance on the samples as follows:
\begin{equation}
\varx{q_{\bm{\phi}}\br{\vect{x}_o\g\bm{\theta}'}\,\prob{\bm{\theta}'}}{\prob{\bm{\phi}\g\mathcal{D}}}\approx
\frac{1}{M}\sum_m \br{q_{\bm{\phi}_m}\br{\vect{x}_o\g\bm{\theta}'}\,\prob{\bm{\theta}'}}^2 - \br{\frac{1}{M}\sum_m q_{\bm{\phi}_m}\br{\vect{x}_o\g\bm{\theta}'}\,\prob{\bm{\theta}'}}^2.
\end{equation}
The above empirical variance is differentiable with respect to $\bm{\theta}'$, so it can be maximized with gradient-based methods and automatic differentiation.

Following the publication of SNL, Conor Durkan, Iain Murray and I co-authored a paper \citep{Durkan:2018:snm} in which we compared the parameter-acquisition strategy of SNL and SNPE with the MaxVar rule described above. The paper, titled \emph{Sequential Neural Methods for Likelihood-free Inference} was presented in December 2018 at the \emph{Bayesian Deep Learning Workshop}, held at the conference \emph{Advances in Neural Information Processing Systems (NeurIPS)}\@. Conor Durkan was the leading author: he performed the experiments and wrote the paper. Iain Murray and I served as advisors.

The paper compares SNPE-B, SNL and MaxVar. In the experiments, all three algorithms use the same Bayesian neural density estimator, namely a mixture-density network trained with variational dropout, which is the same as that used in the previous chapter. Figure \ref{fig:snl:active_learning_comparison}, reproduced from the original paper with Conor Durkan's permission, shows the negative log posterior density of the true parameters versus the number of simulations. As we can see, despite being to a certain extent theoretically motivated, MaxVar doesn't improve upon SNL\@. Both SNL and MaxVar improve upon SNPE-B, which agrees with the rest of our experiments on SNL that were presented in this chapter.

\begin{figure}
\centering
\includegraphics[width=\textwidth]{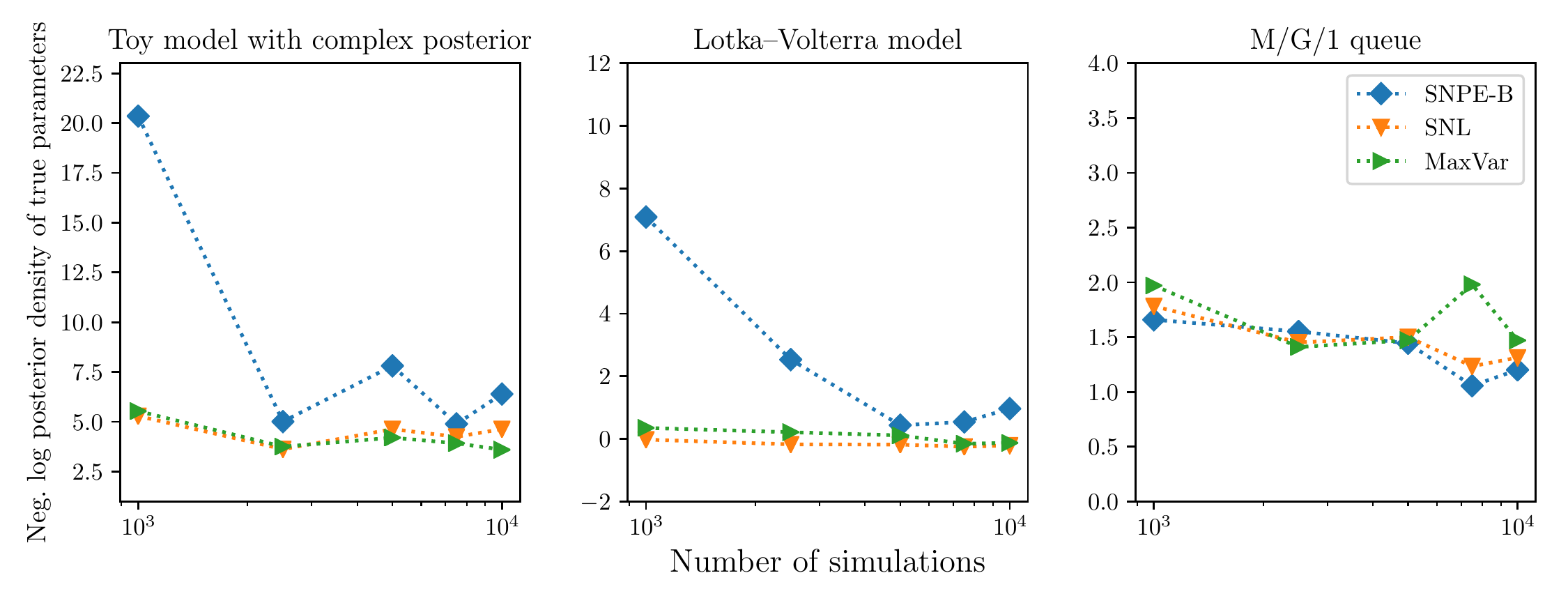}
\caption{Comparison between SNPE-B, SNL and MaxVar for three different simulator models. The plot is taken from \citep{Durkan:2018:snm}.}
\label{fig:snl:active_learning_comparison}
\end{figure}

In addition to the tradeoff between accuracy and simulation cost, it is worth looking at the wall-clock time of each algorithm. Table \ref{table:snl:wall_clock_time} shows the wall-clock time (in hours) of SNL and MaxVar for each simulator model, using a total of $10^{4}$ simulations. As we can see, SNL is about $10$ times faster than MaxVar. The reason for MaxVar being significantly slower than SNL is that MaxVar needs to solve an optimization problem for every parameter it selects. In contrast, SNL proposes parameters via MCMC, which is more efficient.

\begin{table}[h]
\centering
\caption{Wall-clock time (in hours) per experiment with $10^{4}$ simulations. The table is reproduced from \citep{Durkan:2018:snm}.}
\begin{tabular}{@{}lrrr@{}}
\toprule
& Toy model  & Lotka--Volterra & M/G/1 queue\\
\midrule
SNL & $7.48$ & $8.27$ & $7.83$ \\
MaxVar & $91.73$ & $89.39$ & $70.16$ \\
\bottomrule
\end{tabular}
\label{table:snl:wall_clock_time}
\end{table}
\vspace{0.5em}

The above comparison shows that SNL takes about the same number of simulations to achieve a given accuracy as a more sophisticated method that is based on active learning. Yet, SNL is an order of magnitude faster in terms of wall-clock time than the active-learning alternative. This comparison doesn't mean that active-learning methods aren't worthwhile in general; the more expensive a simulator is to run, the more important the parameter-acquisition strategy becomes. However, I would argue that this comparison highlights the importance of evaluating the cost of the parameter-acquisition strategy against the extent to which it makes a better use of simulations. After all, time spent on optimizing for the next parameters to simulate at may also be spent on running more simulations or training a better model.

\section{Summary and conclusions}
\label{sec:snl:summary}

Simulators are useful modelling tools, because they are interpretable, flexible, and a good match to our understanding of mechanistic processes in the physical world. However, the likelihood of a simulator model is often intractable, which makes traditional Bayesian inference over the model's parameters challenging. In this and the previous chapter, I discussed likelihood-free inference, i.e.~Bayesian inference based on simulations, and I introduced efficient likelihood-free inference methods based on neural density estimation.

Approximate Bayesian computation is a family of methods that have been traditionally used for likelihood-free inference. ABC methods are based on repeatedly simulating the model, and rejecting simulations that don't match the observed data. ABC methods don't target the exact posterior but an approximate posterior, obtained by an alternative observation that the simulated data is at most a distance $\epsilon$ away from the observed data. The parameter $\epsilon$ trades off efficiency for accuracy; ABC becomes exact as $\epsilon$ approaches zero, but at the same time the simulation cost increases dramatically.

In the previous chapter, I presented a method for likelihood-free inference that was later given the name \emph{Sequential Neural Posterior Estimation (Type A)}\@. SNPE-A trains a Bayesian mixture-density network on simulated data in order to estimate the exact posterior. The algorithm progresses over multiple rounds; in each round, parameters to simulate at are proposed by the posterior estimate obtained in the round before. Unlike ABC, SNPE-A doesn't reject simulations, and doesn't involve setting a distance $\epsilon$. Experiments showed that SNPE-A achieves orders of magnitude improvement over ABC methods such as rejection ABC, MCMC-ABC and SMC-ABC\@.

SNPE-A has two drawbacks: it is tied to mixture-density networks, and it may terminate early due to the posterior estimate becoming improper. A variant of SNPE-A, which I refer to as SNPE-B, was proposed by \citet{Lueckmann:2017:snpe} to overcome these limitations. However, as we demonstrated in this chapter, SNPE-B is typically less accurate than SNPE-A (at least when SNPE-A doesn't terminate early) due to increased variance in the posterior estimate.

To overcome the limitations of both SNPE-A and SNPE-B, this chapter presents \emph{Sequential Neural Likelihood}, a method that is similar to SNPE but targets the likelihood instead of the posterior. In our experiments, SNL was more accurate and more robust than both SNPE-A and SNPE-B\@. In addition, SNL enables the use of diagnostics such as a two-sample goodness-of-fit test between the simulator and the neural density estimator. A comparison between SNL and an active-learning method for selecting parameters, performed by \citet{Durkan:2018:snm}, showed that SNL is at least as simulation-efficient as the active-learning method, but significantly faster in terms of wall-clock time.

As I discussed in the previous chapter, when using ABC it is often necessary to transform data into lower-dimensional summary statistics, in order to maintain the acceptance rate at a high enough level. One limitation of this thesis is that it didn't explore how the methods that are based on neural density estimation scale with the dimensionality of the parameters or the data. Nonetheless, research in deep learning has demonstrated that neural networks with suitable architectures scale well to high-dimensional inputs or outputs. It is reasonable to expect that, with suitable neural architectures, SNPE or SNL may scale better than ABC to high-dimensional parameters or data.

Some preliminary work in this direction already exists. For example, \citet{Lueckmann:2017:snpe} used a recurrent neural network with SNPE in order to estimate the posterior directly from timeseries data. Similarly, \citet{Chan:2018:exchangeable} used an exchangeable neural network to estimate the posterior when the data is a set of exchangeable (e.g.~independent) datapoints. As I argued in the chapter on Masked Autoregressive Flow, building invariances in our neural architectures that reflect reasonable assumptions about the data is a key element in scaling neural networks to high dimensions. As far as I'm aware, a careful evaluation of how neural likelihood-free methods such as SNPE and SNL scale to high dimensions hasn't been performed yet. However, given the progress of deep-learning research, I would argue that engineering neural architectures to use with SNPE and SNL, or alternative neural methods, is a promising direction for future work.

Having explored in the last two chapters neural methods both for estimating the posterior and for estimating the likelihood, it is natural to ask which approach is better. Despite the success of SNL over SNPE, I would argue that there is no definite answer to this question, as each approach has its own strengths and weaknesses. Ultimately the right approach depends on what the user is interested in. Hence, I would argue that it is worthwhile to continue exploring both approaches, as well as improving our current methods and developing new ones.

\bibliography{thesis}

\begin{thebibliography}{154}
\providecommand{\natexlab}[1]{#1}
\providecommand{\url}[1]{\texttt{#1}}
\expandafter\ifx\csname urlstyle\endcsname\relax
  \providecommand{\doi}[1]{doi: #1}\else
  \providecommand{\doi}{doi: \begingroup \urlstyle{rm}\Url}\fi

\bibitem[Abadi et~al.(2015)]{Tensorflow:2015:whitepaper}
M.~Abadi et~al.
\newblock {T}ensor{F}low: {L}arge-scale machine learning on heterogeneous
  systems, 2015.
\newblock URL \url{https://www.tensorflow.org/}.

\bibitem[Agostinelli et~al.(2003)]{Agostinelli:2003:geant}
S.~Agostinelli et~al.
\newblock Geant4---a simulation toolkit.
\newblock \emph{Nuclear Instruments and Methods in Physics Research Section A:
  Accelerators, Spectrometers, Detectors and Associated Equipment},
  506\penalty0 (3):\penalty0 250--303, 2003.

\bibitem[Al-Rfou et~al.(2016)]{Al-Rfou:2016:theano}
R.~Al-Rfou et~al.
\newblock Theano: A {P}ython framework for fast computation of mathematical
  expressions.
\newblock \emph{arXiv:1605.02688}, 2016.

\bibitem[Alemi et~al.(2018)Alemi, Poole, Fischer, Dillon, Saurous, and
  Murphy]{Alemi:2018:broken_elbo}
A.~A. Alemi, B.~Poole, I.~Fischer, J.~V. Dillon, R.~A. Saurous, and K.~Murphy.
\newblock Fixing a broken {ELBO}.
\newblock \emph{Proceedings of the 35th International Conference on Machine
  Learning}, 2018.

\bibitem[Alsing et~al.(2018{\natexlab{a}})Alsing, Wandelt, and
  Feeney]{Alsing:2018:lfi_cosmology}
J.~Alsing, B.~D. Wandelt, and S.~M. Feeney.
\newblock Massive optimal data compression and density estimation for scalable,
  likelihood-free inference in cosmology.
\newblock \emph{Monthly Notices of the Royal Astronomical Society},
  477\penalty0 (3):\penalty0 2874--2885, 2018{\natexlab{a}}.

\bibitem[Alsing et~al.(2018{\natexlab{b}})Alsing, Wandelt, and
  Feeney]{Alsing:2018:optimal_abc_proposal}
J.~Alsing, B.~D. Wandelt, and S.~M. Feeney.
\newblock Optimal proposals for approximate {B}ayesian computation.
\newblock \emph{arXiv:1808.06040}, 2018{\natexlab{b}}.

\bibitem[Alsing et~al.(2019)Alsing, Charnock, Feeney, and
  Wandelt]{Alsing:2019:lfi_cosmology}
J.~Alsing, T.~Charnock, S.~M. Feeney, and B.~D. Wandelt.
\newblock Fast likelihood-free cosmology with neural density estimators and
  active learning.
\newblock \emph{arXiv:1903.00007}, 2019.

\bibitem[Andrieu and Roberts(2009)]{Andrieu:2009:pseudomarginal}
C.~Andrieu and G.~O. Roberts.
\newblock The pseudo-marginal approach for efficient {M}onte {C}arlo
  computations.
\newblock \emph{The Annals of Statistics}, 37\penalty0 (2):\penalty0 697--725,
  2009.

\bibitem[Arjovsky et~al.(2017)Arjovsky, Chintala, and
  Bottou]{Arjovsky:2017:wgan}
M.~Arjovsky, S.~Chintala, and L.~Bottou.
\newblock {W}asserstein generative adversarial networks.
\newblock \emph{Proceedings of the 34th International Conference on Machine
  Learning}, 2017.

\bibitem[Basri and Jacobs(2003)]{Basri:2003:lambertian}
R.~Basri and D.~W. Jacobs.
\newblock {L}ambertian reflectance and linear subspaces.
\newblock \emph{IEEE Transactions on Pattern Analysis and Machine
  Intelligence}, 25\penalty0 (2):\penalty0 218--233, 2003.

\bibitem[Bauer and Mnih(2019)]{Bauer:2019:resampled}
M.~Bauer and A.~Mnih.
\newblock Resampled priors for variational autoencoders.
\newblock \emph{Proceedings of the 22nd International Conference on Artificial
  Intelligence and Statistics}, 2019.

\bibitem[Beaumont(2010)]{Beaumont:2010:abc_evo_eco}
M.~A. Beaumont.
\newblock Approximate {B}ayesian computation in evolution and ecology.
\newblock \emph{Annual Review of Ecology, Evolution, and Systematics},
  41\penalty0 (1):\penalty0 379--406, 2010.

\bibitem[Beaumont et~al.(2002)Beaumont, Zhang, and
  Balding]{Beaumont:2002:abc_pop_gen}
M.~A. Beaumont, W.~Zhang, and D.~J. Balding.
\newblock Approximate {B}ayesian computation in population genetics.
\newblock \emph{Genetics}, 162:\penalty0 2025--2035, 2002.

\bibitem[Beaumont et~al.(2009)Beaumont, Cornuet, Marin, and
  Robert]{Beaumont:2009:smc_abc}
M.~A. Beaumont, J.-M. Cornuet, J.-M. Marin, and C.~P. Robert.
\newblock Adaptive approximate {B}ayesian computation.
\newblock \emph{Biometrika}, 96\penalty0 (4):\penalty0 983--990, 2009.

\bibitem[Behrmann et~al.(2018)Behrmann, Grathwohl, Chen, Duvenaud, and
  Jacobsen]{Behrmann:2018:iresnet}
J.~Behrmann, W.~Grathwohl, R.~T.~Q. Chen, D.~K. Duvenaud, and J.-H. Jacobsen.
\newblock Invertible residual networks.
\newblock \emph{arXiv:1811.00995}, 2018.

\bibitem[Billingsley(1995)]{Billingsley:1995:measure}
P.~Billingsley.
\newblock \emph{Probability and Measure}.
\newblock Wiley, 1995.

\bibitem[Bishop(2006)]{Bishop:2006:prml}
C.~M. Bishop.
\newblock \emph{Pattern recognition and machine learning}.
\newblock Springer New York, 2006.

\bibitem[Blei et~al.(2017)Blei, Kucukelbir, and McAuliffe]{Blei:2017:vi}
D.~M. Blei, A.~Kucukelbir, and J.~D. McAuliffe.
\newblock Variational inference: A review for statisticians.
\newblock \emph{Journal of the American Statistical Association}, 112\penalty0
  (518):\penalty0 859--877, 2017.

\bibitem[Blum and Fran{\c{c}}ois(2010)]{Blum:2010:reg_abc}
M.~G.~B. Blum and O.~Fran{\c{c}}ois.
\newblock Non-linear regression models for approximate {B}ayesian computation.
\newblock \emph{Statistics and Computing}, 20\penalty0 (1):\penalty0 63--73,
  2010.

\bibitem[Blum et~al.(2013)Blum, Nunes, Prangle, and
  Sisson]{Blum:2013:dim_reduction_review}
M.~G.~B. Blum, M.~A. Nunes, D.~Prangle, and S.~A. Sisson.
\newblock A comparative review of dimension reduction methods in approximate
  {B}ayesian computation.
\newblock \emph{Statistical Science}, 28\penalty0 (2):\penalty0 189--208, 2013.

\bibitem[Bonassi and West(2015)]{Bonassi:2015:smc_abc}
F.~V. Bonassi and M.~West.
\newblock Sequential {M}onte {C}arlo with adaptive weights for approximate
  {B}ayesian computation.
\newblock \emph{Bayesian Analysis}, 10\penalty0 (1):\penalty0 171--187, 2015.

\bibitem[Bornn et~al.(2017)Bornn, Pillai, Smith, and
  Woodard]{Bornn:2017:single_pseudosample}
L.~Bornn, N.~S. Pillai, A.~Smith, and D.~Woodard.
\newblock The use of a single pseudo-sample in approximate {B}ayesian
  computation.
\newblock \emph{Statistics and Computing}, 27\penalty0 (3):\penalty0 583--590,
  2017.

\bibitem[Bottou(2012)]{Bottou:2012:sgd}
L.~Bottou.
\newblock Stochastic gradient descent tricks.
\newblock \emph{Neural Networks, Tricks of the Trade, Reloaded}, 7700:\penalty0
  430--445, 2012.

\bibitem[Brehmer et~al.(2018{\natexlab{a}})Brehmer, Cranmer, Louppe, and
  Pavez]{Brehmer:2018:eft}
J.~Brehmer, K.~Cranmer, G.~Louppe, and J.~Pavez.
\newblock Constraining effective field theories with machine learning.
\newblock \emph{Physical Review Letters}, 121\penalty0 (11):\penalty0 111801,
  2018{\natexlab{a}}.

\bibitem[Brehmer et~al.(2018{\natexlab{b}})Brehmer, Cranmer, Louppe, and
  Pavez]{Brehmer:2018:eft_guide}
J.~Brehmer, K.~Cranmer, G.~Louppe, and J.~Pavez.
\newblock A guide to constraining effective field theories with machine
  learning.
\newblock \emph{Physical Review Letters}, 98\penalty0 (5):\penalty0 052004,
  2018{\natexlab{b}}.

\bibitem[Brehmer et~al.(2018{\natexlab{c}})Brehmer, Louppe, Pavez, and
  Cranmer]{Brehmer:2018:mining}
J.~Brehmer, G.~Louppe, J.~Pavez, and K.~Cranmer.
\newblock Mining gold from implicit models to improve likelihood-free
  inference.
\newblock \emph{arXiv:1805.12244}, 2018{\natexlab{c}}.

\bibitem[Brock et~al.(2019)Brock, Donahue, and Simonyan]{Brock:2018:biggan}
A.~Brock, J.~Donahue, and K.~Simonyan.
\newblock Large scale {GAN} training for high fidelity natural image synthesis.
\newblock \emph{Proceedings of the 7th International Conference on Learning
  Representations}, 2019.

\bibitem[Brubaker et~al.(2012)Brubaker, Salzmann, and
  Urtasun]{Brubaker:2012:constrained_mcmc}
M.~Brubaker, M.~Salzmann, and R.~Urtasun.
\newblock A family of {MCMC} methods on implicitly defined manifolds.
\newblock \emph{Proceedings of the 15th International Conference on Artificial
  Intelligence and Statistics}, 2012.

\bibitem[Buesing et~al.(2018)Buesing, Weber, Racani{\`e}re, Eslami, Rezende,
  Reichert, Viola, Besse, Gregor, Hassabis, and
  Wierstra]{Buesing:2018:fast_gen_models}
L.~Buesing, T.~Weber, S.~Racani{\`e}re, S.~M.~A. Eslami, D.~J. Rezende, D.~P.
  Reichert, F.~Viola, F.~Besse, K.~Gregor, D.~Hassabis, and D.~Wierstra.
\newblock Learning and querying fast generative models for reinforcement
  learning.
\newblock \emph{arXiv:1802.03006}, 2018.

\bibitem[Capp\'{e} and Moulines(2008)]{Cappe:2008:online_em}
O.~Capp\'{e} and E.~Moulines.
\newblock Online {EM} algorithm for latent data models.
\newblock \emph{Journal of the Royal Statistical Society, Series B},
  71\penalty0 (3):\penalty0 593--613, 2008.

\bibitem[Caticha(2004)]{Caticha:2004:entropy}
A.~Caticha.
\newblock Relative entropy and inductive inference.
\newblock \emph{AIP Conference Proceedings}, 707\penalty0 (1):\penalty0 75--96,
  2004.

\bibitem[Chan et~al.(2018)Chan, Perrone, Spence, Jenkins, Mathieson, and
  Song]{Chan:2018:exchangeable}
J.~Chan, V.~Perrone, J.~Spence, P.~Jenkins, S.~Mathieson, and Y.~Song.
\newblock A likelihood-free inference framework for population genetic data
  using exchangeable neural networks.
\newblock \emph{Advances in Neural Information Processing Systems 31}, 2018.

\bibitem[Charnock et~al.(2018)Charnock, Lavaux, and
  Wandelt]{Charnock:2018:IMNN}
T.~Charnock, G.~Lavaux, and B.~D. Wandelt.
\newblock Automatic physical inference with information maximizing neural
  networks.
\newblock \emph{Physical Review D}, 97\penalty0 (8), 2018.

\bibitem[Chen et~al.(2018)Chen, Rubanova, Bettencourt, and
  Duvenaud]{Chen:2018:neural_odes}
R.~T.~Q. Chen, Y.~Rubanova, J.~Bettencourt, and D.~K. Duvenaud.
\newblock Neural ordinary differential equations.
\newblock \emph{Advances in Neural Information Processing Systems 31}, 2018.

\bibitem[Chen and Gutmann(2019)]{Chen:2019:copula_abc}
Y.~Chen and M.~U. Gutmann.
\newblock Adaptive {G}aussian copula {ABC}.
\newblock \emph{Proceedings of the 22nd International Conference on Artificial
  Intelligence and Statistics}, 2019.

\bibitem[Choi et~al.(2019)Choi, Jang, and Alemi]{Choi:2019:waic}
H.~Choi, E.~Jang, and A.~A. Alemi.
\newblock {WAIC}, but why? {G}enerative ensembles for robust anomaly detection.
\newblock \emph{arXiv:1810.01392}, 2019.

\bibitem[Chwialkowski et~al.(2016)Chwialkowski, Strathmann, and
  Gretton]{Chwialkowski:2016:gof}
K.~Chwialkowski, H.~Strathmann, and A.~Gretton.
\newblock A kernel test of goodness of fit.
\newblock \emph{Proceedings of the 33rd International Conference on Machine
  Learning}, 2016.

\bibitem[Cusumano-Towner et~al.(2017)Cusumano-Towner, Radul, Wingate, and
  Mansinghka]{Cusumano:2017:agent_goals}
M.~F. Cusumano-Towner, A.~Radul, D.~Wingate, and V.~K. Mansinghka.
\newblock Probabilistic programs for inferring the goals of autonomous agents.
\newblock \emph{arXiv:1704.04977}, 2017.

\bibitem[Danihelka et~al.(2017)Danihelka, Lakshminarayanan, Uria, Wierstra, and
  Dayan]{Danihelka:2017:rnvp_gan}
I.~Danihelka, B.~Lakshminarayanan, B.~Uria, D.~Wierstra, and P.~Dayan.
\newblock Comparison of maximum likelihood and {GAN}-based training of {R}eal
  {NVP}s.
\newblock \emph{arXiv:1705.05263}, 2017.

\bibitem[De~Cao et~al.(2019)De~Cao, Titov, and Aziz]{DeCao:2019:bnaf}
N.~De~Cao, I.~Titov, and W.~Aziz.
\newblock Block neural autoregressive flow.
\newblock \emph{arXiv:1904.04676}, 2019.

\bibitem[Dempster et~al.(1977)Dempster, Laird, and Rubin]{Dempster:1977:em}
A.~P. Dempster, N.~M. Laird, and D.~B. Rubin.
\newblock Maximum likelihood from incomplete data via the {EM} algorithm.
\newblock \emph{Journal of the Royal Statistical Society, Series B},
  39\penalty0 (1):\penalty0 1--38, 1977.

\bibitem[Devlin et~al.(2018)Devlin, Chang, Lee, and
  Toutanova]{Devlin:2018:bert}
J.~Devlin, M.-W. Chang, K.~Lee, and K.~Toutanova.
\newblock {BERT}: Pre-training of deep bidirectional transformers for language
  understanding.
\newblock \emph{arXiv:1810.04805}, 2018.

\bibitem[Dillon et~al.(2017)Dillon, Langmore, Tran, Brevdo, Vasudevan, Moore,
  Patton, Alemi, Hoffman, and Saurous]{Dillon:2017:tfd}
J.~V. Dillon, I.~Langmore, D.~Tran, E.~Brevdo, S.~Vasudevan, D.~Moore,
  B.~Patton, A.~A. Alemi, M.~D. Hoffman, and R.~A. Saurous.
\newblock Tensor{F}low distributions.
\newblock \emph{arXiv:1711.10604}, 2017.

\bibitem[Dinh et~al.(2017)Dinh, Sohl-Dickstein, and Bengio]{Dinh:2017:rnvp}
L.~Dinh, J.~Sohl-Dickstein, and S.~Bengio.
\newblock Density estimation using {R}eal {NVP}.
\newblock \emph{Proceedings of the 5th International Conference on Learning
  Representations}, 2017.

\bibitem[Duchi et~al.(2011)Duchi, Hazan, and Singer]{Duchi:2011:adagrad}
J.~Duchi, E.~Hazan, and Y.~Singer.
\newblock Adaptive subgradient methods for online learning and stochastic
  optimization.
\newblock \emph{Journal of Machine Learning Research}, 12:\penalty0 2121--2159,
  2011.

\bibitem[Durkan et~al.(2018)Durkan, Papamakarios, and Murray]{Durkan:2018:snm}
C.~Durkan, G.~Papamakarios, and I.~Murray.
\newblock Sequential neural methods for likelihood-free inference.
\newblock \emph{Bayesian Deep Learning Workshop at Neural Information
  Processing Systems}, 2018.

\bibitem[Dziugaite et~al.(2015)Dziugaite, Roy, and
  Ghahramani]{Dziugaite:2015:mmdgan}
G.~K. Dziugaite, D.~M. Roy, and Z.~Ghahramani.
\newblock Training generative neural networks via maximum mean discrepancy
  optimization.
\newblock \emph{Proceedings of the 31st Conference on Uncertainty in Artificial
  Intelligence}, 2015.

\bibitem[Edwards and Storkey(2017)]{Edwards:2017:statistician}
H.~Edwards and A.~J. Storkey.
\newblock Towards a neural statistician.
\newblock \emph{Proceedings of the 5th International Conference on Learning
  Representations}, 2017.

\bibitem[Epanechnikov(1969)]{Epanechnikov:1969:kernel}
V.~A. Epanechnikov.
\newblock Non-parametric estimation of a multivariate probability density.
\newblock \emph{Theory of Probability and its Applications}, 14\penalty0
  (1):\penalty0 153--158, 1969.

\bibitem[Eslami et~al.(2016)Eslami, Heess, Weber, Tassa, Szepesvari,
  Kavukcuoglu, and Hinton]{Eslami:2016:air}
S.~M.~A. Eslami, N.~Heess, T.~Weber, Y.~Tassa, D.~Szepesvari, K.~Kavukcuoglu,
  and G.~E. Hinton.
\newblock Attend, infer, repeat: Fast scene understanding with generative
  models.
\newblock \emph{Advances in Neural Information Processing Systems 29}, 2016.

\bibitem[Fan et~al.(2013)Fan, Nott, and Sisson]{Fan:2013}
Y.~Fan, D.~J. Nott, and S.~A. Sisson.
\newblock Approximate {B}ayesian computation via regression density estimation.
\newblock \emph{Stat}, 2\penalty0 (1):\penalty0 34--48, 2013.

\bibitem[Fearnhead and Prangle(2012)]{Fearnhead:2012:semi_automatic_abc}
P.~Fearnhead and D.~Prangle.
\newblock Constructing summary statistics for approximate {B}ayesian
  computation: {S}emi-automatic approximate {B}ayesian computation.
\newblock \emph{Journal of the Royal Statistical Society: Series B},
  74\penalty0 (3):\penalty0 419--474, 2012.

\bibitem[Germain et~al.(2015)Germain, Gregor, Murray, and
  Larochelle]{Germain:2015:made}
M.~Germain, K.~Gregor, I.~Murray, and H.~Larochelle.
\newblock {MADE}: Masked autoencoder for distribution estimation.
\newblock \emph{Proceedings of the 32nd International Conference on Machine
  Learning}, 2015.

\bibitem[Goodfellow et~al.(2014)Goodfellow, Pouget-Abadie, Mirza, Xu,
  Warde-Farley, Ozair, Courville, and Bengio]{Goodfellow:2014:gan}
I.~Goodfellow, J.~Pouget-Abadie, M.~Mirza, B.~Xu, D.~Warde-Farley, S.~Ozair,
  A.~Courville, and Y.~Bengio.
\newblock Generative adversarial nets.
\newblock \emph{Advances in Neural Information Processing Systems 27}, 2014.

\bibitem[Goodfellow et~al.(2016)Goodfellow, Bengio, and
  Courville]{Goodfellow:2016:deeplearningbook}
I.~Goodfellow, Y.~Bengio, and A.~Courville.
\newblock \emph{Deep Learning}.
\newblock MIT Press, 2016.

\bibitem[Graham and Storkey(2017)]{Graham:2017:differentiable_simulators}
M.~M. Graham and A.~J. Storkey.
\newblock Asymptotically exact inference in differentiable generative models.
\newblock \emph{Electronic Journal of Statistics}, 11\penalty0 (2):\penalty0
  5105--5164, 2017.

\bibitem[Grathwohl et~al.(2018)Grathwohl, Chen, Betterncourt, Sutskever, and
  Duvenaud]{Grathwohl:2018:ffjord}
W.~Grathwohl, R.~T.~Q. Chen, J.~Betterncourt, I.~Sutskever, and D.~K. Duvenaud.
\newblock {FFJORD}: {F}ree-form continuous dynamics for scalable reversible
  generative models.
\newblock \emph{Proceedings of the 7th International Conference on Learning
  Representations}, 2018.

\bibitem[Gregor et~al.(2019)Gregor, Papamakarios, Besse, Buesing, and
  Weber]{Gregor:2019:tdvae}
K.~Gregor, G.~Papamakarios, F.~Besse, L.~Buesing, and T.~Weber.
\newblock Temporal difference variational auto-encoder.
\newblock \emph{Proceedings of the 7th International Conference on Learning
  Representations}, 2019.

\bibitem[Gretton et~al.(2012)Gretton, Borgwardt, Rasch, Sch\"{o}lkopf, and
  Smola]{Gretton:2012:mmd}
A.~Gretton, K.~M. Borgwardt, M.~J. Rasch, B.~Sch\"{o}lkopf, and A.~Smola.
\newblock A kernel two-sample test.
\newblock \emph{Joural of Machine Learning Research}, 13\penalty0 (1):\penalty0
  723--773, 2012.

\bibitem[Grover et~al.(2018)Grover, Dhar, and Ermon]{Grover:2018:flowgan}
A.~Grover, M.~Dhar, and S.~Ermon.
\newblock Flow-{GAN}: Combining maximum likelihood and adversarial learning in
  generative models.
\newblock \emph{Proceedings of 32nd AAAI Conference on Artificial
  Intelligence}, 2018.

\bibitem[Gu et~al.(2015)Gu, Ghahramani, and Turner]{Gu:2015:neural_smc}
S.~Gu, Z.~Ghahramani, and R.~E. Turner.
\newblock Neural adaptive sequential {M}onte {C}arlo.
\newblock \emph{Advances in Neural Information Processing Systems 28}, 2015.

\bibitem[Gutmann and Hyv\"{a}rinen(2012)]{Gutmann:2012:nce}
M.~U. Gutmann and A.~Hyv\"{a}rinen.
\newblock Noise-contrastive estimation of unnormalized statistical models, with
  applications to natural image statistics.
\newblock \emph{Journal of Machine Learning Research}, 13:\penalty0 307--361,
  2012.

\bibitem[Hastie et~al.(2001)Hastie, Tibshirani, and
  Friedman]{Hastie:2001:elements}
T.~Hastie, R.~Tibshirani, and J.~Friedman.
\newblock \emph{The Elements of Statistical Learning}.
\newblock Springer New York, 2001.

\bibitem[He et~al.(2016)He, Zhang, Ren, and Sun]{He:2016:resnet}
K.~He, X.~Zhang, S.~Ren, and J.~Sun.
\newblock Deep residual learning for image recognition.
\newblock \emph{IEEE Conference on Computer Vision and Pattern Recognition},
  2016.

\bibitem[Hinton(2002)]{Hinton:2002:products_experts}
G.~E. Hinton.
\newblock Training products of experts by minimizing contrastive divergence.
\newblock \emph{Neural Computation}, 14\penalty0 (8):\penalty0 1771--1800,
  2002.

\bibitem[Hinton et~al.(1997)Hinton, Dayan, and Revow]{Hinton:1997:manifolds}
G.~E. Hinton, P.~Dayan, and M.~Revow.
\newblock Modeling the manifolds of images of handwritten digits.
\newblock \emph{IEEE Transactions on Neural Networks}, 8\penalty0 (1):\penalty0
  65--74, 1997.

\bibitem[Ho et~al.(2019)Ho, Chen, Srinivas, Duan, and Abbeel]{Ho:2019:flowpp}
J.~Ho, X.~Chen, A.~Srinivas, Y.~Duan, and P.~Abbeel.
\newblock Flow++: Improving flow-based generative models with variational
  dequantization and architecture design.
\newblock \emph{arXiv:1902.00275}, 2019.

\bibitem[Hoogeboom et~al.(2019)Hoogeboom, van~den Berg, and
  Welling]{Hoogeboom:2019:emerging}
E.~Hoogeboom, R.~van~den Berg, and M.~Welling.
\newblock Emerging convolutions for generative normalizing flows.
\newblock \emph{arXiv:1901.11137}, 2019.

\bibitem[Huang et~al.(2018)Huang, Krueger, Lacoste, and
  Courville]{Huang:2018:naf}
C.-W. Huang, D.~Krueger, A.~Lacoste, and A.~Courville.
\newblock Neural autoregressive flows.
\newblock \emph{Proceedings of the 35th International Conference on Machine
  Learning}, 2018.

\bibitem[Hutchinson(1990)]{Hutchinson:1990:estimator}
M.~F. Hutchinson.
\newblock A stochastic estimator of the trace of the influence matrix for
  {L}aplacian smoothing splines.
\newblock \emph{Communications in Statistics---Simulation and Computation},
  19\penalty0 (2):\penalty0 433--450, 1990.

\bibitem[Hyv\"{a}rinen(2005)]{Hyvarinen:05:score_matching}
A.~Hyv\"{a}rinen.
\newblock Estimation of non-normalized statistical models by score matching.
\newblock \emph{Journal of Machine Learning Research}, 6:\penalty0 695--709,
  2005.

\bibitem[Hyv\"{a}rinen and Pajunen(1999)]{Hyvarinen:1999:nonlinear_ica}
A.~Hyv\"{a}rinen and P.~Pajunen.
\newblock Nonlinear independent component analysis: Existence and uniqueness
  results.
\newblock \emph{Neural Networks}, 12\penalty0 (3):\penalty0 429--439, 1999.

\bibitem[J\"arvenp\"a\"a et~al.(2018)J\"arvenp\"a\"a, Gutmann, Pleska, Vehtari,
  and Marttinen]{Jarvenpaa:2018:efficient_acquisition}
M.~J\"arvenp\"a\"a, M.~U. Gutmann, A.~Pleska, A.~Vehtari, and P.~Marttinen.
\newblock Efficient acquisition rules for model-based approximate {B}ayesian
  computation.
\newblock \emph{Bayesian Analysis}, 2018.

\bibitem[Jitkrittum et~al.(2017)Jitkrittum, Xu, Szabo, Fukumizu, and
  Gretton]{Jitkrittum:2017:gof}
W.~Jitkrittum, W.~Xu, Z.~Szabo, K.~Fukumizu, and A.~Gretton.
\newblock A linear-time kernel goodness-of-fit test.
\newblock \emph{Advances in Neural Information Processing Systems 30}, 2017.

\bibitem[Karras et~al.(2018{\natexlab{a}})Karras, Aila, Laine, and
  Lehtinen]{Karras:2018:progan}
T.~Karras, T.~Aila, S.~Laine, and J.~Lehtinen.
\newblock Progressive growing of {GAN}s for improved quality, stability, and
  variation.
\newblock \emph{Proceedings of the 6th International Conference on Learning
  Representations}, 2018{\natexlab{a}}.

\bibitem[Karras et~al.(2018{\natexlab{b}})Karras, Laine, and
  Aila]{Karras:2018:stylegan}
T.~Karras, S.~Laine, and T.~Aila.
\newblock A style-based generator architecture for generative adversarial
  networks.
\newblock \emph{arXiv:1812.04948}, 2018{\natexlab{b}}.

\bibitem[Keskar et~al.(2017)Keskar, Mudigere, Nocedal, Smelyanskiy, and
  Tang]{Keskar:2017:gen_gap}
N.~S. Keskar, D.~Mudigere, J.~Nocedal, M.~Smelyanskiy, and P.~T.~P. Tang.
\newblock On large-batch training for deep learning: Generalization gap and
  sharp minima.
\newblock \emph{Proceedings of 5th International Conference on Learning
  Representations}, 2017.

\bibitem[Kim et~al.(2018)Kim, gil Lee, Song, and Yoon]{Kim:2018:FloWaveNet}
S.~Kim, S.~gil Lee, J.~Song, and S.~Yoon.
\newblock {F}lo{W}ave{N}et: {A} generative flow for raw audio.
\newblock \emph{arXiv:1811.02155}, 2018.

\bibitem[Kingma and Ba(2015)]{Kingma:2015:adam}
D.~P. Kingma and J.~Ba.
\newblock {A}dam: {A} method for stochastic optimization.
\newblock \emph{Proceedings of the 3rd International Conference on Learning
  Representations}, 2015.

\bibitem[Kingma and Dhariwal(2018)]{Kingma:2018:glow}
D.~P. Kingma and P.~Dhariwal.
\newblock {G}low: {G}enerative flow with invertible $1\times 1$ convolutions.
\newblock \emph{Advances in Neural Information Processing Systems 31}, 2018.

\bibitem[Kingma and Welling(2014)]{Kingma:2014:vae}
D.~P. Kingma and M.~Welling.
\newblock Auto-encoding variational {B}ayes.
\newblock \emph{Proceedings of the 2nd International Conference on Learning
  Representations}, 2014.

\bibitem[Kingma et~al.(2016)Kingma, Salimans, Jozefowicz, Chen, Sutskever, and
  Welling]{Kingma:2016:iaf}
D.~P. Kingma, T.~Salimans, R.~Jozefowicz, X.~Chen, I.~Sutskever, and
  M.~Welling.
\newblock Improved variational inference with inverse autoregressive flow.
\newblock \emph{Advances in Neural Information Processing Systems 29}, 2016.

\bibitem[Krizhevsky et~al.(2012)Krizhevsky, Sutskever, and
  Hinton]{Krizhevsky:2012:alexnet}
A.~Krizhevsky, I.~Sutskever, and G.~E. Hinton.
\newblock {I}mage{N}et classification with deep convolutional neural networks.
\newblock \emph{Advances in Neural Information Processing Systems 25}, 2012.

\bibitem[Kucukelbir et~al.(2015)Kucukelbir, Ranganath, Gelman, and
  Blei]{Kucukelbir:2015:autovistan}
A.~Kucukelbir, R.~Ranganath, A.~Gelman, and D.~M. Blei.
\newblock Automatic variational inference in {S}tan.
\newblock \emph{Advances in Neural Information Processing Systems 28}, 2015.

\bibitem[Kulkarni et~al.(2015)Kulkarni, Kohli, Tenenbaum, and
  Mansinghka]{Kulkarni:2015:picture}
T.~D. Kulkarni, P.~Kohli, J.~B. Tenenbaum, and V.~K. Mansinghka.
\newblock Picture: A probabilistic programming language for scene perception.
\newblock \emph{IEEE Conference on Computer Vision and Pattern Recognition},
  2015.

\bibitem[Li and Grathwohl(2018)]{Li:2018:glow}
X.~Li and W.~Grathwohl.
\newblock Training {G}low with constant memory cost.
\newblock \emph{Bayesian Deep Learning Workshop at Neural Information
  Processing Systems}, 2018.

\bibitem[Liu et~al.(2016)Liu, Lee, and Jordan]{Liu:2016:gof}
Q.~Liu, J.~D. Lee, and M.~Jordan.
\newblock A kernelized {S}tein discrepancy for goodness-of-fit tests.
\newblock \emph{Proceedings of the 33rd International Conference on
  International Conference on Machine Learning}, 2016.

\bibitem[Lueckmann et~al.(2017)Lueckmann, Goncalves, Bassetto, \"{O}cal,
  Nonnenmacher, and Macke]{Lueckmann:2017:snpe}
J.-M. Lueckmann, P.~J. Goncalves, G.~Bassetto, K.~\"{O}cal, M.~Nonnenmacher,
  and J.~H. Macke.
\newblock Flexible statistical inference for mechanistic models of neural
  dynamics.
\newblock \emph{Advances in Neural Information Processing Systems 30}, 2017.

\bibitem[Lueckmann et~al.(2018)Lueckmann, Bassetto, Karaletsos, and
  Macke]{Lueckmann:2018:maxvar}
J.-M. Lueckmann, G.~Bassetto, T.~Karaletsos, and J.~H. Macke.
\newblock Likelihood-free inference with emulator networks.
\newblock \emph{Symposium on Advances in Approximate Bayesian Inference at
  Neural Information Processing Systems}, 2018.

\bibitem[MacKay(2002)]{MacKay:2002:itila}
D.~J.~C. MacKay.
\newblock \emph{Information Theory, Inference \& Learning Algorithms}.
\newblock Cambridge University Press, 2002.

\bibitem[Mansinghka et~al.(2013)Mansinghka, Kulkarni, Perov, and
  Tenenbaum]{Mansinghka:2013:graphics}
V.~K. Mansinghka, T.~D. Kulkarni, Y.~N. Perov, and J.~B. Tenenbaum.
\newblock Approximate {B}ayesian image interpretation using generative
  probabilistic graphics programs.
\newblock \emph{Advances in Neural Information Processing Systems 26}, 2013.

\bibitem[Marjoram et~al.(2003)Marjoram, Molitor, Plagnol, and
  Tavar\'{e}]{Marjoram:2003:mcmc_abc}
P.~Marjoram, J.~Molitor, V.~Plagnol, and S.~Tavar\'{e}.
\newblock Markov chain {M}onte {C}arlo without likelihoods.
\newblock \emph{Proceedings of the National Academy of Sciences}, 100\penalty0
  (26):\penalty0 15324--15328, 2003.

\bibitem[Markram et~al.(2015)]{Markram:2015:neurosim}
H.~Markram et~al.
\newblock Reconstruction and simulation of neocortical microcircuitry.
\newblock \emph{Cell}, 163:\penalty0 456--492, 2015.

\bibitem[McLachlan and Basford(1988)]{McLachlan:1988:mixture_models}
G.~J. McLachlan and K.~E. Basford.
\newblock \emph{Mixture models: {I}nference and applications to clustering}.
\newblock Marcel Dekker, 1988.

\bibitem[Menick and Kalchbrenner(2019)]{Menick:2019:pixel_nets}
J.~Menick and N.~Kalchbrenner.
\newblock Generating high fidelity images with subscale pixel networks and
  multidimensional upscaling.
\newblock \emph{Proceedings of the 7th International Conference on Learning
  Representations}, 2019.

\bibitem[Minka(2001)]{Minka:2001:EP}
T.~P. Minka.
\newblock Expectation propagation for approximate {B}ayesian inference.
\newblock \emph{Proceedings of the 17th Conference in Uncertainty in Artificial
  Intelligence}, 2001.

\bibitem[M\"{u}ller et~al.(2018)M\"{u}ller, McWilliams, Rousselle, Gross, and
  Nov\'{a}k]{Muller:2018:nis}
T.~M\"{u}ller, B.~McWilliams, F.~Rousselle, M.~Gross, and J.~Nov\'{a}k.
\newblock Neural importance sampling.
\newblock \emph{arXiv:1808.03856}, 2018.

\bibitem[Murray(2007)]{Murray:2007:mcmc}
I.~Murray.
\newblock \emph{Advances in {M}arkov chain {M}onte {C}arlo methods}.
\newblock {P}h{D} thesis, Gatsby computational neuroscience unit, University
  College London, 2007.

\bibitem[Nash and Durkan(2019)]{Nash:2019:AEM}
C.~Nash and C.~Durkan.
\newblock Autoregressive energy machines.
\newblock \emph{arXiv:1904.05626}, 2019.

\bibitem[Nash et~al.(2017)Nash, Eslami, Burgess, Higgins, Zoran, Weber, and
  Battaglia]{Nash:2017:multi_entity_vae}
C.~Nash, S.~M.~A. Eslami, C.~Burgess, I.~Higgins, D.~Zoran, T.~Weber, and
  P.~Battaglia.
\newblock The multi-entity variational autoencoder.
\newblock \emph{Learning Disentangled Features Workshop at Neural Information
  Processing Systems}, 2017.

\bibitem[Neal(1993)]{Neal:1993:mcmc}
R.~M. Neal.
\newblock Probabilistic inference using {M}arkov chain {M}onte {C}arlo methods.
\newblock Technical Report CRG-TR-93-1, Department of Computer Science,
  University of Toronto, 1993.

\bibitem[Nowozin(2015)]{Nowozin:2015:ess}
S.~Nowozin.
\newblock Effective sample size in importance sampling, Aug. 2015.
\newblock URL
  \url{http://www.nowozin.net/sebastian/blog/effective-sample-size-in-importance-sampling.html}.
\newblock Published online.

\bibitem[Oliva et~al.(2018)Oliva, Dubey, Zaheer, Poczos, Salakhutdinov, Xing,
  and Schneider]{Oliva:2018:tan}
J.~Oliva, A.~Dubey, M.~Zaheer, B.~Poczos, R.~Salakhutdinov, E.~Xing, and
  J.~Schneider.
\newblock Transformation autoregressive networks.
\newblock \emph{Proceedings of the 35th International Conference on Machine
  Learning}, 2018.

\bibitem[Paige and Wood(2016)]{Paige:2016:inference_nets}
B.~Paige and F.~Wood.
\newblock Inference networks for sequential {M}onte {C}arlo in graphical
  models.
\newblock \emph{Proceedings of the 33rd International Conference on Machine
  Learning}, 2016.

\bibitem[Papamakarios(2015)]{Papamakarios:2015:distilling_model_knowledge}
G.~Papamakarios.
\newblock Distilling model knowledge.
\newblock {MS}c by {R}esearch thesis, School of Informatics, University of
  Edinburgh, 2015.

\bibitem[Papamakarios(2018)]{Papamakarios:2018:maf_datasets}
G.~Papamakarios.
\newblock Preprocessed datasets for {MAF} experiments, 2018.
\newblock URL \url{https://doi.org/10.5281/zenodo.1161203}.

\bibitem[Papamakarios and Murray(2015)]{Papamakarios:2015:distilling}
G.~Papamakarios and I.~Murray.
\newblock Distilling intractable generative models.
\newblock \emph{Probabilistic Integration Workshop at Neural Information
  Processing Systems}, 2015.

\bibitem[Papamakarios and Murray(2016)]{Papamakarios:2016:efree}
G.~Papamakarios and I.~Murray.
\newblock Fast $\epsilon$-free inference of simulation models with {B}ayesian
  conditional density estimation.
\newblock \emph{Advances in Neural Information Processing Systems 29}, 2016.

\bibitem[Papamakarios et~al.(2017)Papamakarios, Pavlakou, and
  Murray]{Papamakarios:2017:maf}
G.~Papamakarios, T.~Pavlakou, and I.~Murray.
\newblock Masked autoregressive flow for density estimation.
\newblock \emph{Advances in Neural Information Processing Systems 30}, 2017.

\bibitem[Papamakarios et~al.(2019)Papamakarios, Sterratt, and
  Murray]{Papamakarios:2019:snl}
G.~Papamakarios, D.~C. Sterratt, and I.~Murray.
\newblock Sequential neural likelihood: Fast likelihood-free inference with
  autoregressive flows.
\newblock \emph{Proceedings of the 22nd International Conference on Artificial
  Intelligence and Statistics}, 2019.

\bibitem[Paszke et~al.(2017)Paszke, Gross, Chintala, Chanan, Yang, DeVito, Lin,
  Desmaison, Antiga, and Lerer]{Paszke:2017:pytorch}
A.~Paszke, S.~Gross, S.~Chintala, G.~Chanan, E.~Yang, Z.~DeVito, Z.~Lin,
  A.~Desmaison, L.~Antiga, and A.~Lerer.
\newblock Automatic differentiation in {P}y{T}orch.
\newblock \emph{Autodiff Workshop at Neural Information Processing Systems},
  2017.

\bibitem[Pospischil et~al.(2008)Pospischil, Toledo-Rodriguez, Monier,
  Piwkowska, Bal, Fr{\'e}gnac, Markram, and
  Destexhe]{Pospischil:2008:hh_models}
M.~Pospischil, M.~Toledo-Rodriguez, C.~Monier, Z.~Piwkowska, T.~Bal,
  Y.~Fr{\'e}gnac, H.~Markram, and A.~Destexhe.
\newblock Minimal {H}odgkin--{H}uxley type models for different classes of
  cortical and thalamic neurons.
\newblock \emph{Biological Cybernetics}, 99\penalty0 (4):\penalty0 427--441,
  2008.

\bibitem[Prangle(2017)]{Prangle:2017:adapting_distance}
D.~Prangle.
\newblock Adapting the {ABC} distance function.
\newblock \emph{Bayesian Analysis}, 12\penalty0 (1):\penalty0 289--309, 2017.

\bibitem[Prangle et~al.(2014)Prangle, Fearnhead, Cox, Biggs, and
  French]{Prangle:2014:semi_automatic_abc_model_choice}
D.~Prangle, P.~Fearnhead, M.~P. Cox, P.~J. Biggs, and N.~P. French.
\newblock Semi-automatic selection of summary statistics for {ABC} model
  choice.
\newblock \emph{Statistical applications in genetics and molecular biology},
  13\penalty0 (1):\penalty0 67--82, 2014.

\bibitem[Prenger et~al.(2018)Prenger, Valle, and
  Catanzaro]{Prenger:2018:WaveGlow}
R.~Prenger, R.~Valle, and B.~Catanzaro.
\newblock {W}ave{G}low: {A} flow-based generative network for speech synthesis.
\newblock \emph{arXiv:1811.00002}, 2018.

\bibitem[Pritchard et~al.(1999)Pritchard, Seielstad, Perez-Lezaun, and
  Feldman]{Pritchard:1999:pop_growth}
J.~K. Pritchard, M.~T. Seielstad, A.~Perez-Lezaun, and M.~W. Feldman.
\newblock Population growth of human {Y} chromosomes: a study of {Y} chromosome
  microsatellites.
\newblock \emph{Molecular Biology and Evolution}, 16\penalty0 (12):\penalty0
  1791--1798, 1999.

\bibitem[Qian(1999)]{Qian:1999:momentum}
N.~Qian.
\newblock On the momentum term in gradient descent learning algorithms.
\newblock \emph{Neural Networks}, 12\penalty0 (1):\penalty0 145--151, 1999.

\bibitem[Radford et~al.(2016)Radford, Metz, and Chintala]{Radford:2016:dcgan}
A.~Radford, L.~Metz, and S.~Chintala.
\newblock Unsupervised representation learning with deep convolutional
  generative adversarial networks.
\newblock \emph{Proceedings of the 4th International Conference on Learning
  Representations}, 2016.

\bibitem[Ranganath et~al.(2014)Ranganath, Gerrish, and
  Blei]{Ranganath:2014:bbvi}
R.~Ranganath, S.~Gerrish, and D.~M. Blei.
\newblock Black box variational inference.
\newblock \emph{Proceedings of the 17th International Conference on Artificial
  Intelligence and Statistics}, 2014.

\bibitem[Ratmann et~al.(2007)Ratmann, J{\o}rgensen, Hinkley, Stumpf,
  Richardson, and Wiuf]{Ratmann:2007:protein_nets}
O.~Ratmann, O.~J{\o}rgensen, T.~Hinkley, M.~Stumpf, S.~Richardson, and C.~Wiuf.
\newblock Using likelihood-free inference to compare evolutionary dynamics of
  the protein networks of {H}.~pylori and {P}.~falciparum.
\newblock \emph{PLoS Computational Biology}, 3\penalty0 (11):\penalty0
  2266--2278, 2007.

\bibitem[Reddi et~al.(2018)Reddi, Kale, and Kumar]{Reddi:2018:amsgrad}
S.~J. Reddi, S.~Kale, and S.~Kumar.
\newblock On the convergence of {A}dam and beyond.
\newblock \emph{Proceedings of 6th International Conference on Learning
  Representations}, 2018.

\bibitem[Rezende(2018)]{Rezende:2018:divergence}
D.~J. Rezende.
\newblock Short notes on divergence measures, July 2018.
\newblock URL
  \url{https://danilorezende.com/wp-content/uploads/2018/07/divergences.pdf}.
\newblock Published online.

\bibitem[Rezende and Mohamed(2015)]{Rezende:2015:flows}
D.~J. Rezende and S.~Mohamed.
\newblock Variational inference with normalizing flows.
\newblock \emph{Proceedings of the 32nd International Conference on Machine
  Learning}, 2015.

\bibitem[Rezende et~al.(2014)Rezende, Mohamed, and Wierstra]{Rezende:2014:vae}
D.~J. Rezende, S.~Mohamed, and D.~Wierstra.
\newblock Stochastic backpropagation and approximate inference in deep
  generative models.
\newblock \emph{Proceedings of the 31st International Conference on Machine
  Learning}, 2014.

\bibitem[Romaszko et~al.(2017)Romaszko, Williams, Moreno, and
  Kohli]{Romaszko:2017:vision_as_graphics}
L.~Romaszko, C.~K.~I. Williams, P.~Moreno, and P.~Kohli.
\newblock Vision-as-inverse-graphics: Obtaining a rich {3D} explanation of a
  scene from a single image.
\newblock \emph{IEEE International Conference on Computer Vision Workshop},
  2017.

\bibitem[Salakhutdinov and Murray(2008)]{Salakhutdinov:2008:ais}
R.~Salakhutdinov and I.~Murray.
\newblock On the quantitative analysis of deep belief networks.
\newblock \emph{Proceedings of the 25th Annual International Conference on
  Machine Learning}, 2008.

\bibitem[Salimans et~al.(2017)Salimans, Karpathy, Chen, and
  Kingma]{Salimans:2017:pixelcnnpp}
T.~Salimans, A.~Karpathy, X.~Chen, and D.~P. Kingma.
\newblock {P}ixel{CNN}++: Improving the {P}ixel{CNN} with discretized logistic
  mixture likelihood and other modifications.
\newblock \emph{Proceedings of the 5th International Conference on Learning
  Representations}, 2017.

\bibitem[Salman et~al.(2018)Salman, Yadollahpour, Fletcher, and
  Batmanghelich]{Salman:2018:diffeomorphic}
H.~Salman, P.~Yadollahpour, T.~Fletcher, and K.~Batmanghelich.
\newblock Deep diffeomorphic normalizing flows.
\newblock \emph{arXiv:1810.03256}, 2018.

\bibitem[Schafer and Freeman(2012)]{Schafer:2012:lfi_cosmology}
C.~M. Schafer and P.~E. Freeman.
\newblock Likelihood-free inference in cosmology: Potential for the estimation
  of luminosity functions.
\newblock \emph{Statistical Challenges in Modern Astronomy V}, pages 3--19,
  2012.

\bibitem[Schroecker et~al.(2019)Schroecker, Vecerik, and
  Scholz]{Schroecker:2019:imitation}
Y.~Schroecker, M.~Vecerik, and J.~Scholz.
\newblock Generative predecessor models for sample-efficient imitation
  learning.
\newblock \emph{Proceedings of 7th International Conference on Learning
  Representations}, 2019.

\bibitem[Scott(1992)]{Scott:1992:mde}
D.~W. Scott.
\newblock \emph{Multivariate Density Estimation: Theory, Practice, and
  Visualization}.
\newblock John Wiley \& Sons, 1992.

\bibitem[Silver et~al.(2016)]{Silver:2016:alphago}
D.~Silver et~al.
\newblock Mastering the game of {G}o with deep neural networks and tree search.
\newblock \emph{Nature}, 529\penalty0 (7587):\penalty0 484--489, 2016.

\bibitem[Silverman(1986)]{Silverman:1986:density}
B.~W. Silverman.
\newblock \emph{Density Estimation for Statistics and Data Analysis}.
\newblock Chapman \& Hall, 1986.

\bibitem[Sisson et~al.(2007)Sisson, Fan, and Tanaka]{Sisson:2007:smc_abc}
S.~A. Sisson, Y.~Fan, and M.~M. Tanaka.
\newblock Sequential {M}onte {C}arlo without likelihoods.
\newblock \emph{Proceedings of the National Academy of Sciences}, 104\penalty0
  (6):\penalty0 1760--1765, 2007.

\bibitem[Sj\"{o}strand et~al.(2008)Sj\"{o}strand, Mrenna, and
  Skands]{Sjostrand:2008:pythia}
T.~Sj\"{o}strand, S.~Mrenna, and P.~Skands.
\newblock A brief introduction to {PYTHIA} 8.1.
\newblock \emph{Computer Physics Communications}, 178\penalty0 (11):\penalty0
  852--867, 2008.

\bibitem[Sterratt et~al.(2011)Sterratt, Graham, Gillies, and
  Willshaw]{Sterratt:2011:neuro}
D.~C. Sterratt, B.~Graham, A.~Gillies, and D.~Willshaw.
\newblock \emph{Principles of Computational Modelling in Neuroscience}.
\newblock Cambridge University Press, 2011.

\bibitem[Tipping and Bishop(1999)]{Tipping:1999:ppca}
M.~E. Tipping and C.~M. Bishop.
\newblock Probabilistic principal component analysis.
\newblock \emph{Journal of the Royal Statistical Society, Series B},
  21\penalty0 (3):\penalty0 611--622, 1999.

\bibitem[Toni et~al.(2009)Toni, Welch, Strelkowa, Ipsen, and
  Stumpf]{Toni:2009:smc_abc}
T.~Toni, D.~Welch, N.~Strelkowa, A.~Ipsen, and M.~P.~H. Stumpf.
\newblock Approximate {B}ayesian computation scheme for parameter inference and
  model selection in dynamical systems.
\newblock \emph{Journal of The Royal Society Interface}, 6\penalty0
  (31):\penalty0 187--202, 2009.

\bibitem[Tran et~al.(2018)Tran, Hoffman, Moore, Suter, Vasudevan, and
  Radul]{Tran:2018:distr_pp}
D.~Tran, M.~D. Hoffman, D.~Moore, C.~Suter, S.~Vasudevan, and A.~Radul.
\newblock Simple, distributed, and accelerated probabilistic programming.
\newblock \emph{Advances in Neural Information Processing Systems 31}, 2018.

\bibitem[van~den Berg et~al.(2018)van~den Berg, Hasenclever, Tomczak, and
  Welling]{Berg:2018:sylvester}
R.~van~den Berg, L.~Hasenclever, J.~M. Tomczak, and M.~Welling.
\newblock {S}ylvester normalizing flows for variational inference.
\newblock \emph{Proceedings of the 34th Conference on Uncertainty in Artificial
  Intelligence}, 2018.

\bibitem[van~den Oord et~al.(2016{\natexlab{a}})van~den Oord, Dieleman, Zen,
  Simonyan, Vinyals, Graves, Kalchbrenner, Senior, and
  Kavukcuoglu]{VanDenOord:2016:WaveNet}
A.~van~den Oord, S.~Dieleman, H.~Zen, K.~Simonyan, O.~Vinyals, A.~Graves,
  N.~Kalchbrenner, A.~W. Senior, and K.~Kavukcuoglu.
\newblock Wave{N}et: {A} generative model for raw audio.
\newblock \emph{arXiv:1609.03499}, 2016{\natexlab{a}}.

\bibitem[van~den Oord et~al.(2016{\natexlab{b}})van~den Oord, Kalchbrenner,
  Espeholt, Kavukcuoglu, Vinyals, and Graves]{VanDenOord:2016:PixelCNN}
A.~van~den Oord, N.~Kalchbrenner, L.~Espeholt, K.~Kavukcuoglu, O.~Vinyals, and
  A.~Graves.
\newblock Conditional image generation with {P}ixel{CNN} decoders.
\newblock \emph{Advances in Neural Information Processing Systems 29},
  2016{\natexlab{b}}.

\bibitem[Vikram et~al.(2019)Vikram, Hoffman, and Johnson]{Vikram:2019:loracs}
S.~Vikram, M.~D. Hoffman, and M.~J. Johnson.
\newblock The {LORAC}s prior for {VAE}s: {L}etting the trees speak for the
  data.
\newblock \emph{Proceedings of the 22nd International Conference on Artificial
  Intelligence and Statistics}, 2019.

\bibitem[Wasserman(2010)]{Wasserman:2010:stats}
L.~Wasserman.
\newblock \emph{All of Statistics: A Concise Course in Statistical Inference}.
\newblock Springer Publishing Company, 2010.

\bibitem[Welling et~al.(2004)Welling, Williams, and Agakov]{Welling:2004:eca}
M.~Welling, C.~K.~I. Williams, and F.~V. Agakov.
\newblock Extreme components analysis.
\newblock \emph{Advances in Neural Information Processing Systems 16}, 2004.

\bibitem[Wilkinson(2011)]{Wilkinson:2011:systems_bio}
D.~J. Wilkinson.
\newblock \emph{Stochastic Modelling for Systems Biology, Second Edition}.
\newblock Taylor \& Francis, 2011.

\bibitem[Wilkinson(2013)]{Wilkinson:2013:abc_wrong_model}
R.~D. Wilkinson.
\newblock Approximate {B}ayesian computation ({ABC}) gives exact results under
  the assumption of model error.
\newblock \emph{Statistical applications in genetics and molecular biology},
  12\penalty0 (2):\penalty0 129--141, 2013.

\bibitem[Williams and Agakov(2002)]{Williams:2002:pmca}
C.~K.~I. Williams and F.~V. Agakov.
\newblock Products of {G}aussians and probabilistic minor component analysis.
\newblock \emph{Neural Computation}, 14\penalty0 (5):\penalty0 1169--1182,
  2002.

\bibitem[Wood(2010)]{Wood:2010:sl}
S.~N. Wood.
\newblock Statistical inference for noisy nonlinear ecological dynamic systems.
\newblock \emph{Nature}, 466\penalty0 (7310):\penalty0 1102--1104, 2010.

\bibitem[Wu et~al.(2015)Wu, Yildirim, Lim, Freeman, and
  Tenenbaum]{Wu:2015:galileo}
J.~Wu, I.~Yildirim, J.~J. Lim, B.~Freeman, and J.~B. Tenenbaum.
\newblock {G}alileo: Perceiving physical object properties by integrating a
  physics engine with deep learning.
\newblock \emph{Advances in Neural Information Processing Systems 28}, 2015.

\bibitem[Zeiler(2012)]{Zeriler:2012:adadelta}
M.~D. Zeiler.
\newblock {ADADELTA}: {A}n adaptive learning rate method.
\newblock \emph{arXiv:1212.5701}, 2012.

\bibitem[Zhang et~al.(2018)Zhang, E, and Wang]{Zhang:2018:monge_ampere}
L.~Zhang, W.~E, and L.~Wang.
\newblock {M}onge--{A}mp\`{e}re flow for generative modeling.
\newblock \emph{arXiv:1809.10188}, 2018.

\bibitem[Ziegler and Rush(2019)]{Ziegler:2019:latent_flows}
Z.~M. Ziegler and A.~M. Rush.
\newblock Latent normalizing flows for discrete sequences.
\newblock \emph{arXiv:1901.10548}, 2019.

\bibitem[Zoran and Weiss(2011)]{Zoran:2011:patches}
D.~Zoran and Y.~Weiss.
\newblock From learning models of natural image patches to whole image
  restoration.
\newblock \emph{Proceedings of the 13rd International Conference on Computer
  Vision}, 2011.

\end{thebibliography}

\end{document}